\documentclass{article}

\usepackage{microtype}
\usepackage{graphicx}
\usepackage{subcaption} 
\usepackage{booktabs} 
\usepackage{bbm}

\usepackage{tikz}

\usepackage{xcolor}

\usepackage{hyperref}
\usepackage{url}
\usepackage[normalem]{ulem}
\usepackage{stmaryrd,scalerel}

\usepackage[accepted]{icml2026}

\usepackage{amsmath}
\usepackage{amssymb}
\usepackage{mathtools}
\usepackage{amsthm}

\usepackage{pifont}

\newcommand{\cmark}{\Large\textcolor{green!80!black}{\ding{51}}}
\newcommand{\xmark}{\Large\textcolor{red}{\ding{55}}}
\newcommand{\mygreen}[1]{{\color{green!60!black}#1}}

\usepackage[capitalize,noabbrev]{cleveref}

\theoremstyle{plain}
\newtheorem{theorem}{Theorem}[section]
\newtheorem{proposition}[theorem]{Proposition}
\newtheorem{lemma}[theorem]{Lemma}

\theoremstyle{definition}
\newtheorem{definition}[theorem]{Definition}

\theoremstyle{remark}
\newtheorem{remark}[theorem]{Remark}

\usepackage[disable,textsize=tiny]{todonotes}

\icmltitlerunning{Conformal Policy Control}

\begin{document}

\twocolumn[
\icmltitle{Conformal Policy Control}



\icmlsetsymbol{equal}{*}

\begin{icmlauthorlist}
\icmlauthor{Drew Prinster}{gen,jhu}
\icmlauthor{Clara Fannjiang}{gen}
\icmlauthor{Ji Won Park}{gen}
\icmlauthor{Kyunghyun Cho}{nyu,prev_gen}
\icmlauthor{Anqi Liu}{jhu}
\icmlauthor{Suchi Saria}{jhu}
\icmlauthor{Samuel Stanton}{prev_gen}
\end{icmlauthorlist}

\icmlaffiliation{gen}{Prescient Design,
Genentech, U.S.A.}
\icmlaffiliation{jhu}{Johns Hopkins University, Baltimore, U.S.A.}
\icmlaffiliation{prev_gen}{Formerly Prescient Design, Genentech}
\icmlaffiliation{nyu}{New York University, NYC, U.S.A.}

\icmlcorrespondingauthor{Drew Prinster}{drewprinster@gmail.com}
\icmlcorrespondingauthor{Samuel Stanton}{sdstanton1@gmail.com}

\icmlkeywords{machine-learning, conformal-risk-control, model-based-optimization, safe-exploration, conservative-optimization}

\vskip 0.3in
]



\printAffiliationsAndNotice{}  

\begin{abstract}
    An agent must try new behaviors to explore and improve. In high-stakes environments, an agent that violates safety constraints may cause harm and must be taken offline, curtailing any future interaction. Imitating old behavior is safe, but excessive conservatism discourages exploration. How much behavior change is too much? We show how to use any safe reference policy as a probabilistic regulator for any optimized but untested policy. Conformal calibration on data from the safe policy determines how aggressively the new policy can act, while provably enforcing the user's declared risk tolerance. Unlike conservative optimization methods, we do not assume the user has identified the correct model class nor tuned any hyperparameters. Unlike previous conformal methods, our theory provides finite-sample guarantees even for non-monotonic bounded loss functions, and it introduces a new policy control setting. Our experiments on applications ranging from natural language question answering to biomolecular engineering show that safe exploration is not only possible from the first moment of deployment, but can also improve performance.
\end{abstract}

\section{Introduction}

Safe exploration is one of the most studied problems in AI safety \citep{amodei2016concrete}.
The choice between probable safety and the hope of possible discovery presents a dilemma known by many names, from the explore-exploit tradeoff in reinforcement learning to the validity-power tradeoff in statistics.
The optimal balance is problem-dependent, yet general solutions must be problem-independent.
So, to apply to real-world problems, any general solution must encode this balance as a hyperparameter, which then must be tuned from data or derived from assumptions about problem structure.
In consequential settings, collecting data from an unsafe behavior policy (a context-dependent strategy or distribution over actions) may be unacceptable. 
Deriving the balance from assumptions about the problem reintroduces a subtler form of
problem-dependence, namely the practitioner's understanding of the setting itself.
Yet consequential, poorly understood settings are precisely where machine learning offers the greatest potential benefit.

\begin{figure*}
    \centering
    \includegraphics[width=\linewidth]{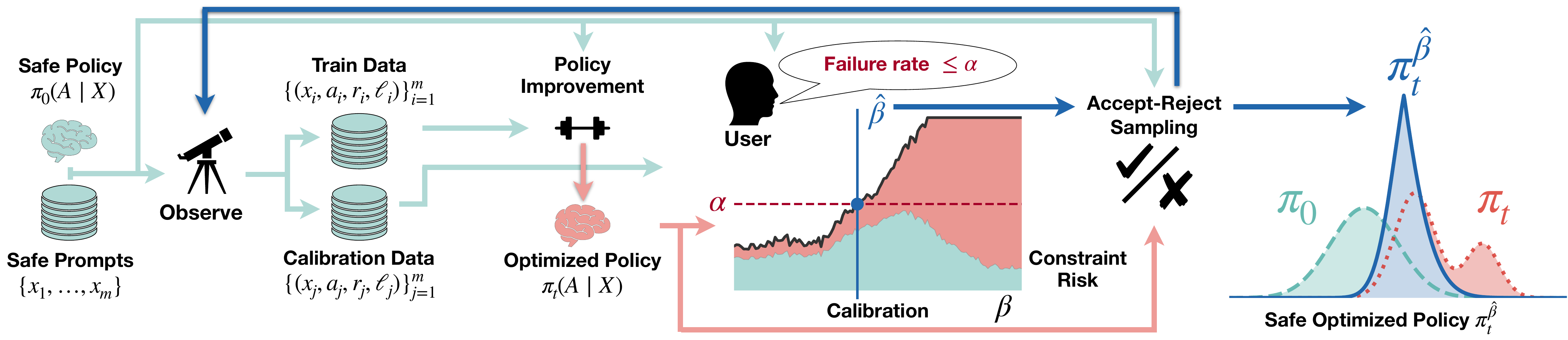}
    \caption{
        An illustration of safe exploration with conformal policy control.
        Starting with a safe policy $\pi_0$  and safe context prompts $\mathcal{X}_0:=\{x_1, \dots, x_m\}$, we observe the safe policy's actions $a_i$ and corresponding reward $r_i \in \mathbb{R}$ and constraint violation $\ell_i \in \{0, 1\}$.
        We split the observations into train and calibration data, and optimize the safe policy to obtain $\pi_t$. We then query the user for their constraint violation risk tolerance $\alpha$, and apply conformal risk control to find a bound on the likelihood ratio $\pi_t / \pi_0$ which guarantees the calibrated risk is controlled at level $\alpha$.
        Finally we apply rejection sampling to probabilistically regulate the optimized policy for deployment and observation, allowing the user to flexibly trade off reward, constraint risk, and test-time compute.
    }
    \label{fig:safe_policy_improvement_marquee}
\end{figure*}

Consider an agent that observes contexts (i.e., states) and takes actions to maximize reward while satisfying safe operating constraints to within some tolerance.
  For example, a language model responding to medical questions aims to give answers that are as helpful as possible while avoiding false claims.
  In biomolecular engineering, a generative model proposing improvements to lead molecules aims to maximize efficacy of its designs while ensuring that they are synthesizable.
In such settings, the agent has candidate policies optimized for performance but does not yet know which, if any, are safe to deploy.
  From past experience, the agent usually has a default safe policy it knows controls the risk, but playing it safe may be far from optimal.
  
  What the safe policy does supply is calibration data: Reweighting each observation by a candidate-to-safe likelihood ratio gives an estimate of each candidate's risk.
  But which candidate should the agent deploy?
  The natural move is to estimate every candidate's risk and pick the most aggressively optimized one that still looks safe, but this approach introduces bias: Selection corrupts the procedure by favoring estimates that understate their true risk.
  Yet the change in risk the agent must control is due to the policy shift it will author.
  The agent does not know how the world will respond to a given action, but it knows exactly how each candidate would distribute its actions, which is the only difference between the distributions it has experienced in the past and those it would encounter in the future.\footnote{If we assume the environment is stationary.}

We introduce conformal policy control (CPC, Figure~\ref{fig:safe_policy_improvement_marquee}), a method for safe exploration that uses the calibration data to interpolate between any safe and any optimized policy while provably respecting the risk tolerance. 
CPC requires no assumptions on the reward or constraint functions, no access to the optimized policy's training process, and no additional samples to tune the balance beyond what a safe policy already provides.
The key idea is to invert the usual importance-weighting paradigm: Rather than estimating risk for a fixed policy, we search over candidate policies, parameterized by a likelihood-ratio threshold, for the most aggressive one such that all evaluated candidates still satisfy the risk tolerance (with a weighted conservative adjustment). Conformal calibration on the safe policy's data determines both the weights and the threshold.
At deployment, the interpolated policy can be realized through rejection sampling (or other Monte Carlo methods), which allows the agent to probabilistically self-regulate to a zone of competence determined by its calibration data.
Because CPC operates entirely at test time, the same safe and optimized policies can be reused under different risk tolerances, trading test-time compute for risk control without retraining.

Providing finite-sample guarantees in this setting poses several challenges for uncertainty quantification. The deployed policy depends on the calibration data, introducing subtle, data-dependent distribution shift. 
The relevant losses are general constraint functions, not just coverage indicators. 
Conformal risk control (CRC) \citep{angelopoulos2022conformal} can address these challenges, but assumes a monotone relationship between the control parameter and the loss. In CPC, the loss function does not depend on the control parameter, so the monotonicity condition does not hold. We generalize CRC (gCRC) to non-monotonic bounded loss functions, and we prove that CPC maintains finite-sample guarantees when the control parameter governs the policy, not the loss.


We validate our methods on three tasks. In medical question answering, gCRC controls the false discovery rate, a non-monotonic loss that standard CRC cannot
  handle, achieving tighter control and better claim recall than existing methods. In constrained active learning, CPC provides finite-sample guarantees under
  feedback-loop covariate shifts where previous approaches offered only asymptotic control. In black-box biomolecular sequence optimization, CPC extends conformal
  methods to a setting where no prior guarantees existed, and we find that moderate risk control can counterintuitively improve optimization performance by reducing
   wasted evaluations on infeasible actions. Safe exploration is not only safe, it can also be more efficient.

\vspace{-0.25cm}

\section{Preliminaries and Background}
\label{sec:background}

\textbf{Notation:} We consider a decision-making agent that interacts with an environment over a sequence of rounds. At each round $t$, the agent observes a context (i.e., prompt) $X_t \in \mathcal{X}$ and selects an action $A_t \in \mathcal{A}$ according to a policy $\pi(A \mid X)$.
If the policy is parameterized by an LLM, then $\mathcal{A} \subseteq \mathcal{X}$.
The agent then observes (potentially noisy) rewards $R_t \in \mathbb{R}$ and losses $L_t \in \mathbb{R}$.
Throughout, uppercase letters denote random variables and lowercase letters their realized values.
The reward function $r(x, a)$ quantifies the degree to which a context-action pair is aligned with the agent's objective, and the loss function $\ell(x, a)$ quantifies the degree to which the context-action pair violates the agent's constraints.
While a policy can be defined implicitly by solving an optimization problem with respect to estimated reward and loss surrogates \citep{watkins1992q, jones1998efficient}, in this paper we focus on controlling explicit policies parameterized by a conditional distribution \citep{williams1992simple}, such as an autoregressive transformer. 

\subsection{The Safe Policy Improvement Problem}
The setting just introduced provides a general framework for sequential decision making under uncertainty \citep{langford2008epoch}, capturing many problems of practical interest \citep{li2010contextual}.
In natural language question answering \citep{rajpurkar2016squad}, the context could be a user query, the action a generated response, and the reward and loss could reflect response completeness and (in)correctness, respectively.
In constrained black-box optimization \citep{gardner2014bayesian}, the context could be the previous evaluations, the action the next solution to evaluate, and the reward and loss functions could be black-box oracle objective and constraint functions, respectively \citep{chen2025generalists}.
For example, in protein engineering, the loss function may indicate whether a proposed sequence lies outside a feasibility set $\mathcal{F}$ of expressible proteins, i.e., $\ell(X_t, A_t) := \mathbbm{1}\{A_t \notin \mathcal{F}\}$.

The agent's goal is to learn a policy, $\pi$, that maximizes expected reward while controlling expected loss:
\begin{align}\label{eq:constrained_policy_opt}
    \max_{\pi} \quad & \mathbb{E}_{X \sim p} \ \mathbb{E}_{A \sim \pi(\cdot \mid X)} \ [r(X, A)] \nonumber \\
    \text{subject to} \quad & \mathbb{E}_{X \sim p} \ \mathbb{E}_{A \sim \pi(\cdot \mid X)} \ [\ell(X, A)] \leq \alpha,
\end{align}
where $\alpha \in [0, B]$ (for some $B<\infty$) is a user-specified risk tolerance,  and $p(X)$ is the marginal context distribution.\footnote{In constrained optimization notation, $r$ and $\ell$ map to $f$ and $g$.}

In practice, the agent usually begins with a safe reference policy $\pi_0$, whether inherited from a prior deployment or trained as a density model of known safe actions, and seeks to improve upon it.
In the LLM literature, policy improvement is also known as post-training or continual learning and encompasses methods such as supervised fine-tuning \citep[SFT;][]{wei2022finetuned}, reinforcement learning with human feedback \citep[RLHF;][]{christiano2017deep} or direct preference optimization \citep[DPO;][]{rafailov2023direct}.
The quality of the improved policy depends on the post-training data, the fidelity of any reward/preference models, and design decisions 
governing divergence limits from the safe reference. 

Post-training design decisions immediately prompt two questions: which policy divergence measure is best, and how much divergence is allowable to maintain an acceptable risk profile?
The first question has been studied extensively (see Section \ref{sec:related_work}).
The second, however, is typically treated as a hyperparameter tuning problem \citep{bergstra2011algorithms, feurer2019hyperparameter}, in which the user must specify a constraint on a ``control parameter'' (such as a divergence budget or penalty weight) and the algorithm optimizes subject to that constraint.
The resulting indirection creates a gap between what the user actually wants and the control parameters that are accessible in available algorithms \citep{hutchins1986direct, norman1988design}.
The user's goal is declarative (i.e., a desired outcome of the program): control risk at some level $\alpha$.
The algorithms demand imperative instructions (i.e., the steps the program should execute): bound divergence at some level, or set the penalty weight to some value.
Without a principled translation, the user must learn the mapping from outcome to instruction by trial and error, spending the very samples the agent seeks to use efficiently.

As previewed in the introduction, importance weighting offers a more direct route. Given data from the safe policy, the agent can estimate the expected loss under any candidate policy by reweighting each observation by the likelihood ratio between the candidate and safe policies, via the identity
\begin{align*}
    \operatorname*{\mathbb{E}}_{\substack{X \sim p(\cdot) \\ A \sim \pi_t(\cdot \mid X)}} [\ell(X, A)] = \operatorname*{\mathbb{E}}_{\substack{X \sim p(\cdot) \\ A \sim \pi_0(\cdot \mid X)}} \left[ \frac{\pi_t(A \mid X)}{\pi_0(A \mid X)} \, \ell(X, A) \right].
\end{align*}
Bounding this ratio at some threshold $\beta$ simultaneously defines a constrained policy and bounds the importance weights, so the safe policy's data suffices to determine whether the candidate still respects the risk tolerance.\footnote{Bounding the likelihood ratio $\pi_t(a \mid x) / \pi_0(a \mid x) \leq \beta$ strictly constrains divergence by uniformly implying a bound on the KL
  divergence and other $f$-divergences between $\pi_t$ and $\pi_0$. 
   }

\subsection{Conformal Risk Control}

What remains is a principled way to choose $\beta$ from finite safe-policy data, with a guarantee on the risk of the resulting deployed policy.
Conformal risk control \citep[CRC;][]{angelopoulos2022conformal} provides this for a restricted class of problems: it sets a control parameter $\lambda$ to control the expected value of a user-specified loss $L_i(\lambda)$.
Given exchangeable (e.g., IID) calibration loss functions, $L_1(\cdot), \ldots, L_n(\cdot)$, and test loss function, $L_{n+1}(\cdot)$, the key assumption is that each $L_i(\lambda)$ is monotonically nonincreasing with the control parameter, $\lambda$.
For example, conformal prediction \citep{vovk2005algorithmic} is a special case of CRC where the loss is the miscoverage indicator $L(\lambda) = \mathbbm{1}\{Y \notin \hat{C}_{\lambda}(X)\}$ for labels $Y$ and prediction sets $\hat C_\lambda$, and $\lambda$ controls the size of the prediction set: increasing the prediction set size can only decrease miscoverage.
CRC extends conformal prediction to cases such as setting an inclusion threshold on predicted probabilities for multilabel classification with false negative rate (FNR) control (i.e., controlling the expected proportion of true labels that are not included in the set). Like miscoverage, FNR is monotonic in $\lambda$: adding labels to the prediction set can only increase the proportion of true labels contained. 
By selecting $\hat{\lambda}$ as the smallest value such that the (conservatively adjusted) empirical risk falls below $\alpha$, CRC attains the finite-sample guarantee $\mathbb{E}[L_{n+1}(\hat{\lambda})] \leq \alpha$.
In each case, increasing $\lambda$ can only decrease the loss. This monotonicity is what lets CRC search from aggressive to conservative and stop at the first value that satisfies the guarantee.

  For policy control, however, it is not immediately clear how CRC could be applied.
  Many losses of practical interest admit no tunable control parameter at all.
  Consider again the feasibility indicator $L_t := \mathbbm{1}\{A_t \notin \mathcal{F}\}$: a prediction set $\hat{C}_{\lambda}$ can be grown to reduce miscoverage, but the feasibility set $\mathcal{F}$ is fixed by the environment, and there is no $\lambda$ that makes a given action more or less feasible.
  We instead place the control parameter on the policy, governing the distribution of actions rather than the loss they incur.
  Even so, the resulting expected loss is not necessarily monotonic in that parameter, a theoretical challenge we address in Section \ref{sec:theory}.

\section{Abbreviated Related Work}
\label{sec:related_work}

\textbf{Conservative Model-Based Optimization:}
A standard approach to safe policy improvement relies on the observation that the risk of a new policy can be bounded in terms of the risk of a reference policy plus a penalty that grows with the divergence between them \citep{kakade2002approximately}.
When the reference policy is known to satisfy the safety constraint, controlling this divergence provides a mechanism for controlling risk.
The idea of bounding risk through bounded divergence motivates a broad family of locality-constrained policy optimization methods \citep{kim2025offline}.
Prior work has proposed entropy-regularized control and KL-penalized objectives \citep{todorov2009efficient, fox2015taming}, trust-region policy optimization \citep{schulman2015trust, schulman2017proximal}, conservative value penalties in offline reinforcement learning \citep{kumar2020conservative, trabucco2021conservative}, trust-region methods in black-box optimization \citep{eriksson2019scalable}, and uncertainty-penalized safe Bayesian optimization \citep{sui2015safe}.
\citet{damani2023beyond} proposed using a likelihood ratio as a safety filter for model-based optimization, but treated the ratio threshold as a hyperparameter.

From an optimization perspective, our work is most similar to \citet{stanton2023bayesian}, who proposed constraining a Bayesian optimization policy to regions where the resulting conformal prediction sets would have bounded width given a declared miscoverage tolerance $\alpha$.
However, their approach was limited by three weaknesses.
First, their approach was restricted to the miscoverage loss.
Second, their use of an implicit policy created an awkward fixed-point problem: finding the policy that induced an action distribution that controlled the coverage.
Third, they did not fully address the multi-step feedback covariate shift induced by sequential model-based optimization.\footnote{A form of covariate shift in which the test distribution is dependent on previously observed data.}
Indeed, \citet{stanton2023bayesian} noted that, at the time, the existing conformal theory had not yet proved general enough guarantees 
for that setting.

\begin{figure}[!t]
\centering
\begin{subfigure}{0.42\textwidth}
    \centering\fbox{\includegraphics[width=\textwidth]{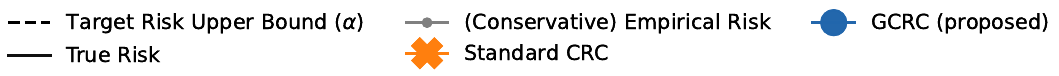}}\\
    \begin{subfigure}{0.49\textwidth}
        \includegraphics[width=\textwidth]{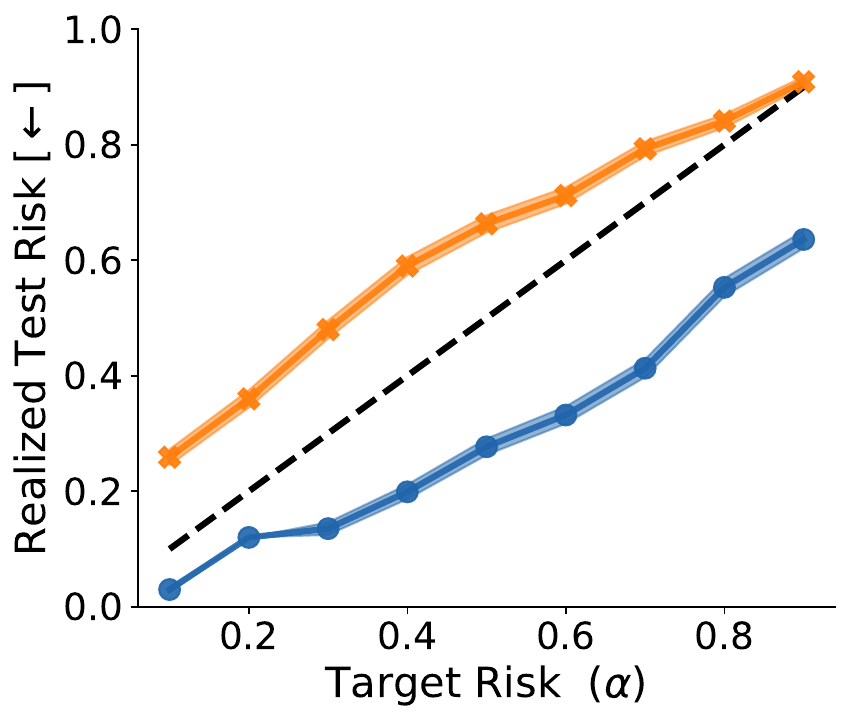}
    \end{subfigure}
    \hfill
    \begin{subfigure}{0.49\textwidth}
        \includegraphics[width=\textwidth]{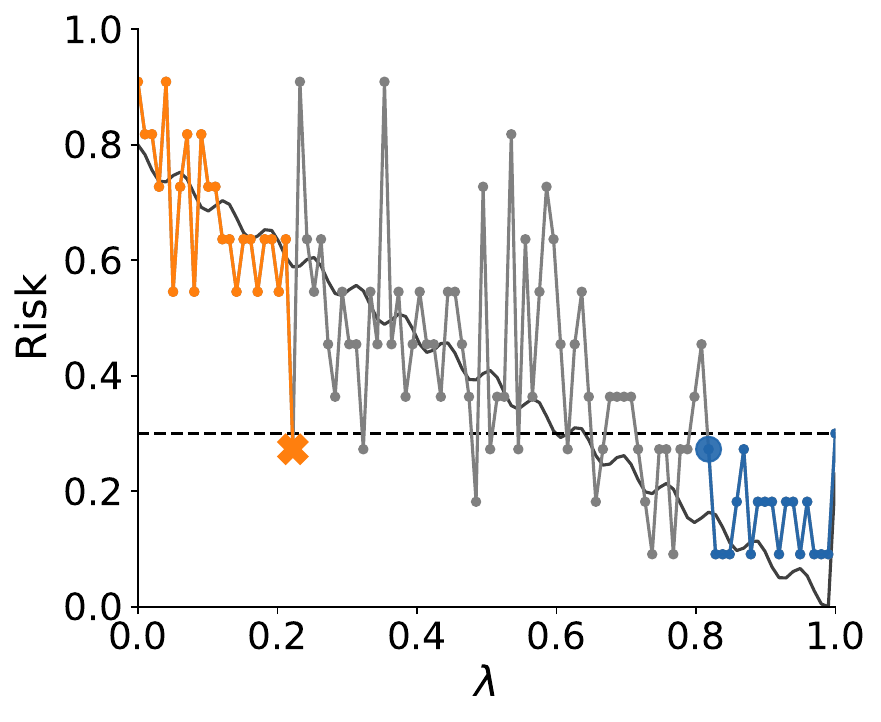}
    \end{subfigure}
\end{subfigure}
\caption{Empirical comparison of standard CRC \citep{angelopoulos2022conformal} (orange) and the proposed generalized CRC (blue) on synthetic non-monotonic losses. gCRC achieves risk control across various target risk levels while CRC does not \textbf{(left)}. On an example empirical risk trajectory for these experiments, CRC underestimates the risk while gCRC maintains risk control by searching from safe to aggressive hyperparameter values \textbf{(right)}.}
\label{fig:CRC_counterexample}
\end{figure}

\textbf{Conformal Methods:}
\citet{stanton2023bayesian} drew motivation from \citet{fannjiang2022conformal}, who first formalized the phenomenon of \textit{feedback covariate shift} to generalize conformal prediction to the setting of a single round of model-based optimization, building on techniques introduced by \citet{tibshirani2019conformal}.
\citet{prinster2024conformal} further extended these techniques to generalize conformal prediction for any data distribution, including showing coverage guarantees for the multi-round model-based optimization setting.

However, the techniques used throughout this line of work do not enable one to prescribe a data distribution (e.g., select a policy) that guarantees coverage.
That is, these methods can provide prediction sets with coverage guarantees for actions generated by any \textit{fixed} policy (a ``descriptive'' use of conformal techniques), but using those prediction sets to then \textit{select} a policy (a ``prescriptive'' use) introduces further, unaccounted-for dependence between the observed data and test policy that nullifies coverage guarantees.

To enable this prescriptive use case, we build upon techniques from conformal risk control (CRC) \cite{angelopoulos2022conformal}, a general framework for selecting a hyperparameter with finite-sample guarantees for any monotonic loss, of which conformal prediction can be cast as a special case.
When the loss of interest is non-monotonic, \citet{angelopoulos2022conformal} also showed that deploying the procedure on a monotonic relaxation of the losses provides asymptotic guarantees of control.
However, they left the problem of finite-sample guarantees for non-monotonic losses open.
In the next section, we develop a variant of CRC that drops the monotonicity requirement, and instead relies on smoothness of the losses and a form of stability of the calibration procedure to control the risk on the test point.
Finally, we apply techniques from \citet{prinster2024conformal} to prove finite-sample guarantees for our procedure in the multi-round model-based optimization setting.
See Appendix \ref{app:extended_related_work} for an extended discussion of related work.

\begin{figure*}[t]
    \centering
    \subfloat[$\beta \ll 1$]{
        \includegraphics[width=0.33\textwidth]{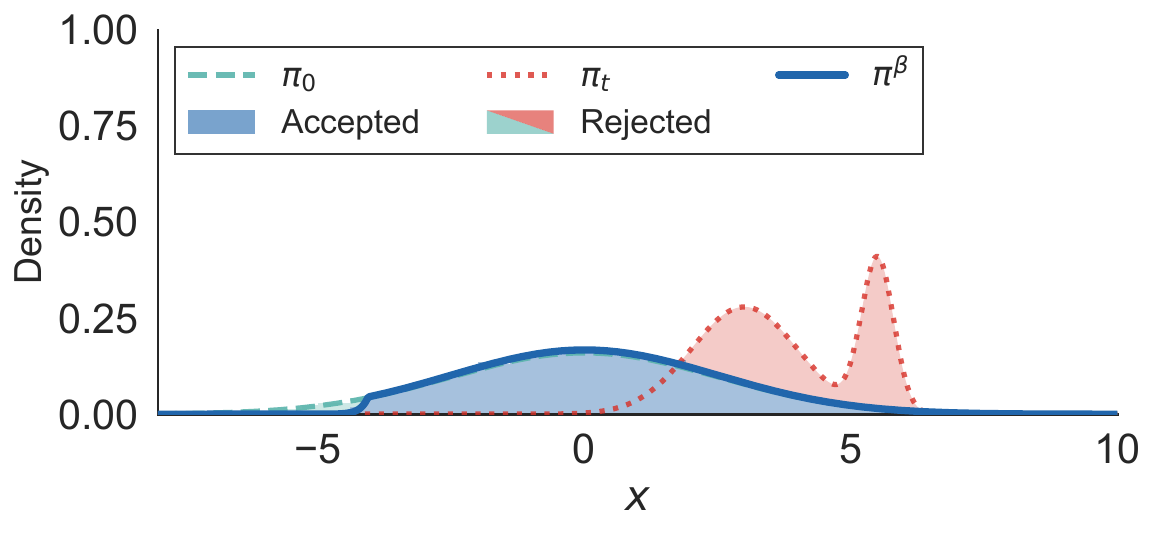}
        \label{fig:ar_beta_near_zero}
    }
    \subfloat[$\beta = 1$]{
        \includegraphics[width=0.33\textwidth]{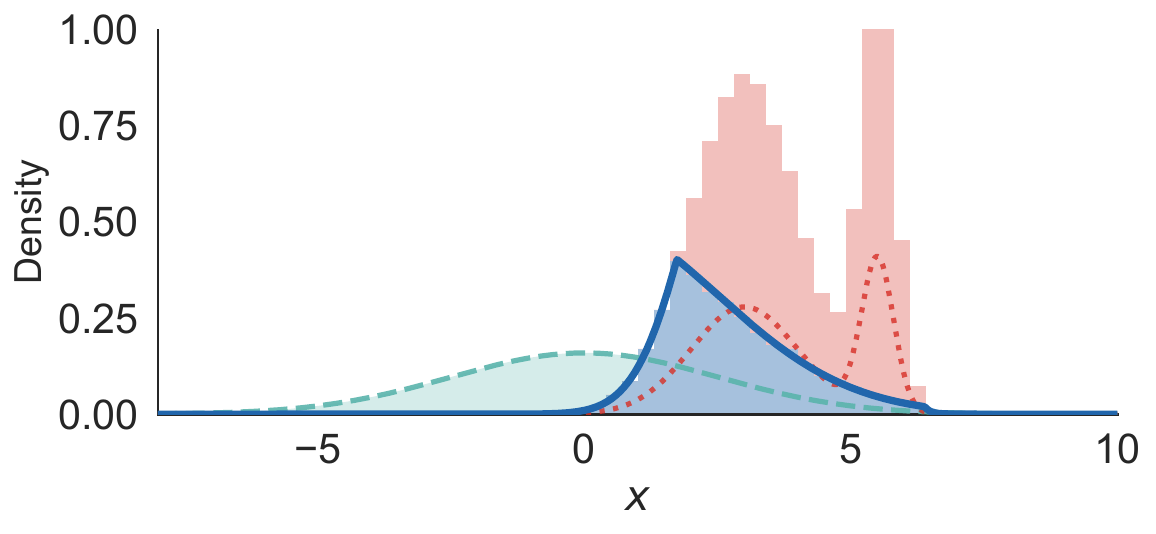}
        \label{fig:ar_beta_one}
    }
    \subfloat[$\beta \gg 1$]{
        \includegraphics[width=0.33\textwidth]{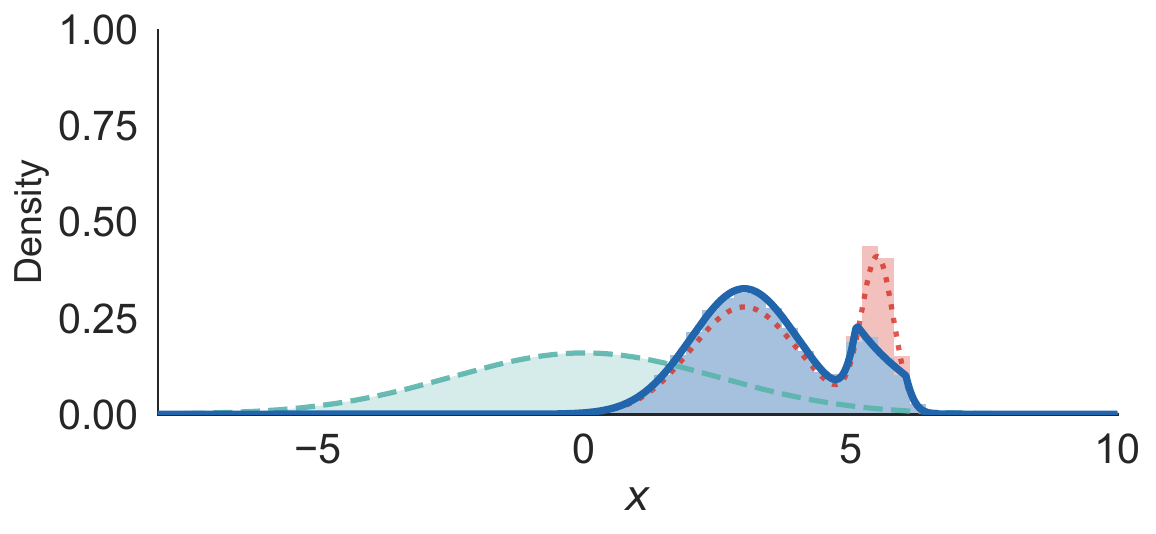}
        \label{fig:ar_beta_large}
    }
    \caption{Rejection sampling for policy interpolation at three $\beta$ values. 
    As $\beta$ increases from near-zero to large values, the constrained policy $\pi_t^{(\beta)}$
    interpolates from the safe policy $\pi_0$ to the optimized policy $\pi_t$. 
    Blue histogram bars show accepted samples matching $\pi_t^{(\beta)}$, while teal/red bars show
    rejected samples from the proposal distribution. (a) Small $\beta$ nearly recovers the safe policy (b) Intermediate $\beta$ shows intermediate interpolation with frequent rejections. (c) Large $\beta$ approaches the optimized policy with minimal rejection.}
    \label{fig:interpolating_distributions}
\end{figure*}

Concurrently, \citet{angelopoulos2026conformal} also present CRC guarantees for non-monotonic losses, but for parameterized losses, rather than the policy-control setting we introduce here where the control parameter tunes the policy instead of the loss. 
Even in the comparable setting of parameterized losses, however, the contributions are distinct, with neither subsuming the other: Relative to the univariate method in \citet{angelopoulos2026conformal}, our method (Section \ref{subsec:finite_sample_guarantees_crc_non_monotone}) will always be at least as conservative \textit{a priori} (e.g., as in Figure \ref{fig:CRC_counterexample}, right); consequently, whereas that paper's analogous guarantee relies on leave-one-out stability, our results use replace-one stability, which is generally a more relaxed assumption \citep{bousquet2002stability, shalev2010learnability}.

\section{Methods and Theory}
\label{sec:theory}

\setlength{\fboxrule}{0.1pt} 
\setlength{\fboxsep}{0.1pt}

\subsection{Extending Conformal Risk Control to Non-Monotonic Losses}
\label{subsec:finite_sample_guarantees_crc_non_monotone}

We begin with intuition. Standard CRC's proof \citep{angelopoulos2022conformal} has two key steps: First, prove that the oracle solution (using knowledge of the test point's loss), $\lambda^*$, achieves risk control; and second, show that the empirical CRC solution (with a conservative adjustment for the unknown test loss), $\hat{\lambda}$, is deterministically greater than or equal to the oracle, $\hat{\lambda}\geq \lambda^*$. Note that here, larger $\lambda$ are safer. In particular, for standard CRC, monotonicity of the loss function is then necessary and sufficient\footnote{Sufficient given the other CRC assumptions: boundedness and right-continuity of the loss, and existence of a safe $\lambda_{\max}$.} for concluding that the empirical CRC algorithm's risk is controlled: $\mathbb{E}[L_{n+1}(\hat{\lambda})] \leq \mathbb{E}[L_{n+1}(\lambda^*)]\leq \alpha$.

Without monotonicity, what can be done? The crux is how to find a modified oracle solution, $\lambda^*_+$, and a modified empirical algorithm, $\hat\lambda_+$, such that establishing the relation $\hat\lambda_+\geq \lambda^*_+$ enables analyzing the risk of the empirical algorithm \textit{by construction}, rather than by monotonicity. Consider the family of hyperparameters defined as
\begin{align}\label{eq:methods_lambda_hat_family_def}
    &\widehat{\lambda}_+(\lbag L_{1:m}\rbag, \alpha) \\
    & \ := \inf \Big\{\lambda_0 \in \Lambda \ : \ \forall \ \lambda\geq \lambda_0,
    \ \ \sum_{i=1}^{m}\frac{L_i(\lambda)}{m} \leq \alpha \Big\}, \nonumber
\end{align}
when the set is nonempty, and otherwise define $\widehat{\lambda}_+(\lbag L_{1:m}\rbag, \alpha):=\lambda_{\max}$, where $\lambda_{\max}$ is a ``safe hyperparameter'' assumed to exist as in \citet{angelopoulos2022conformal}:  $\lambda_{\max}:=\sup\Lambda\in \Lambda$ and $L_i(\lambda_{\max})\leq \alpha$.
In words, $\widehat{\lambda}_+$ is the smallest (most aggressive) hyperparameter such that the empirical risk is controlled at level $\alpha$ \textit{for all hyperparameters at least as large as it}; this ``for all $\lambda\geq \lambda_0$'' qualification is the key algorithmic modification to standard CRC. 
We then define the oracle solution by Eq. \eqref{eq:methods_lambda_hat_family_def} fit on the bag (or multiset) containing both the calibration and test loss functions, $\lbag L_{1:n+1}\rbag$:
\begin{align}\label{eq:methods_lambda_oracle_def}
    \lambda_+^*:= \widehat{\lambda}_+(\lbag L_{1:n+1}\rbag, \alpha).
\end{align}
Similarly, define the empirical algorithm as Eq. \eqref{eq:methods_lambda_hat_family_def} fit on the bag of the calibration losses, with the bound $B$ in place of $L_{n+1}$ to conservatively adjust for the unknown test loss:
\begin{align}\label{eq:methods_lambda_empirical_def}
    \hat{\lambda}_+:= \widehat{\lambda}_+(\lbag L_{1:n}, B\rbag, \alpha).
\end{align}
Our empirical algorithm, $\hat{\lambda}_+$, reduces to CRC when the losses are monotonic, and it can be viewed as searching hyperparameter space from \textit{a priori} safest to most aggressive, the opposite direction of standard CRC (Figure \ref{fig:CRC_counterexample}). Also note that our generalized CRC (gCRC) is conservatively designed to be \textit{a priori} safer than the oracle, that is, $\hat\lambda_+\geq \lambda^*_+$ almost surely. Unfortunately, this construction is not itself sufficient for $\hat{\lambda}_+$ to achieve validity at level $\alpha$ (Appendix \ref{app:counterexample}). Relating the empirical algorithm to the oracle, however, does provide finite-sample guarantees for $\hat\lambda_+$ under Lipschitz continuity and depending on algorithmic stability.

\begin{definition}[$\epsilon$-replace-one stability]
    \label{def:def_replace_one_stability}
    Let $z_{1:n+1}$ denote the observed values, so that $\lbag Z_{1:n+1}\rbag = \lbag z_{1:n+1}\rbag$. We say a function  $h$ is $\epsilon$-replace-one stable if for all $i, j\in [n+1]$, 
    \begin{align*}
        \mathbb{E}\Big[\big|h(\lbag Z_{1:n+1}\rbag\backslash \{z_i\})-h(\lbag Z_{1:n+1}\rbag\backslash \{Z_j\})\big|\Big] \leq \epsilon.
    \end{align*}
\end{definition}
The following result then provides finite-sample guarantees for $\hat\lambda_+$ even for non-monotonic losses.
\begin{theorem}
\label{thm:gcrc_nonmonotonic}
    Assume exchangeable $L_i(\lambda)$. Define  $\lambda_{\max}:=\sup\Lambda\in \Lambda$ and assume 
    \begin{align*}
        L_i(\lambda_{\max})\leq \alpha, \qquad \sup_{\lambda}L_i(\lambda)\leq B < \infty \quad \text{almost surely}.
    \end{align*}
    If the $L_i(\lambda)$ are $K$-Lipschitz in $\lambda$ and $\hat\lambda_+$ is $\epsilon$-replace-one stable, then
    \begin{align*}
        \mathbb{E}\big[L_{n+1}(\hat\lambda_+)\big] \leq \alpha + K\epsilon.
    \end{align*}
\end{theorem}
\textbf{Proof sketch:} Conditional on the bag of data, $\lbag z_{1:n+1} \rbag$, we show that the oracle $\lambda_+^*$ achieves risk control, which implies risk control for all $\lambda\geq\lambda_+^*$ that are constant conditional on $\lbag z_{1:n+1} \rbag$. We then show $\hat\lambda_+\geq \lambda_+^*$ almost surely. Lastly, using Lipschitz continuity and replace-one stability, we relate the risk of $\hat\lambda_+$ to that of some $\lambda\geq\lambda_+^*$ that is constant conditional on $\lbag z_{1:n+1} \rbag$. We defer the proof to App. \ref{app:gcrc_proof_nonmonotonic_1dim}.

\begin{remark}[Attaining $\alpha$-validity if $K$ is known]
    If $K$ is known, then a more conservative significance level, $\hat\alpha \leq \alpha$, can be used for level $\alpha$ validity, regardless of instability:
    \begin{align*}
        \mathbb{E}\big[L_{n+1}\big(\hat{\lambda}_{+}(\hat\alpha)\big)\big] \leq \alpha,
    \end{align*}
    where $\hat\alpha$ is defined in Appendix \ref{app:conservative_adjustment_K_known}.
\end{remark}

\begin{remark}[Monotone envelope interpretation]
    The $\hat{\lambda}_+$ in Eq. \eqref{eq:methods_lambda_empirical_def} can be viewed as the smallest hyperparameter such that the ``monotonized empirical risk'' is controlled at $\alpha$. \citet{angelopoulos2022conformal} presented the procedure with this interpretation in their App. A, but noted that finite-sample guarantees for $\hat{\lambda}_+$ were at the time unknown. 
\end{remark}

\begin{algorithm}[!t]
    \caption{Pseudo-code for $\beta$ calibration.}
    \label{alg:calibrate_beta}
    \begin{algorithmic}
        \REQUIRE Safe policy $\pi_0$; optimized policies $\pi_1, ..., \pi_t$; sets of prompts $X^{(0)}, ..., X^{(t)}$; calibration data $\mathcal{D}_{\text{cal}}$; proposal data $\mathcal{D}_{\text{prop}} \sim \pi_t$; target risk $\alpha$; loss upper bound $B$.
        \ENSURE Likelihood ratio bound $\hat{\beta}$
        \STATE $\rho_i \gets \pi_t(A_i | X^{(t)}) \,/\, \pi_0(A_i | X^{(0)})$ for $i \in \mathcal{D}_{\text{cal}} \cup \mathcal{D}_{\text{prop}}$
        \STATE $G \gets \textsc{PrepareGrid}(\{\rho_i\})$
        \STATE Initialize: $\hat{\beta} \gets \beta_{\min}$
        \FOR{$\beta \in G$ (ascending)}
            \STATE \# \text{Compute constrained PDF, mixture of past PDFs}
            \STATE $\psi_\beta \gets \textsc{EstimatePsi}(\{\rho_i\}, \beta)$ 
            \STATE $\pi_t^{(\beta)}(A_i) \gets \min(\pi_t(A_i), \beta\cdot \pi_0(A_i)) \,/\, \psi_\beta$ for $i \in \mathcal{D}_{\text{cal}}$
            \STATE $\pi_{0:t-1}^{\text{mix}}(A_i) \gets \textsc{MixPDF}(A_i, \pi_{0:t-1})$ for $i \in \mathcal{D}_{\text{cal}}$
            \STATE
            \STATE \# \text{Compute conformal weights, check empirical risk}
            \STATE $w_i \gets \pi_t^{(\beta)}(A_i) / \pi_{0:t-1}^{\text{mix}}(A_i)$ for $i \in \mathcal{D}_{\text{cal}}$
            \STATE $w_{\max} \gets \textsc{EstimateWMax}(\{w_i\}, \pi_0, \pi_t)$
            \STATE $\tilde{w}_i \gets w_i/(\sum_{j \in \mathcal{D}_{\text{cal}}} w_j + w_{\max})$ for $i\in \mathcal{D}_{\text{cal}}$
            \STATE $\tilde{w}_{\max} \gets w_{\max}/(\sum_{j \in \mathcal{D}_{\text{cal}}} w_j + w_{\max})$
            \IF{$\sum_{i \in \mathcal{D}_{\text{cal}}} \tilde{w}_i \cdot l_i + \tilde{w}_{\max} \cdot B > \alpha$}
                \STATE \textbf{Return:} $\hat{\beta}$
            \ENDIF
            \STATE $\hat{\beta} \gets \beta$
        \ENDFOR
        \STATE \textbf{Return:} $\max(G)$
    \end{algorithmic}
\end{algorithm}

\subsection{Conformal Policy Control: Selecting a Risk-Controlling Policy}
\label{subsec:cpc_selecting_rc_policy}

We first state the primary assumptions required by CPC: access to a safe policy and subsequent optimized policies. For ease of notation we describe CPC in an online setting where one action is taken each policy improvement step, but it also works similarly in a multiround batch setting.

\textbf{Assumption 1: An Initial Safe Policy} Assume access to an initial safe policy, $\pi_{0}$, for which the true risk is controlled at the desired level $\alpha$; that is, assume $\mathbb{E}_{X \sim p} \ \mathbb{E}_{A \sim \pi_{0}(\cdot \mid X)} \ [\ell(X, A)] \leq \alpha$
holds for $\pi_{0}$.
Note that this assumption is weaker than it may appear: any existing policy (a previous deployment, a domain default, or even a uniform baseline) qualifies as $\pi_0$ so long as its risk is at most $\alpha$. When little is known, $\alpha$ can be set loosely and any reasonable baseline suffices; the strength of the resulting guarantee scales with how much is actually known about what is safe.

\textbf{Assumption 2: Access to Each Optimized Policy} At each time $t = 1, 2, ...$, also assume access to the optimized model, $\pi_{t}$, which may be trained to maximize expected reward (as in the ``unconstrained'' part of Eq. \eqref{eq:constrained_policy_opt}). While we require that $\pi_0, \pi_1, ..., \pi_t$ are mutually absolutely continuous,\footnote{It is simple to enforce absolute continuity, for instance with a linear mixture $\bar\pi_t=(1-\delta)\pi_t + \delta\pi_0$ for some $\delta : 0 < \delta \ll 1$.}
we do not require any other assumptions on how $\pi_{t}$ is trained, although better optimized performance (with respect to both the constraint and the reward) may translate into better efficiency and reward of the constrained policy.

\textbf{Constraining Policies by Clipping Likelihood Ratios} We define the constrained policy $\pi_{t}^{(\beta)}$ by clipping the likelihood ratio $\pi_t / \pi_0$ at a bound $\beta > 0$ (i.e., $\beta\in\mathcal{B}\subseteq(0, \infty)$) and renormalizing. Actions whose likelihood ratio exceeds $\beta$ are capped at $\beta \cdot \pi_0$, limiting how far the constrained policy can deviate from the safe policy. That is,
\begin{align}\label{eq:piecewise_constraint_def}
\pi_{t}^{(\beta)}(a \mid x) \propto \min\big(\pi_{t}(a \mid x), \ \beta\cdot \pi_{0}(a \mid x)\big) .
\end{align}
Assume there is some small, positive $\beta_{\min}\ll1$ such that $\beta_{\min}\leq \pi_t(a \mid x)/\pi_0(a \mid x)$ for all $x\in \mathcal{X}$, $a \in \mathcal{A}$. As $\beta \to \beta_{\min}$ the constrained policy approaches the safe one, $\pi_t^{(\beta)} \to \pi_0$, and as $\beta \to \infty$ it approaches the optimized one, $\pi_t^{(\beta)} \to \pi_t$ (Figure \ref{fig:interpolating_distributions}). Crucially, because $\beta$ governs the policy and not the loss, any smoothness needed for a guarantee will be smoothness the agent can compute and verify. Remarks \ref{remark:cpc_extension} and \ref{remark:enforced_smoothness} clarify the smoothness required for a guarantee and how the choice of $\beta$ enforces it.

\textbf{Calibrating $\boldsymbol{\beta}$:} We calibrate $\beta$ analogously to $\hat\lambda_+$ (Section \ref{subsec:finite_sample_guarantees_crc_non_monotone}), searching from what we \textit{a priori} expect is safest to most aggressive. For $\beta$, this means searching in increasing order in a grid (or set of intervals) $G\subset (0, \infty)$, to find the largest value where the conservatively-adjusted weighted empirical risk is still controlled (Algorithm \ref{alg:calibrate_beta}):
\begin{align}\label{eq:beta_hat_def_methods}
    &\hat\beta = \hat{\beta}\big(\lbag L_{0:t-1}, B \rbag, (\tilde{w}_{0:t-1}^{(\cdot)}, \tilde{w}_{\max}^{(\cdot)}),  \alpha\big) \\
    & := \sup \Big\{\beta' \in G \ : \ \forall \ \beta \leq \beta', \ B\tilde{w}_{\max}^{(\beta)} + \sum_{i=0}^{t-1}l_i\tilde{w}_i^{(\beta)} \leq \alpha \Big\}, \nonumber
\end{align}
where the $\tilde{w}_i^{(\beta)}$ and $\tilde{w}_{\max}^{(\beta)}$ are (normalized) conformal importance weights. 
We assume $\hat\beta$ in Eq. \eqref{eq:beta_hat_def_methods} exists almost surely, but if Eq. \eqref{eq:beta_hat_def_methods} does not exist then let $\hat\beta=\beta_{\min}$, which recovers the safe policy, $\pi_t^{(\beta_{\min})}=\pi_0$. For a given deployment, the assumption that $\hat\beta$ exists is easily verifiable by simply checking if $\beta_{\min}$ satisfies Eq. \eqref{eq:beta_hat_def_methods}; i.e., if the safe policy's (conservatively weighted) empirical risk is below $\alpha$ for a given calibration set.
These weights correct for the distribution shift between the policies that generated the calibration data and the candidate policy $\pi_t^{(\beta)}$: since the agent's past actions were drawn from earlier policies, each observation must be reweighted to obtain a valid risk estimate under $\pi_t^{(\beta)}$. This is the importance weighting described in Section~\ref{sec:background}, now made precise for the multi-round setting.

Formally, define $\tilde{w}_i^{(\beta)}$ as in \citet{prinster2024conformal} via the likelihoods of all permutations of calibration data $z_{0:t-1}$ and a hypothetical test sample $z_t \in \mathcal{Z} := \mathcal{X} \times \mathcal{A}$. A key subtlety here is that $z_t$ is not yet known, as CPC occurs prior to sampling. For any $z_t$, define the unnormalized weight
\begin{align*}
    w_i^{(\beta)}:=\sum_{\sigma:\sigma(t)=i}\pi_{0:t}^{(\beta)}(z_{\sigma(0)}, ..., z_{\sigma(t)})
\end{align*}
for all $i\in\{0, ..., t\}$ and for permutations $\sigma$ of the indices. For $i\in\{0, ..., t\}$, further denote the normalized weights as
\begin{align*}
    \tilde{w}_i^{(\beta)}:= w_i^{(\beta)}/\sum_{i=0}^tw_i^{(\beta)}, \hspace{4mm}\tilde{w}_{\max}^{(\beta)}:=\sup_{z_t\in \mathcal{Z}}(\tilde{w}_t^{(\beta)}),
\end{align*}
where $\tilde{w}_{\max}^{(\beta)}$ is a conservative estimate of the test point weight.
Hereon, we refer to normalized calibration and (conservative) test point weights, that is, assuming $\sum_{i=0}^{t-1}\tilde{w}_{i}^{(\beta)}+\tilde{w}_{\max}^{(\beta)}=1$.
In practice these weights are not computationally tractable, but we estimate the $\tilde{w}_i^{(\beta)}$ by approximating past policies as one mixture distribution, $\pi_{0:t-1}^{\text{mix}}\approx \pi_{0:t-1}$, which yields a fast linear runtime; additionally, we estimate $\tilde{w}_{\max}^{(\beta)}$ by an empirical maximum over samples. For exact weights, CPC attains a guarantee that depends on \textit{relative} stability with respect to replacing one calibration sample.
\begin{definition}[$\epsilon^{(1)}$-relative replace-one stability]
    \label{def:def_relative_replace_one_stability}
    Let $\hat\beta^{\backslash Z_j}$ and $\hat\beta^{\backslash z_i}$ be $\hat\beta$ with $Z_j$ or $z_i$ removed from $\lbag Z_{0:t}\rbag=\lbag z_{0:t}\rbag$, respectively (see App. \ref{app:cpc_notation_and_defs}). Then, $\hat\beta$ is $\epsilon^{(1)}$-relative replace-one stable if for all $i, j\in \{0, ..., t\}$,
    \begin{align*}
        \mathbb{E}_{\pi_{0:t}^{(\hat\beta)}}\bigg[\frac{\big|\hat\beta^{\backslash Z_j}-\hat\beta^{\backslash z_i}\big|}{\min\big(\hat\beta^{\backslash Z_j},\hat\beta^{\backslash z_i}\big)}\bigg] \; \leq \; \epsilon^{(1)}.
    \end{align*}
\end{definition}

\begin{theorem}[Conformal policy control guarantee]
\label{thm:cpc}
    Assume that $\mathbb{E}_{\pi_0}[L_0]\leq \alpha$,  $L_i\in[0,B], B<\infty$ for all $i$, and that $\pi_0, \pi_1, ..., \{\pi_t^{(\beta)}\}_{\beta\in \mathcal{B}}$ are mutually absolutely continuous for all $\beta\in\mathcal{B}$. Assume that $\hat\beta$ (Eq. \eqref{eq:beta_hat_def_methods}) exists almost surely, for the likelihood-ratio control parameter as in Eq. \eqref{eq:piecewise_constraint_def}. If  $\hat\beta$ is $\epsilon^{(1)}$-relative replace-one stable, then
    \begin{align*}
        \mathbb{E}_{\pi_{0:t}^{(\hat\beta)}}\big[L_{t} \big] \; \leq \; \alpha + \max(\alpha,\,B-\alpha)\,\epsilon^{(1)}. 
    \end{align*}
\end{theorem}
\begin{remark}[Extension of CPC to other control parameters]
    \label{remark:cpc_extension}
    Theorem \ref{thm:cpc} is the special case of a more general guarantee that allows $\beta$ to be any control parameter for which $\pi_t^{(\beta)}$ is $K$-log-Lipschitz in $\log(\beta)$ (Theorem \ref{thm:cpc_general} in App. \ref{app:cpc_proofs}).
\end{remark}

\begin{figure}[!t]
\centering
\includegraphics[width=\linewidth]{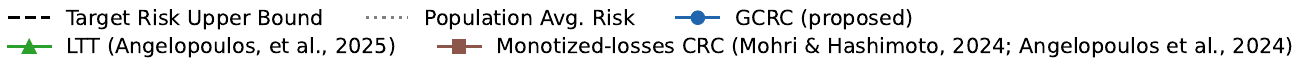}\\[0.5em]
\begin{subfigure}{0.48\linewidth}
    \includegraphics[width=\linewidth]{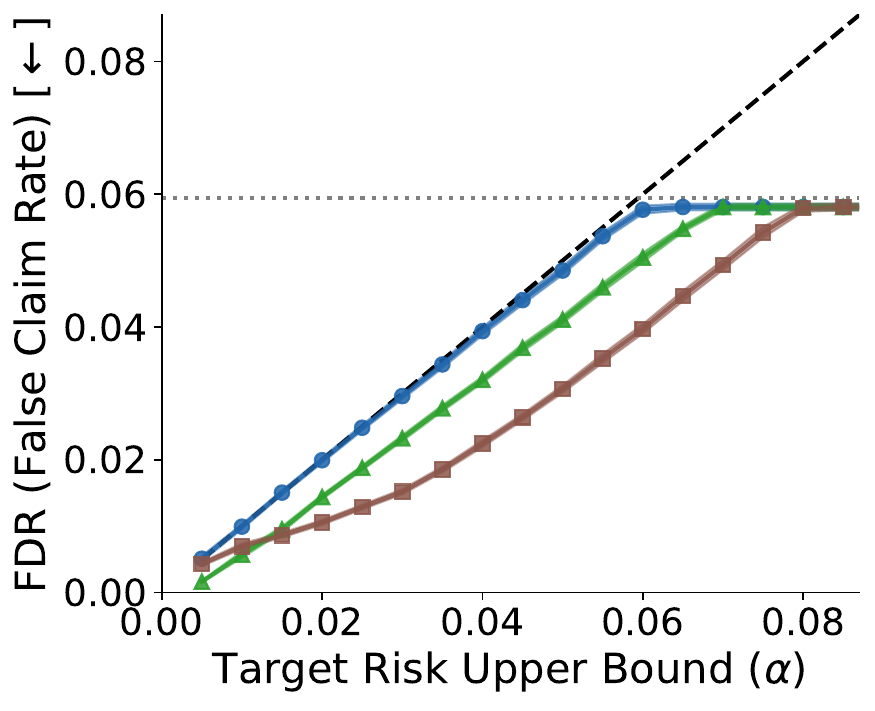}
\end{subfigure}
\hfill
\begin{subfigure}{0.48\linewidth}
    \includegraphics[width=\linewidth]{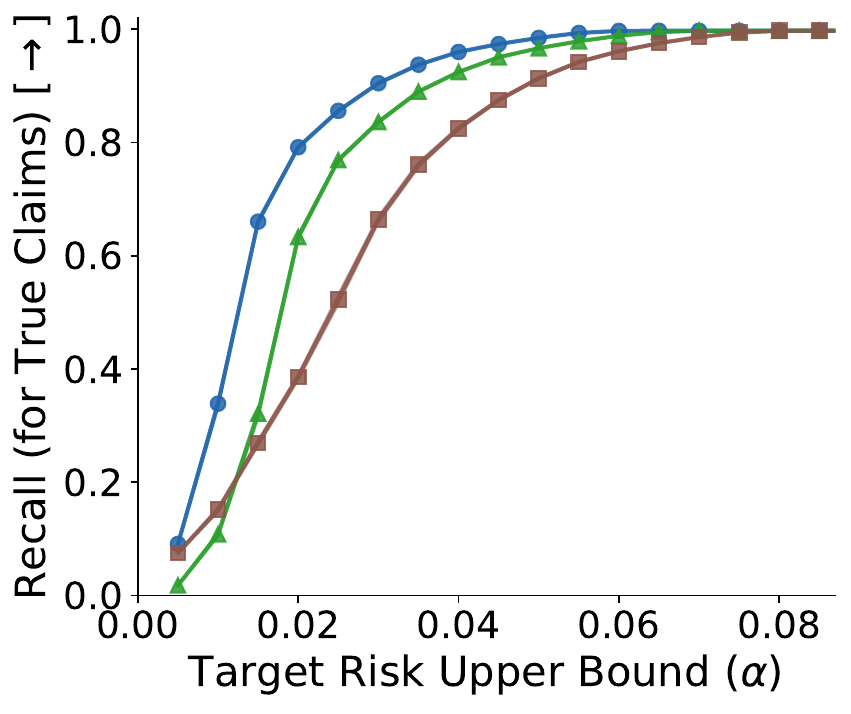}
\end{subfigure}
\caption{Medical QA factuality control results on MedLFQA. \textbf{Left:} Empirical FDR (fraction of retained claims that are false) vs.\ target risk level $\alpha$. The dashed line indicates $y = x$; valid methods should fall at or below this line. \textbf{Right:} Recall (fraction of true claims retained) vs.\ $\alpha$. gCRC tightly controls FDR at the target level while achieving superior recall to baselines. Results averaged over 25 random splits. Error regions are standard errors.}
\label{fig:MedicalQA}
\end{figure}

\begin{remark}[CPC's enforced smoothness condition]
    \label{remark:enforced_smoothness}
    A key advantage of CPC over gCRC is that the smoothness requirement becomes enforceable. 
    Theorem \ref{thm:gcrc_nonmonotonic} requires smoothness of $l_i(\lambda)$, the loss as a function of the control parameter, a property of the unknown environment. For CPC, smoothness is only required of the agent's own constrained policy, $\pi_t^{(\beta)}$, with respect to the control parameter, $\beta$ (see Remark \ref{remark:cpc_extension}). The choice of $\beta$ as the likelihood-ratio bound (Eq. \eqref{eq:piecewise_constraint_def}) directly enforces the needed smoothness: $\pi_t^{(\beta)}$ is 1-log-Lipschitz in $\log(\beta)$ (Lemma \ref{lem:loglip}). 
\end{remark}


\textbf{Sampling from the Constrained Policy in Combinatorial Action Spaces}
In large action spaces, the normalization constant of $\pi_t^{(\hat\beta)}$ is intractable. However, accept-reject (AR) sampling \citep{von1963various} avoids normalization entirely. When $\hat\beta$ is small, proposals from $\pi_0$ are efficient, with the AR envelope constant set to $\hat\beta$. When $\hat\beta$ is large, proposals from $\pi_t$ are more efficient, since $\pi_t^{(\hat\beta)} \approx \pi_t$. See Figure \ref{fig:interpolating_distributions} for an illustration and Appendix \ref{app:sampling} for details.

\begin{figure*}[!t]
\centering
    \begin{subfigure}{\textwidth}
        \centering
        \includegraphics[width=0.95\textwidth]{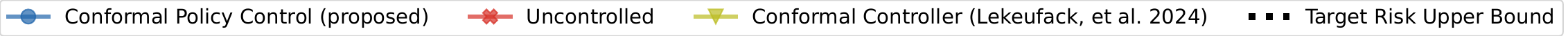}
    \end{subfigure}
    \\
    \begin{subfigure}{0.27\textwidth}
        \includegraphics[width=\textwidth]{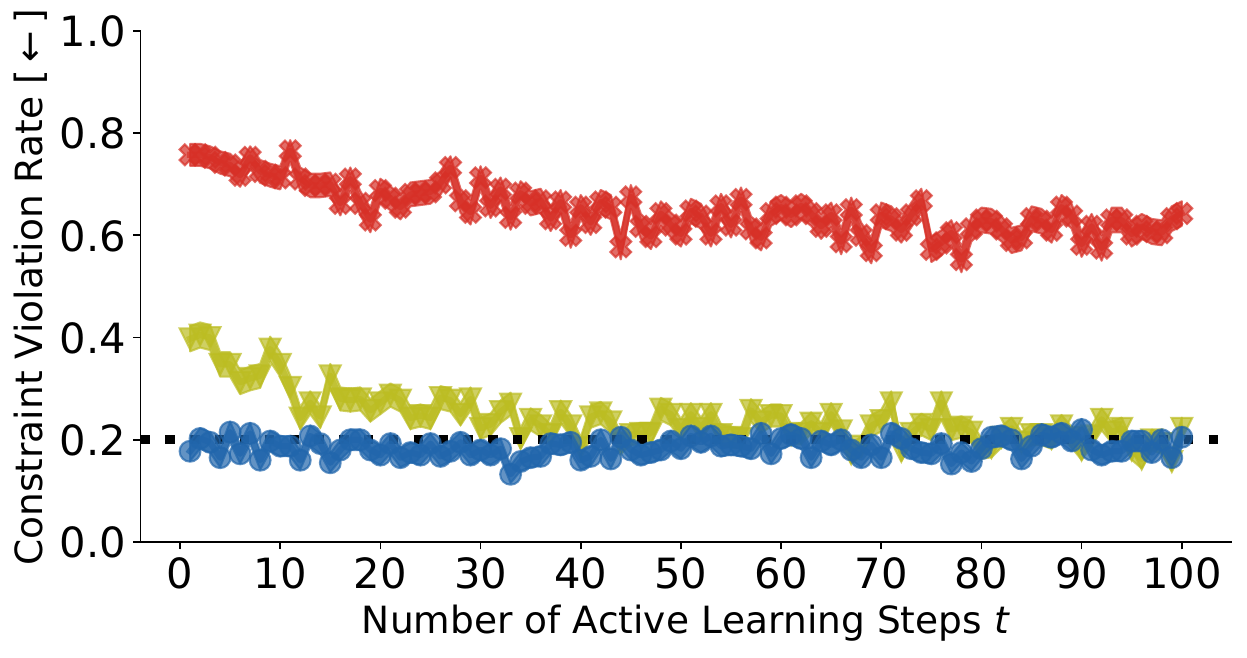}
    \end{subfigure}
    \hspace{1cm}
    \begin{subfigure}{0.27\textwidth}
        \includegraphics[width=\textwidth]{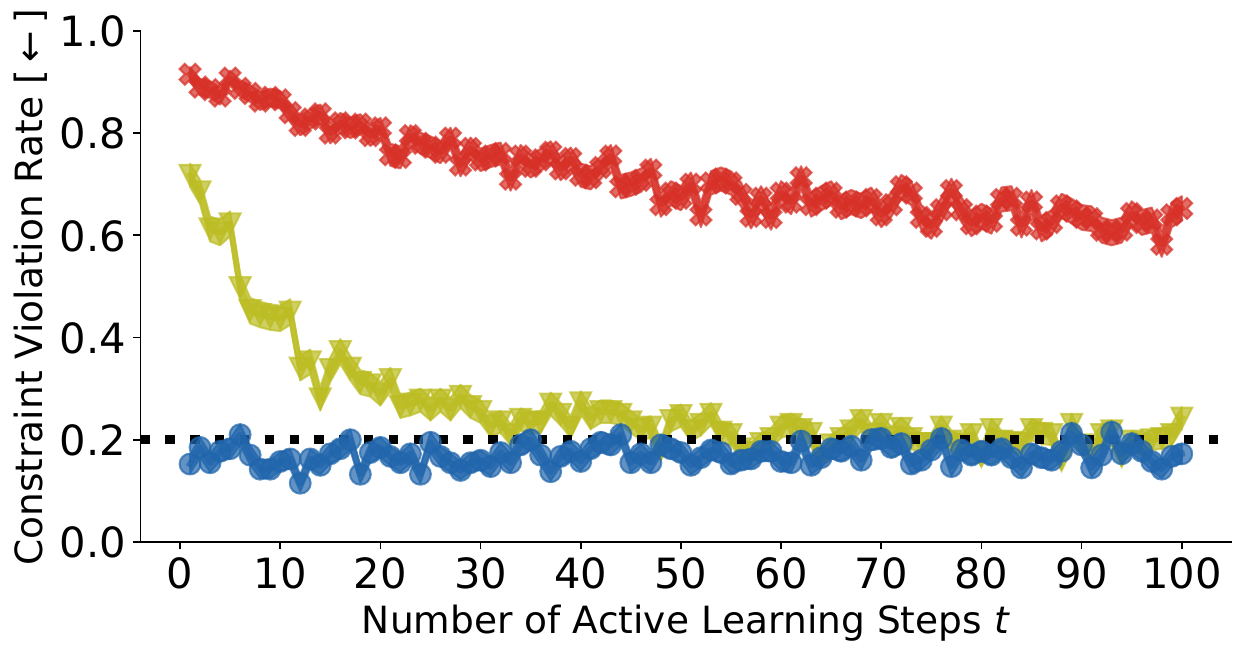}
    \end{subfigure}
    \hspace{1cm}
    \begin{subfigure}{0.27\textwidth}
        \includegraphics[width=\textwidth]{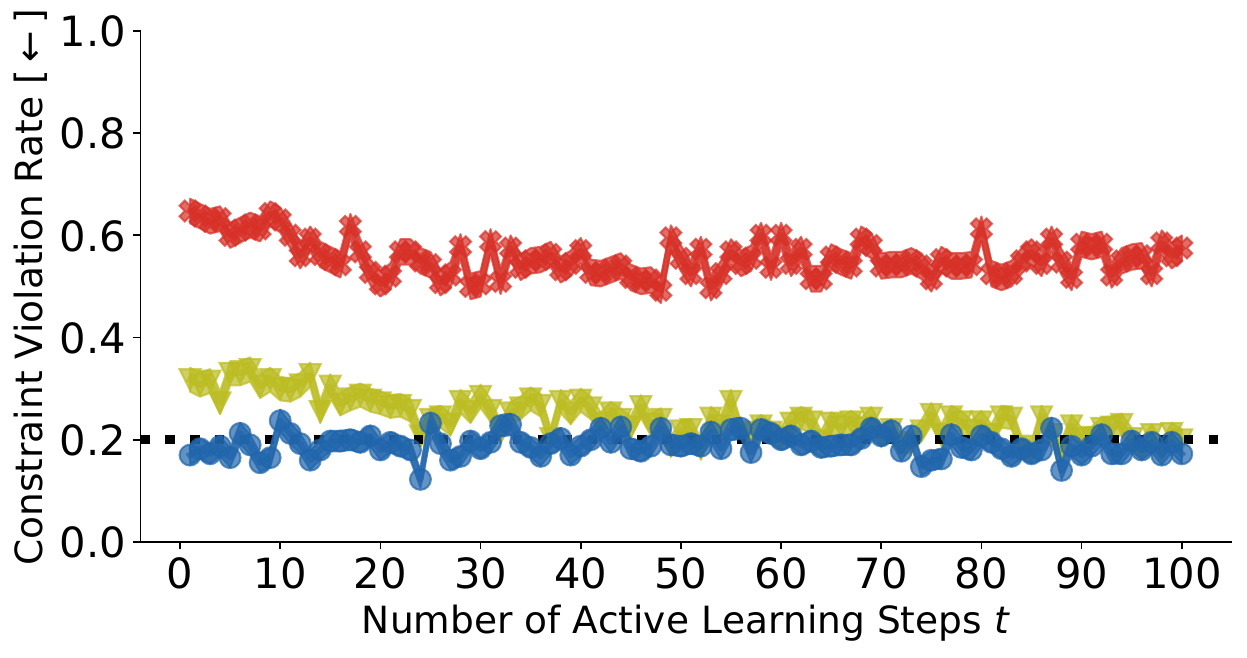}
    \end{subfigure}
    \\
    \begin{subfigure}{0.27\textwidth}
        \includegraphics[width=\textwidth]{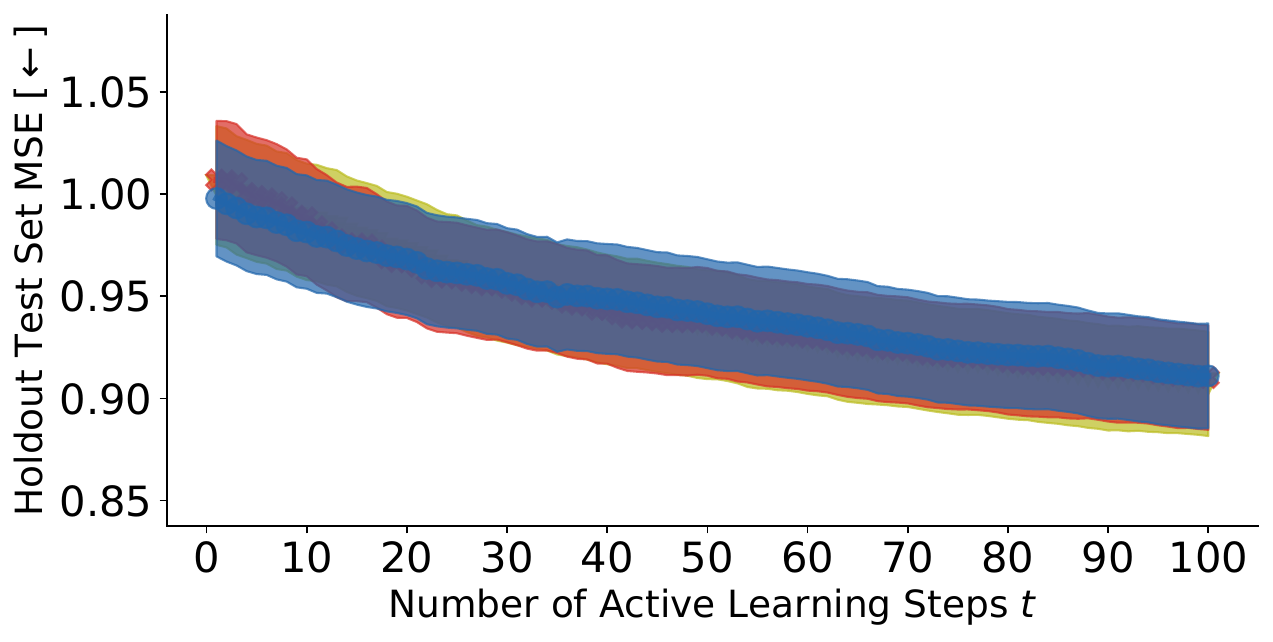}
    \end{subfigure}
    \hspace{1cm}
    \begin{subfigure}{0.27\textwidth}
        \includegraphics[width=\textwidth]{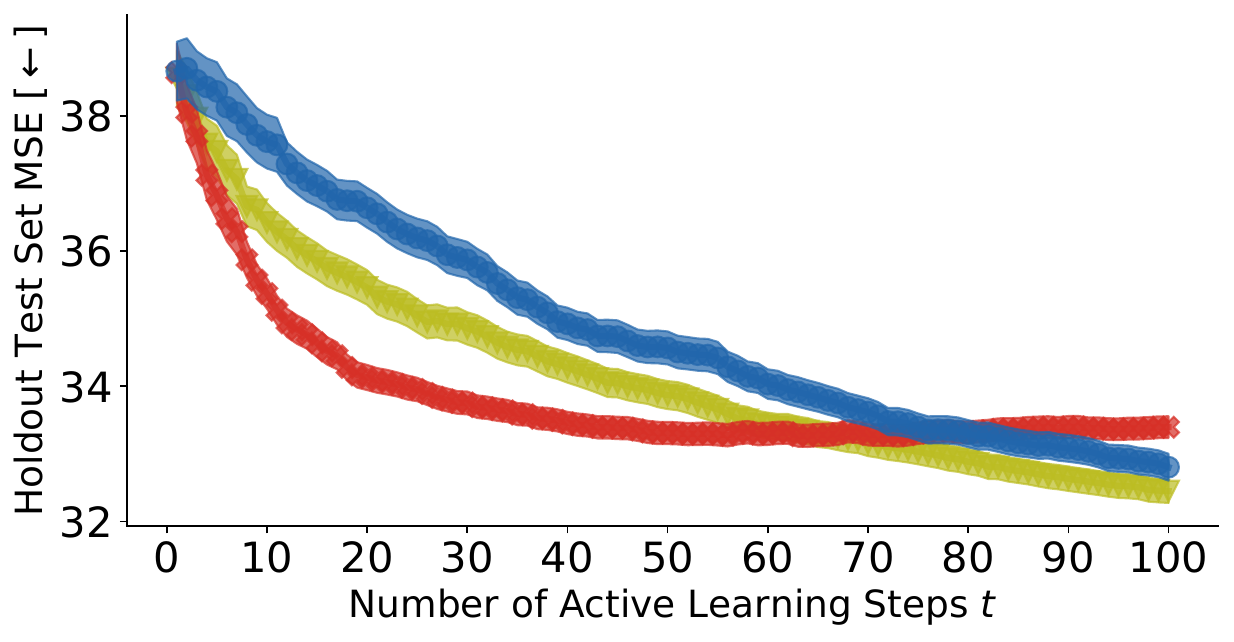}
    \end{subfigure}
    \hspace{1cm}
    \begin{subfigure}{0.27\textwidth}
        \includegraphics[width=\textwidth]{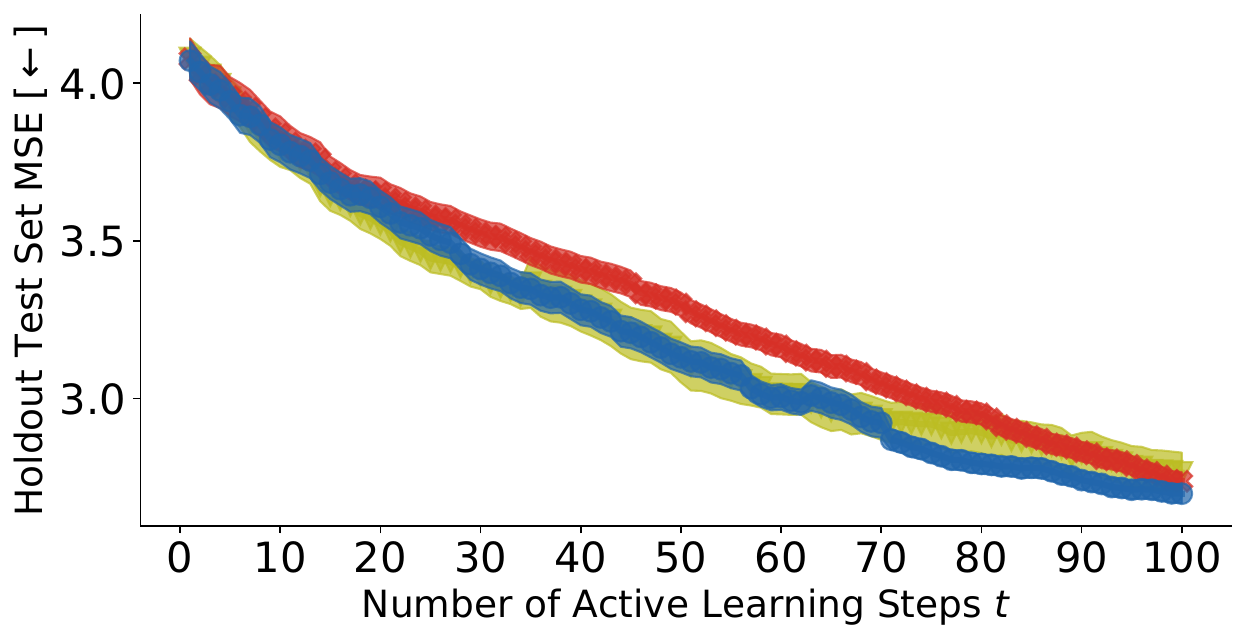}
    \end{subfigure}
    \\
    \begin{subfigure}{0.33\textwidth}
        \centering
        \textbf{Robot Arm Kinematics}
    \end{subfigure}
    \hfill
    \begin{subfigure}{0.33\textwidth}
        \centering
        \textbf{Airfoil}
    \end{subfigure}
    \hfill
    \begin{subfigure}{0.33\textwidth}
        \centering
        \textbf{Healthcare Utilization (MEPS)}
    \end{subfigure}
\caption{Conformal policy control in the active learning setting. We train Gaussian process regression models on tabular datasets, providing a small amount of initial data for training and selecting the remaining training data via exponential tilting toward the posterior predictive variance. We introduce a feasibility constraint based on alignment with the leading principal component of the Gram matrix to make the task more difficult. CPC controls the constraint violation risk at our desired threshold $\alpha = 0.2$, while simultaneously reducing test MSE. Surprisingly, in some cases the risk-controlled data selection policy attains lower test MSE than the uncontrolled policy.}
\label{fig:AL_enumerated_options}
\end{figure*}

\section{Experiments}

We present three experiments here. First, we apply generalized CRC (gCRC) to control the false discovery rate in medical question answering, a non-monotonic loss. Second, we apply CPC to constrained active learning, where feedback-loop shifts break exchangeability. Third, we apply CPC to black-box sequence optimization with a language model. Additionally, Appendix \ref{app:conservative_opt} presents a fourth experimental setting that compares CPC to standard conservative optimization on a constrained Bayesian optimization task.

\subsection{Factual Medical Question-Answering}

\textbf{Setup:} We evaluate CRC methods on the MedLFQA dataset \citep{jeong2024olaph}, which aggregates medical question-answering benchmarks. 
We use the dataset of \citet{cherian2024large}, which contains GPT-3.5-Turbo responses, atomic subclaims extracted via GPT-4o, and factuality annotations obtained by checking each subclaim against reference answers. 
Given a confidence score $p(C_{ij} \mid X_i)$ for each subclaim, we filter claims below a threshold $\tau$ and measure the false discovery rate (FDR), the fraction of retained claims that are false, as our loss. 
We measure recall, the fraction of total true claims retained, as our utility metric. Unlike the binary loss used by \citet{mohri2024language} and \citet{cherian2024large}, FDR is non-monotonic in $\tau$, making it a natural test for gCRC. We compare gCRC against CRC applied to monotonized losses and Learn Then Test (LTT) \citep{angelopoulos2021learn}. See Appendix~\ref{app:medical_qa} for details.

\textbf{Results:} Figure~\ref{fig:MedicalQA} shows FDR and recall across risk levels, $\alpha$. gCRC controls FDR at or below the target level across all $\alpha$ values while achieving recall at least as high, or often much higher, than baselines. 
Relative to gCRC, LTT is conservative because it uses concentration inequalities for a high-probability guarantee. Monotonized-loss CRC is conservative due to Jensen's inequality: monotonizing each loss is far more conservative than monotonizing the empirical risk.
These results demonstrate that gCRC provides valid risk control a non-monotonic loss where standard CRC may not be valid, with higher empirical power than baselines.

\subsection{Constrained Active Learning}
\label{subsec:active_learning_experiments}

\textbf{Setup:}
We evaluate CPC in a pool-based active learning setting on three regression datasets: \textbf{Robot Arm Kinematics}, \textbf{Airfoil}, and \textbf{Healthcare Utilization (MEPS)}. Starting from a training set selected from an initial safe region, the learner iteratively selects records from the pool to minimize test MSE.
For a challenging constrained setting, we construct synthetic feasibility constraints based on the first principal component of the covariate matrix. Records following the dominant covariation pattern (high PC1 values) are likely feasible, while records deviating from this pattern are likely infeasible. 
Uncertainty-based acquisition naturally targets the opposite end of the spectrum.
We fit a Gaussian process regression model and define an acquisition policy via exponential tilting toward posterior variance: $\pi(a) \propto \exp(\lambda \cdot \hat{\sigma}^2_a)$. CPC constrains this policy relative to the source distribution used for initialization. See Appendix~\ref{app:active_learning_details} for details.

\textbf{Results:}
In Figure~\ref{fig:AL_enumerated_options} we present results across all three datasets. CPC successfully controls constraint violation risk at the target $\alpha$ while reducing test MSE. Surprisingly, in some cases the risk-controlled policy attains \textit{lower} test MSE than the unconstrained policy, suggesting that avoiding infeasible regions improves sample efficiency.

\begin{figure*}
    \centering
    \begin{subfigure}{\textwidth}\centering
        \includegraphics[width=0.9\textwidth]{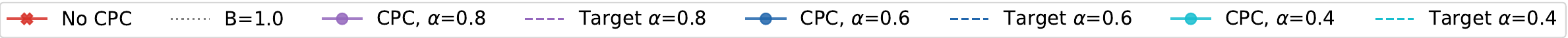}
    \end{subfigure}
    \\
    \begin{subfigure}{0.33\textwidth}
        \includegraphics[width=\textwidth]{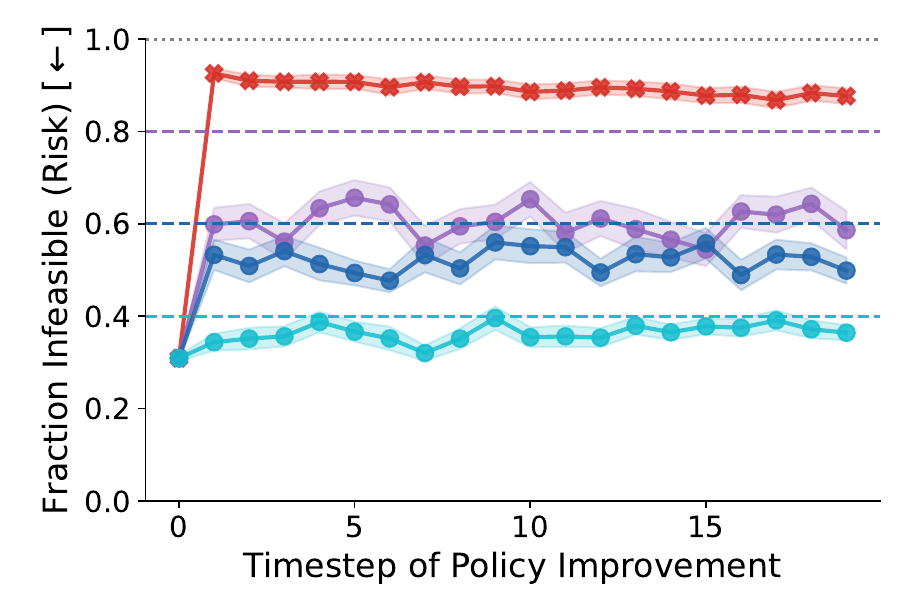}
    \end{subfigure}
    \hfill
    \begin{subfigure}{0.33\textwidth}
        \includegraphics[width=\textwidth]{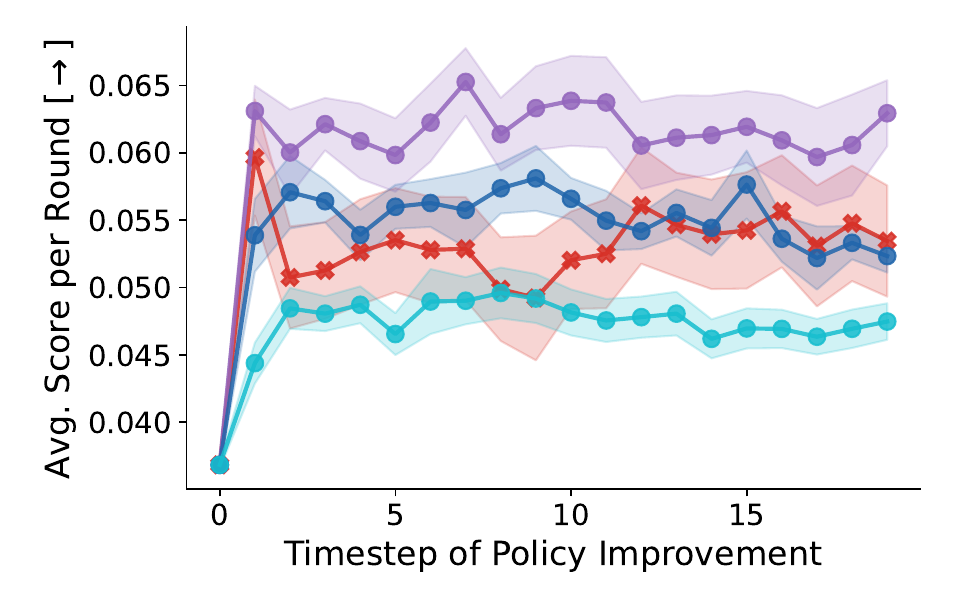}
    \end{subfigure}
    \hfill
    \begin{subfigure}{0.33\textwidth}
        \includegraphics[width=\textwidth]{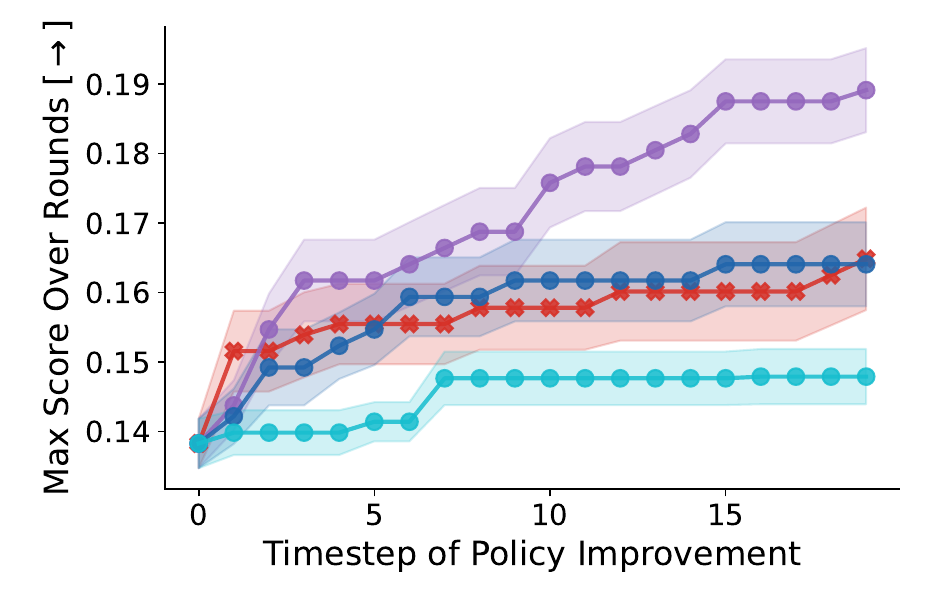}
    \end{subfigure}
    \caption{
        We apply CPC to constrained black-box sequence optimization scored with Ehrlich functions.
        \textbf{Left:} average infeasibility rate per round.
        \textbf{Center:} average objective value of sequences sampled from the policy.
        \textbf{Right:} best objective value found so far. CPC allows direct control over constraint violation risk via $\alpha$. Without CPC, the infeasibility risk quickly rises to over 80\% of policy samples.
        Moderate risk control ($\alpha > 0.6$) stabilizes optimization and improves overall performance by reducing wasted samples on infeasible actions.
    }
    \label{fig:llm_policy_control_results}
\end{figure*}

\subsection{Constrained Black-Box Optimization}
\label{subsec:cpc_LLM_methods}

\textbf{Setup:}
Now consider a more difficult setting, black-box sequence optimization. Here, our goal is to identify a sequence of tokens $a \in \mathcal{X}$ to maximally improve the objective value of a current solution $x \in \mathcal{X}$. We evaluate on Ehrlich functions \citep{chen2025generalists}, synthetic test functions designed to capture the geometric structure of biomolecular sequence optimization problems. We use \textbf{Ehr(32, 32)-4-4-4}: sequences of length 32 over a vocabulary of size 32, with 4 motifs of length 4.
We create an initial data package by running a genetic algorithm from a single feasible solution. We train $\pi_{0}$ by applying SFT to a Pythia 14M-parameter decoder-only transformer \citep{biderman2023pythia} on improving pairs from the GA history. In later rounds, we sample from the current policy, collect oracle feedback, and fine-tune via DPO. CPC constrains each improved policy $\pi_{t}$ before sampling. See Appendix \ref{app:sequence_opt} for details.

\textbf{Results:}
In Figure \ref{fig:llm_policy_control_results} we show that CPC effectively controls the risk of generating infeasible sequences while still improving the objective value over time.
Again we find that moderate risk control ($\alpha > 0.6$) can produce controlled policies that slightly \textit{outperform} their uncontrolled counterparts, likely by wasting fewer evaluations on infeasible actions.

\section{Discussion}

We set out to let an agent choose how far to deviate from a safe policy without the choice biasing the estimate used to justify it. The resolution exploits the fact that the only shift between the safe policy's data and any candidate's deployment is the agent's own: it does not know how the world will respond, but it knows exactly how each candidate would distribute its actions. Conformal policy control parameterizes that choice as a likelihood-ratio threshold and calibrates it directly on the safe policy's data, so that selection is the calibrated act rather than a step that follows it. Because the threshold bounds how far the candidate can deviate from the safe policy, the resulting guarantee applies to the chosen policy itself, not merely each candidate in isolation.

More broadly, CPC instantiates a prescriptive use of conformal methods: rather than characterizing uncertainty under a fixed distribution, it searches over distributions to find one that satisfies a finite-sample guarantee. This inversion applies wherever one selects a data-generating process subject to a risk constraint, from experimental design, to active data collection, to generative model deployment. We hope that the ideas in this paper encourage further development of conformal methods as tools for decision-making, beyond their traditional role in uncertainty quantification.

The guarantees we provide are marginal: risk is controlled on average over context and calibration data draws.
This is the natural guarantee for deployment decisions, where one asks whether a policy is safe to deploy across a population.
It is less informative for individual decisions, where one asks whether a particular action is safe for a particular context.
Stronger conditional guarantees are possible but require additional assumptions or more conservative bounds \citep{thomas2019preventing}.
Our guarantees also depend on the stability of the context distribution.
If the context distribution begins to shift then recalibration may be necessary, or some risk control gap may need to be tolerated \citep{barber2023}.
Monitoring methods based on conformal martingales offer a principled approach to detecting such shifts \citep{prinster2025watch}.
We have also assumed the policy likelihoods are computable in closed form. If these likelihoods are intractable, then density or density-ratio estimation \citep{sugiyama2012density,cranmer2020frontier} may be applicable.

A recurring theme in this work is that assumptions should reflect what the practitioner knows rather than what is theoretically convenient. CPC requires a safe policy, but any baseline whose risk the practitioner is willing to accept will do. 
In general it requires the agent's policies to be smooth, but searching via a likelihood-ratio bound enforces this condition.
It requires a risk tolerance, but as a direct declaration rather than an indirect penalty weight. 
Safety and exploration, in this view, need not run counter to each other. In the right balance they correct each other's deficiencies, avoiding complacency on the one hand and dissipation on the other. The balance must be set from what is actually known, not what is hoped. It can be enough, it turns out, to know what works, and to know what we want to do next.



\section*{Code Availability}

Code to reproduce the results is available at \url{https://github.com/samuelstanton/conformal-policy-control}.

\section*{Impact Statement}

The dominant paradigm for deploying machine learning systems might be characterized as ``train, deploy, and pray''. Collect data, train a model, tune hyperparameters until empirical performance seems acceptable, deploy, and hope that test-time behavior remains within tolerable bounds. When failures occur, the response is typically reactive, patching the system after harm has been observed. This paradigm persists not because practitioners prefer it, but because principled alternatives have been lacking.

This work contributes to a shift from safety-through-patching toward safety-by-design. Rather than discovering failure modes empirically and addressing them post-hoc, conformal policy control allows practitioners to specify acceptable risk levels before deployment and obtain guarantees that these levels will be respected. The hyperparameter tuning phase becomes about improving performance within safety constraints, rather than searching for a configuration that might satisfy unspecified feasibility requirements.

If such methods become widespread, the implications for where and how machine learning is deployed could be substantial. High-stakes domains that have resisted ML adoption due to liability concerns or regulatory barriers, such as clinical decision support, autonomous systems, and financial services, may become more accessible when developers can provide formal guarantees rather than empirical assurances. Such a change in framing would align ML deployment with the safety certification practices already standard in aviation, pharmaceuticals, and other regulated industries.

\section*{Acknowledgments} 

D.P., A.L., and S. Saria were partially funded by the Gordon and Betty Moore Foundation grant \#12128. The authors would like to thank Hannah Lawrence and Neha Verma for helpful feedback. 

\bibliography{references}
\bibliographystyle{icml2026}

\newpage
\appendix
\onecolumn

\begin{center}
{\Large\bfseries Supplementary Material}
\end{center}
\vspace{1em}

\section*{Table of Contents}
\newcommand{\appentry}[3]{\noindent\textbf{#1}\hspace{1em}\hyperref[#2]{#3}\dotfill\pageref{#2}\par\vspace{0.3em}}
\appentry{A}{app:extended_related_work}{Extended Related Work}
\appentry{B}{app:conservative_opt}{Supplemental Constrained Bayesian Optimization Experiments}
\appentry{C}{app:proofs}{Proofs}
\appentry{D}{app:detailed_methodology}{Implementation Details}
\appentry{E}{app:medical_qa}{Medical QA Experiment Details}
\appentry{F}{app:active_learning_details}{Tabular Active Learning Experiment Details}
\appentry{G}{app:sequence_opt}{Sequence Optimization Experiment Details}

\section{Extended Related Work}
\label{app:extended_related_work}

\subsection{Conservative Model-Based Optimization}

The tension between optimization performance and reliability has deep roots in decision theory, from early work on optimal decision-making under uncertainty \citep{raiffa1961applied} to stochastic methods for optimization under noise \citep{kushner1964new}. 
The \textit{optimizer's curse} \citep{smith2006optimizer} captures the key challenge: when decisions are selected by maximizing an estimated value function, the selected action is systematically biased toward outcomes whose value is overestimated. 
Crucially, this bias persists even when the value estimates are unbiased on average and arises from the act of selection itself. Among many uncertain alternatives, those with the largest positive estimation errors are most likely to be chosen. 
In the context of sequential decision making and policy optimization, this effect manifests when an optimized policy is trained to maximize expected reward under limited or noisy feedback, leading it to concentrate probability mass on actions whose apparent advantage is driven by estimation error rather than true performance. 
As a result, policies that appear optimal in expectation may perform poorly when deployed, particularly when they extrapolate beyond regions where reliable evidence is available.
In sequential settings, the effect compounds. Aggressive optimization shifts the data distribution away from regions where estimators are reliable, further degrading the estimates and increasing error for the optimizer to exploit.

Various communities have aimed at guiding models toward high-performing solutions while explicitly constraining distribution shift away from trusted data or policies.
In stochastic control and reinforcement learning, entropy-regularized control and KL-penalized objectives constrain optimized policies to stay close to a reference distribution \citep{todorov2009efficient, fox2015taming}. 
Related formulations arise in path integral control and KL-regularized path measures \citep{kappen2005path, theodorou2010generalized}. 
Trust-region policy optimization methods can be viewed as a practical instantiation that explicitly bound policy updates with local constraints \citep{schulman2015trust, schulman2017proximal}. 
In offline and batch reinforcement learning, conservative objectives penalize actions whose value estimates rely on out-of-distribution extrapolation \citep{kumar2020conservative, trabucco2021conservative}. 
More recently, the control-theoretic perspective has been extended to generative model fine-tuning. Fine-tuning diffusion and flow models to maximize the expected reward while retaining the pretrained model information via KL regularization has been framed as an entropy-regularized optimal control problem \citep{uehara2024fine, domingo2024adjoint}, and generalized to other utilities and divergence measures \citep{de2025flow}. 

Finally, in black-box optimization, trust-region Bayesian optimization restricts search to local regions where surrogate models are reliable \citep{eriksson2019scalable}, while safe Bayesian optimization enforces hard safety constraints to prevent evaluation in unsafe regions \citep{sui2015safe, berkenkamp2016safe, kirschner2019adaptive, berkenkamp2023bayesian}.
All of these methods operationalize safety by bounding how far the optimized distribution may deviate from a known baseline.

In practice, however, the strength of this constraint is governed by a scalar hyperparameter that has no intrinsic semantic interpretation. Whether expressed as a KL penalty weight, a likelihood ratio threshold, or a trust-region radius, this parameter controls the degree of aggressiveness in optimization but does not directly correspond to any declarative risk criterion, such as a target failure rate. As a result, practical deployment requires solving a secondary  \textit{hyperparameter selection problem}: one must select many candidate values of the constraint parameter and empirically evaluate their downstream risk in order to identify a satisfactory operating point. Critically, this calibration must be performed \textit{on-policy}, since the risk of interest depends on the distribution induced by each specific choice of the hyperparameter. This requirement is costly and introduces further feedback-loop dependencies \citep{bergstra2011algorithms, feurer2019hyperparameter}.

Yet the distribution shift these methods seek to control is self-inflicted. The agent chooses its own policy, and therefore knows the likelihood ratio between any candidate policy and the safe baseline in closed form. The distribution shift is not an external unknown but a consequence of the agent's own design. This means the hyperparameter selection problem can be solved with off-policy data from the safe policy alone, provided the calibration is done carefully. Conformal techniques provide exactly this kind of careful calibration.

    \begin{table*}[!t]
    \caption{Summary comparison of selected related work on conformal prediction for decision making under agent-induced shifts.}
    \centering
    {\small
    \begin{tabular}{ c | c | c | c  c  c }
    \toprule
    & & &  \multicolumn{3}{|c}{\textbf{3. Type of Risk Control}}  \\
      $\begin{matrix}
      \textbf{Method Group}\\
      \text{References}
      \end{matrix}$ & $\begin{matrix}
        \textbf{1. Agent-} \\
            \textbf{Induced} \\
           \textbf{Data Shift}
       \end{matrix}$
     & $\begin{matrix}
     \textbf{2. Action} \\
     \textbf{Space}
     \end{matrix}$ & $\begin{matrix}\textbf{Prediction Set} \\ \textbf{Miscoverage} \\ \mathbbm{1}\{Y_t\not\in \hat{C}_{\color{blue}\alpha}(X_t)\} \end{matrix}$ & $\begin{matrix}\lambda\textbf{-Parameterized} \\ \textbf{Losses} \\ L_t(A_t({\color{blue}\lambda}))\end{matrix}$ &
      $\begin{matrix}
          \textbf{Unparameterizable} \\
          \textbf{Losses} \\
          \textbf{($\beta$-Param. Policy)} \\
          L_t(A_t), \ A_t\sim p_t({\color{blue}\beta})
      \end{matrix}$
       \\
     \hline
     $\begin{matrix}\textbf{Conformal Prediction} \\
     \text{\citet{vovk2005algorithmic}}
     \\
     \text{\citet{vovk2018conformal}}
     \end{matrix}$ & \xmark & $\begin{matrix}
         \text{Small/} \\ \text{Finite}
     \end{matrix}$ & \cmark & \xmark & \xmark \\
     \hline
     $\begin{matrix}\textbf{Weighted Conformal Pred.} \\
     \text{\citet{tibshirani2019conformal}}
     \\
     \text{\citet{fannjiang2022conformal}} \\
     \text{\citet{stanton2023bayesian}} \\
     \text{\citet{nair2023randomization}} \\
     \text{\citet{prinster2024conformal}} \\
     \end{matrix}$  & $\begin{matrix}
         \text{Single-Round} \\

         \hline \\
         \textbf{\mygreen{Multi-Round}}
     \end{matrix}$ & $\begin{matrix}
         \text{Small/} \\ \text{Finite} \\ \hline  \text{Contin.} \\ \hline \text{Small/} \\ \text{Finite}
     \end{matrix}$ & \cmark & \xmark & \xmark \\
     \hline
     $\begin{matrix}\textbf{Conformal Risk Control} \\
     \text{\citet{angelopoulos2022conformal}}
     \\
     \text{\citet{lekeufack2024conformal}}
     \end{matrix}$ & $\begin{matrix}
      \\
         \text{Single-Round} \\
         \hline
         \text{Eventually safe}
     \end{matrix}$ & $\begin{matrix}
         \text{Small/} \\ \text{Finite}
     \end{matrix}$ & \cmark & \cmark & \xmark \\
     \hline
     $\begin{matrix} \textbf{Conformal Policy Control} \\
     \text{(Proposed)}
     \end{matrix}$ & \textbf{\mygreen{Multi-Round}} & $\begin{matrix}
         \text{Large/} \\ \boldsymbol{\infty}
     \end{matrix}$ & \cmark & \cmark & \cmark  \\
    \bottomrule
    \end{tabular}
    }
    \label{table:comparisons}
    \end{table*}

\subsection{Conformal Methods}

\textbf{Conformal prediction under distribution shift in high-dimensional spaces}

Conformal prediction, originally developed for exchangeable or i.i.d. data by \cite{vovk2005algorithmic}, has been extended to a variety of distribution shift settings. Closest to the current paper are \textit{weighted} conformal methods that \textit{proactively adapt} to distribution shift by re-weighting calibration data using knowledge of the change, though another line of work develops conformal-inspired methods for \textit{retroactively reacting} to adversarial shifts over time (e.g., \citet{Gibbs2021-ed, zaffran2022adaptive, bastani2022, Angelopoulos_pid}). Among weighted conformal methods, those most relevant to the present work are methods tailored for covariate shift \cite{tibshirani2019conformal, prinster2022jaws, yang2024doubly, jonkers2024conformal}, and especially its extensions in which the training and test data are dependent, such as one-step \cite{fannjiang2022conformal, prinster2023jaws} and multi-step \cite{nair2023randomization, prinster2024conformal} feedback covariate shift, as is the case in the policy optimization setting. Also relevant are the fixed-weight methods of \citet{barber2023}, which bound the coverage violation due to misspecified data-independent weights, and the unifying weighted conformal framework of \citet{barber2025unifying}.

When deviating from the assumption of exchangeable calibration and test data, any analysis of proactive conformal prediction methods invokes quantities characterizing the relationship between the calibration and test distributions, such as their density ratios or divergences.
Therefore, despite the advances in theory, deployment of conformal prediction methods in practice in distribution shift settings has been hamstrung by the practical difficulty of evaluating such quantities, particularly in high-dimensional spaces.
Consider the combinatorially large action spaces involved in high-impact applications such as biological sequence design.
Unless modeling the distribution using factorizations amenable to tractable computation (e.g., autoregressive generative models), it is generally computationally intractable to evaluate the density of distributions over sequence space, as exact computation of normalizing constants is infeasible.
As a workaround, some approaches have resorted to density ratio estimation \cite{stanton2023bayesian, Laghuvarapu2023, fannjiang2025, correia2025neurips}, a task that is also challenging and ill-posed in high dimensions \cite{rhodes2020}, and which nullifies validity guarantees due to estimation error.
In the present work, we circumvent the difficulty of evaluating density ratios between the calibration and test policies by making their density ratio itself the tunable control parameter, as follows. 
We first assume existence of an initial safe policy, $\pi_0$, that satisfies the safety constraint, as well as naively optimized policies at each time, $\pi_t$.
At each time, we define a family of ``constrained policies'' that interpolates between $\pi_0$ and $\pi_t$ by clipping their density ratio, $\pi_t(\cdot) / \pi_0(\cdot)$, at an upper bound, $\beta_t$ (Eq.~\eqref{eq:piecewise_constraint_def}).
By making this density ratio upper bound, $\beta_t$, the tunable control parameter, we have done away with the need to estimate the density ratio altogether.

\textbf{Conformal methods for decision-making in the offline setting}

Conformal methods, by which we mean conformal prediction techniques and their generalizations to settings beyond constructing instance-wise prediction sets with coverage guarantees, have also been developed for various decision-making settings.
\citet{vovk2018conformal} laid the groundwork for conformal decision-making, demonstrating how conformal prediction (specifically conformal predictive distributions \citep{vovk2017nonparametric}) can be used to make decisions that satisfy a notion of efficiency.
They noted, however, that they had ``limited [themselves] to analyzing a given batch of data, without attempting active learning, and this limitation has made it possible to develop a fairly systematic theory.''

Indeed, following \citet{vovk2018conformal}, subsequent work on conformal methods for decision making has largely also assumed an offline (i.e., non-active) setting, where training and calibration are entirely independent of test-time decisions: that is, where pre-deployment training and calibration cannot influence distributions at test time, and vice versa (without breaking statistical guarantees). For instance, see \citet{hullman2025conformal} for a recent overview. Approaches in this offline setting have included conformal prediction sets for individual treatment effect estimates \citep{lei2021conformal}, off-policy prediction \citep{taufiq2022conformal, zhang2023conformal}, extending conformal prediction sets to handle notions of risk beyond miscoverage \citep{bates2021distribution, angelopoulos2022conformal}, risk assessment \citep{prinster2022jaws, singh2023distribution, zhou2025conformalized}, developing connections with Bayesian decision theory \citep{hoff2023bayes, snell2025conformal}, and designing conformal sets for (or the filtering of) outputs of fixed/pre-trained generative models \citep{quach2023conformal, teneggi2023trust, mohri2024language, cherian2024large, overman2024aligning, jiang2025conformal}. Another line of work has developed conformal-inspired optimization procedures to simulate decision making during training \citep{colombo2020training, stutz2021learning, cortes2024utility, yeh2024end, yeh2025conformal}, but notably, these methods all require a further split of the data beyond that of split conformal \citep{papadopoulos2002inductive}, thus effectively side-stepping the ``optimizer's curse'' at the cost of reduced data efficiency.\footnote{That is, 
these ``conformal training'' works typically require at least three disjoint splits of the data to maintain statistical validity: a ``proper'' training set for learning, a validation set used for hyperparameter tuning (or simulating conformal calibration), and a separate fresh calibration set at test time for inference.
In contrast, the present paper's conformal policy control methods avoid such an additional data split by carefully reweighting the calibration data to adjust for data-dependent policy selection.}

While \citet{kiyani2025decision} theoretically demonstrated that prediction sets are particularly suitable for risk-averse decision-making and several works have empirically demonstrated the utility of prediction sets for human decision makers \citep{straitouri2023improving, cresswell2024conformal, zhang2024evaluating}, there nonetheless remains a gap between methodological progress and practical impact of conformal methods \citep{hullman2025conformal}. Arguably, this persistent gap is due to the fact that predictions (and consequently, prediction sets) are almost never the end goal in practice. Instead, the end goal in practice is often to control how prediction error affects the specific downstream decisions, or the policy, of interest.

\textbf{Conformal methods for decision-making in the active setting}

In the active setting, where an agent's actions may influence the observed data distribution, there are two main threads of literature that parallel corresponding methods for conformal prediction under distribution shift. The first thread uses ideas from weighted conformal to proactively adapt prediction sets to agent-induced shifts: \citet{fannjiang2022conformal} extended the weighted conformal methods of \citet{tibshirani2019conformal} to a one-step instance of  ``feedback covariate shift,'' where the input distribution of a single test batch can depend on the training data. 
Inspired by this work, \citet{stanton2023bayesian} incorporated weighted conformal methods into Bayesian optimization, aiming to close the decision-making loop by allowing conformal calibration and optimization to bidirectionally influence each other; however, they noted gaps in the theory that prevented the desired guarantees from being realized.  \citet{prinster2023jaws} extended \citet{fannjiang2022conformal}'s methods to cross-validation-style conformal calibration; \citet{nair2023randomization} introduced an extension to multiple rounds of dependent covariate shifts; and \citet{prinster2024conformal} extended weighted conformal methods to any data distribution, with multiround agent-induced covariate shifts as a main focus. However, these prior works differ from the present paper in several ways: they all operate \textit{descriptively} to adapt prediction sets for a fixed policy, whereas our methods function \textit{prescriptively}, using conformal calibration to \textit{find} a risk-controlling policy; in combinatorial action spaces, they are bottlenecked by challenges of density-ratio estimation and normalization; and they focus on (mis)coverage rather than more general notions of risk.

The second line of conformal methods in the active setting uses ideas from online adversarial conformal (e.g., \citet{Gibbs2021-ed}) to retroactively learn decision rules over time: \citet{feldman2023achieving} extended the online algorithm of \citet{Gibbs2021-ed} to more general loss functions (similarly as \citet{bates2021distribution, angelopoulos2022conformal} did in the offline setting), and \citet{lekeufack2024conformal} extended those ideas to a conformal decision theory framework, where the online learning update directly steers an agent's decisions at the next timestep. \citet{zhang2024bayesian} use ideas from online adversarial conformal for Bayesian optimization and \citet{zhao2024conformalized} for imitation learning. These works achieve bounds on the long-term average risk as the number of time steps grows indefinitely.
In robotics, conformal prediction has also been used to construct predictive sets of trajectories, which are used to make planning or control procedures safer \cite{Lindemann2023-aj, Luo2024-na, sun2023, wang2025neurips, lindemann2024formal}.

\textbf{Conformal selection and high-probability risk control}

Along a different vein, motivated by candidate selection problems that arise in drug discovery, conformal selection methods \cite{Jin2022, Jin2023-bw, Bai2024-bb, gui2025} merge ideas from conformal prediction and multiple testing to select designs from a pool of individual candidates, such that the false discovery rate, or expected proportion of selected designs whose label does not surpass a desired threshold, is controlled.
While these methods assume access to a pool of exchangeable (e.g., i.i.d.) candidates, in some applications---for example, biological sequence design---this is not a natural setting.
Instead, it is more relevant to reason about generating candidates from some distribution, which brings us to our problem of how to select a generative model that controls risk as desired.
\citet{fannjiang2025} tackle this problem as a multiple testing problem to arrive at high-probability guarantees of risk control. 
Such high-probability risk control guarantees are akin in spirit to prior works on risk controlling prediction sets \citep{bates2021distribution} and the Learn-then-Test framework \citep{angelopoulos2021learn}, which provide high-probability guarantees for monotonic and non-monotonic losses respectively. In contrast, we build on 
techniques from conformal risk control \cite{angelopoulos2022conformal} to achieve guarantees in expectation, which typically provides tighter risk control. Relative to conformal risk control, we also introduce and focus on the policy control setting, wherein the control parameter directly controls the agent's policy, to fully connect the decision-making loop, circumvent challenges of density-ratio estimation, and address the optimizer's curse head-on.

\subsection{Seldonian Algorithms}

Seldonian algorithms \citep{thomas2019preventing} provide a framework for constructing machine learning systems that satisfy user-specified behavioral constraints with high probability.
The approach builds on earlier work on high-confidence policy improvement \citep{thomas2015highconfidence}, which introduced the use of concentration inequalities to obtain probabilistic guarantees on policy performance.
Rather than requiring users to tune abstract regularization weights or penalty coefficients, Seldonian algorithms allow direct specification of constraints such as ``the probability of causing harm should be at most 5\%.''
The framework uses concentration inequalities to derive high-probability bounds on constraint satisfaction, returning a solution only when sufficient evidence exists that the constraints will be met.

The Seldonian framework has been extended in several directions.
\citet{metevier2019offline} adapted the approach to contextual bandits with fairness constraints, demonstrating how user-specified fairness definitions can be enforced with high-probability guarantees.
\citet{chandak2020towards} addressed non-stationary MDPs by incorporating time-series forecasting to anticipate distribution shift.
\citet{giguere2022fairness} extended the framework to handle demographic shift between training and deployment, providing guarantees that remain valid when the population composition changes.

Both Seldonian algorithms and conformal policy control share the goal of translating user-specified risk tolerances directly into algorithmic constraints, eliminating the need for trial-and-error hyperparameter tuning.
However, the two frameworks differ in the nature of their guarantees and the distinction between \textit{certification} and \textit{regulation}.
Seldonian algorithms provide \textit{conditional, high-probability} bounds: with probability at least $1 - \delta$, the deployed policy satisfies the specified constraints.
In contrast, conformal methods provide \textit{marginal, expected value} guarantees: on average, the expected loss under the deployed policy is bounded by the user-specified level $\alpha$.
Neither guarantee type strictly dominates the other. High-probability bounds are more conservative but offer protection against worst-case outcomes. Expected value bounds are tighter on average but permit occasional constraint violations.
The greatest difference between our work and Seldonian algorithms is our method can be used to safely deploy (control the expected risk of) a model that would not be safe otherwise through the rejection sampling regulation procedure. 
By contrast, a Seldonian algorithm can certify that the solution it returns satisfies its constraints with high probability, however it cannot regulate an existing solution trained by some third-party algorithm.

\clearpage

\section{Comparison to Conservative Optimization on Constrained Bayesian Optimization Task}
\label{app:conservative_opt}

Whereas our previous experiments compared to conformal baselines and other methods with distribution-free guarantees, in Figure \ref{fig:app_conservative_opt_TownsendFunc} we report an additional comparison to a standard conservative optimization baseline in a Constrained Bayesian Optimization (CBO) setting. Standard conservative optimization baselines have two key limitations relative to CPC: First, they typically assume a particular model class, whereas CPC does not; Second, they typically cannot translate a user’s declared risk tolerance, $\alpha$, into imperative instructions for the algorithm, such as a specific divergence penalty or constraint parameter that achieves risk control. In practice, these limitations mean that extensive hyperparameter tuning is often required for such standard methods to meet desired operating constraints. These experiments illustrate that CPC can circumvent this hyperparameter tuning challenge. Moreover, CPC can be viewed as complementing rather than competing with standard safe exploration methods, since any policy for exploration can be used as the “optimized policy” in CPC.


\textbf{Experiment Details:} Figure \ref{fig:app_conservative_opt_TownsendFunc} compares CPC-controlled CBO to the CBO baseline \citep{balandat2020botorch} on a synthetic constrained Bayesian optimization task with unknown binary feasibility constraints. We evaluate across temperature settings for the CBO’s probabilistic acquisiftion sampling to simulate an attempt at hyperparameter-tuning the CBO baseline. The data are 2-dimensional and continuous, and the unknown objective/reward function is the synthetic Townsend function from \citet{townsend2014constrained} with unknown binary feasibility constraints. Figure \ref{fig:app_Townsend_objective} illustrates the Townsend objective function with an overlay of the Gaussian distribution used as the safe policy in these experiments. The Gaussian safe policy has most of its density inside the objective function's feasible region but away from the true optimum to simulate a safe but suboptimal initialization, as is common in practice (in particular, the Gaussian is centered at (0, -1) and with a standard deviation of 12/13, or the range of the search space divided by 6.5). 
We use a Gaussian process from BoTorch \citep{balandat2020botorch} to model the objective function and a classifier to model the unknown constraint. The CBO method's acquisition is done with probabilistic sampling from temperature-scaling on the qLogExpectedImprovement function in BoTorch \citep{balandat2020botorch}.  A uniform random baseline is provided for reference (gray). Our implementation of the CBO baseline \citep{tian2024boundary} was based on the following BoTorch notebook: \url{https://botorch.org/docs/notebooks_community/clf_constrained_bo/}

\textbf{Results:} Across temperature settings of 2.0 (left column), 1.0 (middle column), and 0.1 (right column), the CPC-controlled CBO method controls the constraint violation rate well below the target level of alpha=0.4 (blue, top row). In contrast, the CBO baseline has much higher constraint violation rates above alpha=0.4; the baseline's violation rate  in early policy improvement rounds appears to slightly decrease with decreased sampling temperatures (red, top row), but not as a function of alpha, so further hyperparameter tuning would be required to find an acceptable temperature value (if any exists). Notably, the CPC+CBO method attains comparable or even improved reward over the CBO baseline without CPC across settings (bottom row), possibly due to wasting fewer evaluations on infeasible samples.

\begin{figure}
    \centering
    \includegraphics[width=0.5\linewidth]{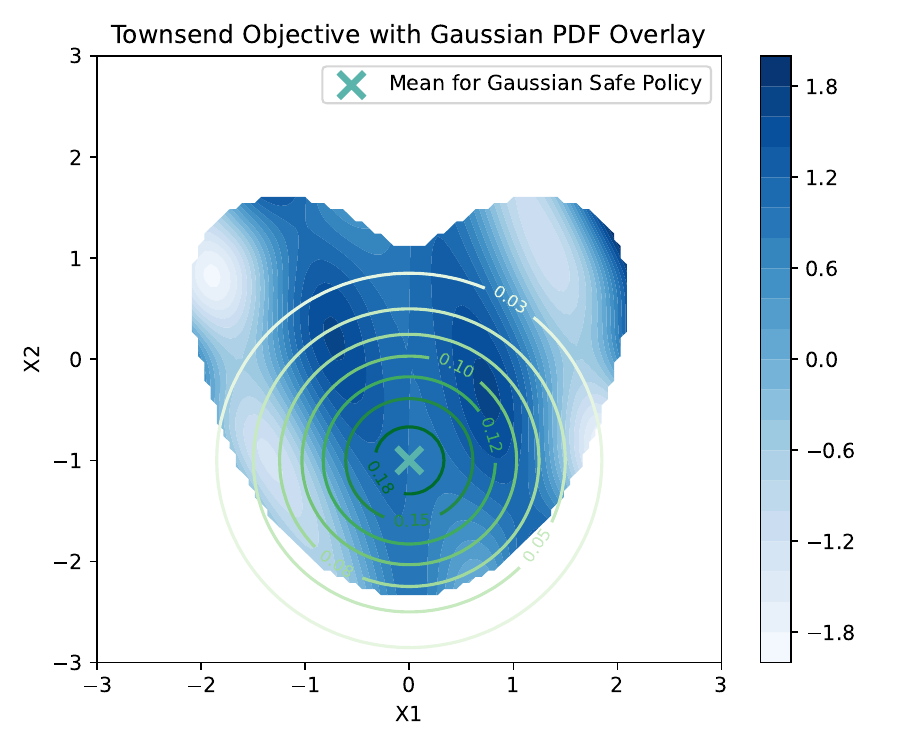}
    \caption{Townsend objective function and contours of Gaussian safe policy used in constrained Bayesian optimization experiments.}
    \label{fig:app_Townsend_objective}
\end{figure}

\begin{figure*}
    \begin{subfigure}{\textwidth}
        \includegraphics[width=\textwidth]{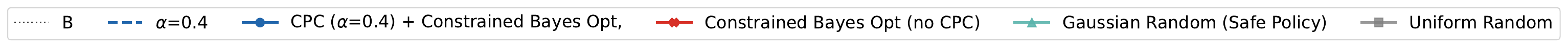}
    \end{subfigure}
    \\
    \begin{subfigure}{0.3\textwidth}
        \centering
        \textbf{Temperature = 2.0}
    \end{subfigure}
    \hfill
    \begin{subfigure}{0.3\textwidth}
        \centering
        \textbf{Temperature = 1.0}
    \end{subfigure}
    \hfill
    \begin{subfigure}{0.3\textwidth}
        \centering
        \textbf{Temperature = 0.1}
    \end{subfigure}
    \\
    \begin{subfigure}{0.3\textwidth}
        \includegraphics[width=\textwidth]{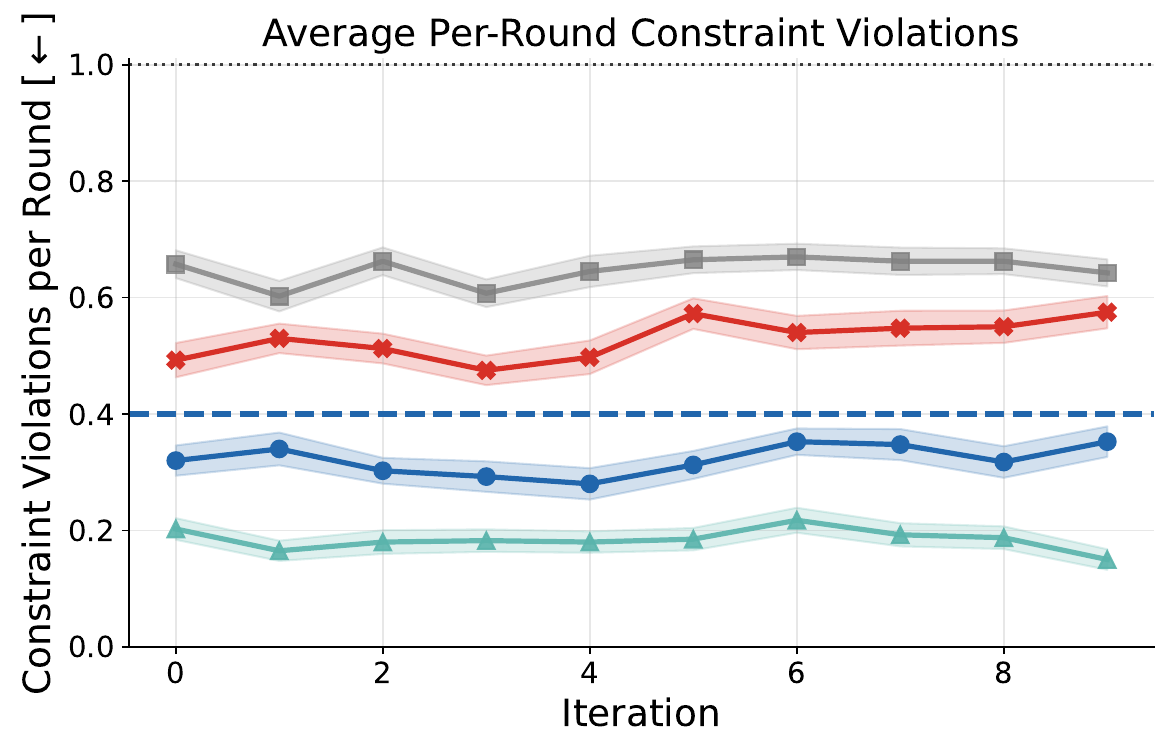}
    \end{subfigure}
    \hfill
    \begin{subfigure}{0.3\textwidth}
        \includegraphics[width=\textwidth]{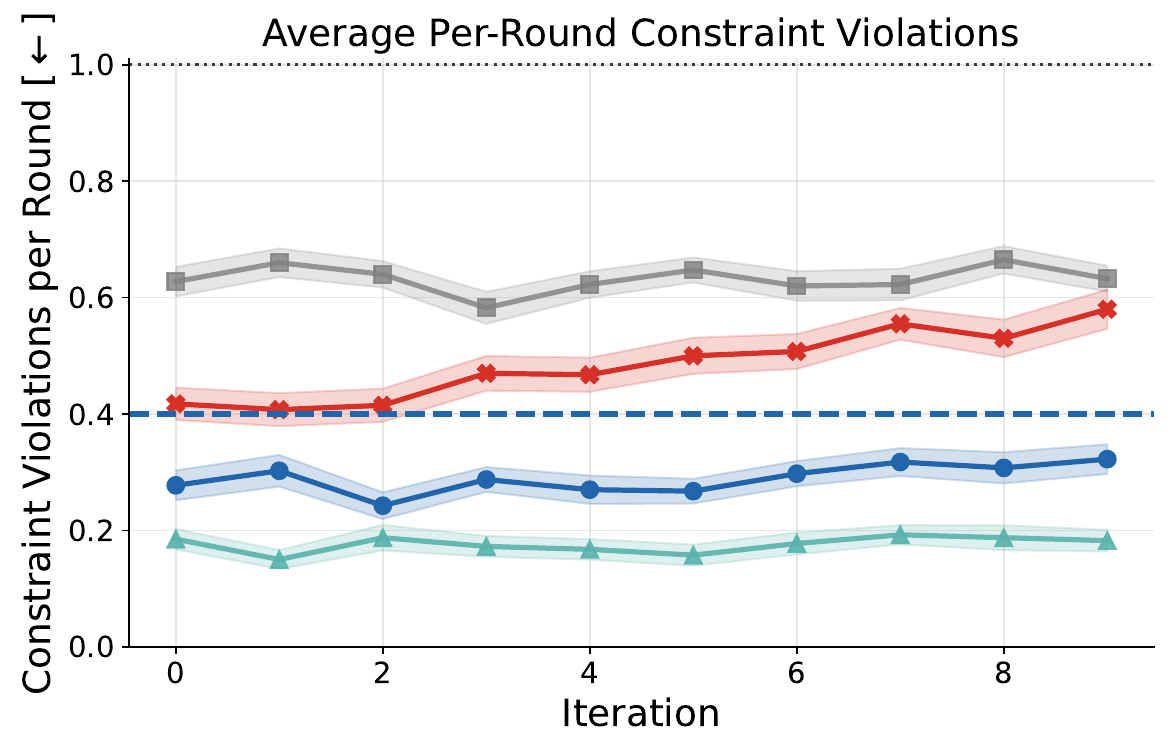}
    \end{subfigure}
    \hfill
    \begin{subfigure}{0.3\textwidth}
        \includegraphics[width=\textwidth]{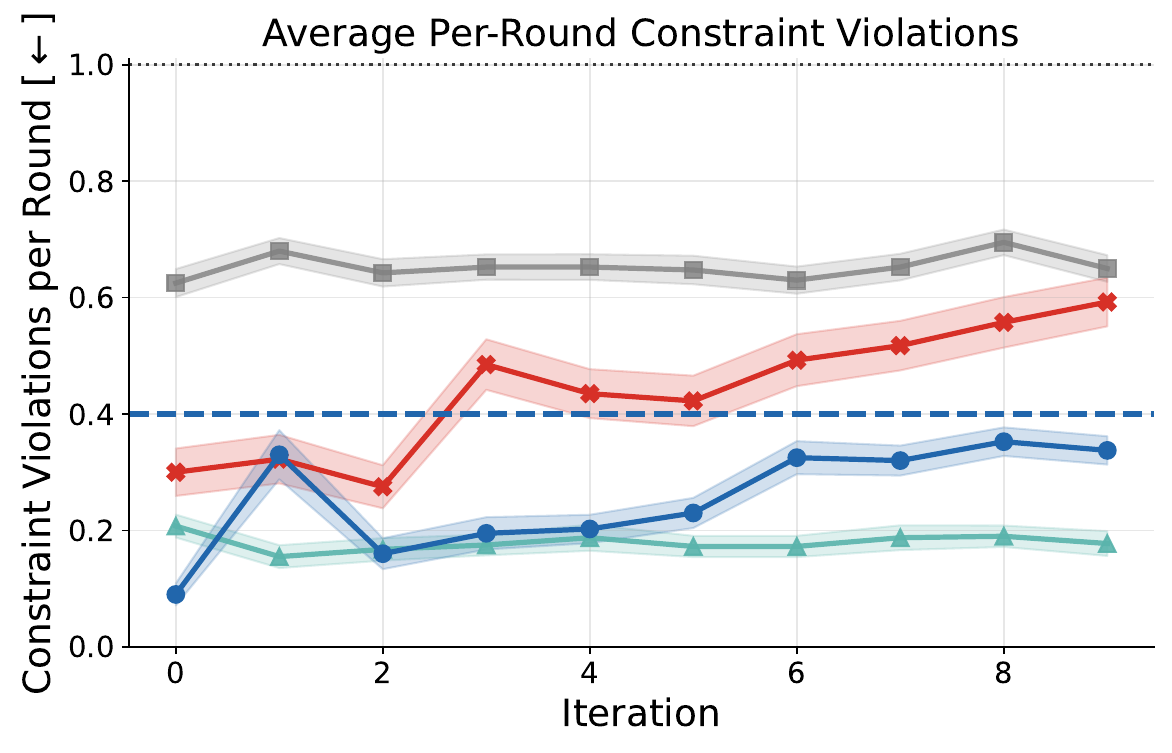}
    \end{subfigure}
    \\
    \begin{subfigure}{0.3\textwidth}
        \includegraphics[width=\textwidth]{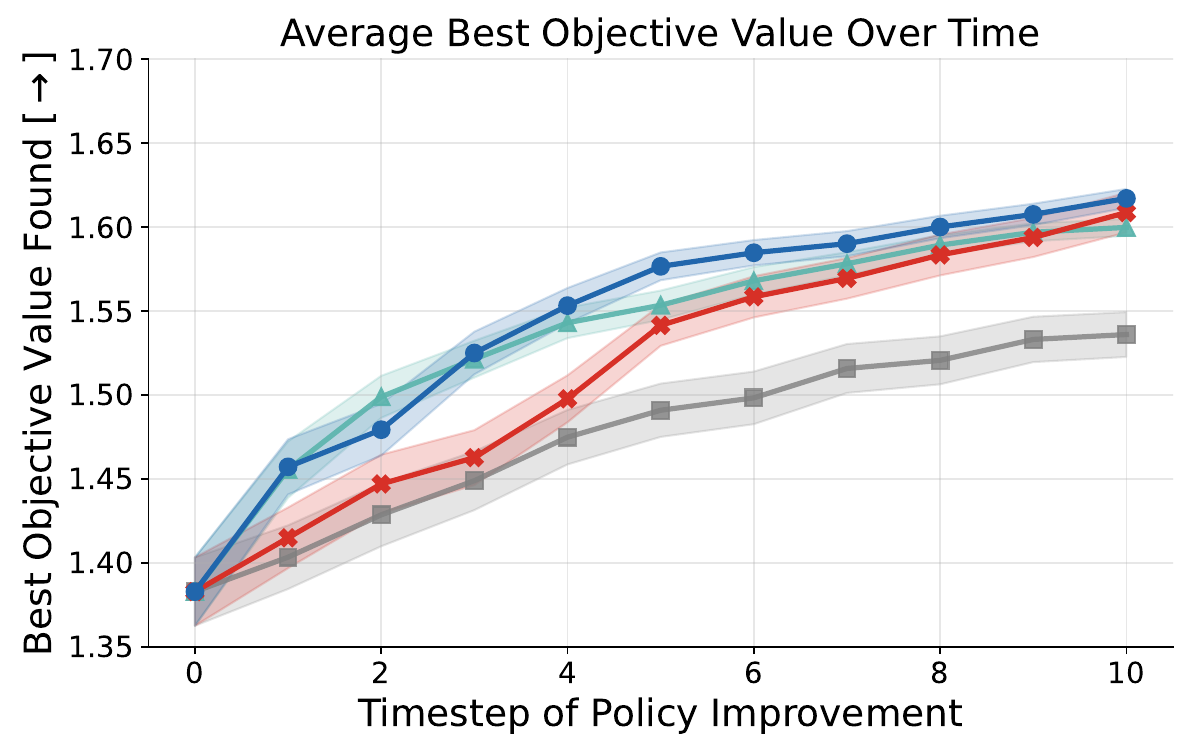}
    \end{subfigure}
    \hfill
    \begin{subfigure}{0.3\textwidth}
        \includegraphics[width=\textwidth]{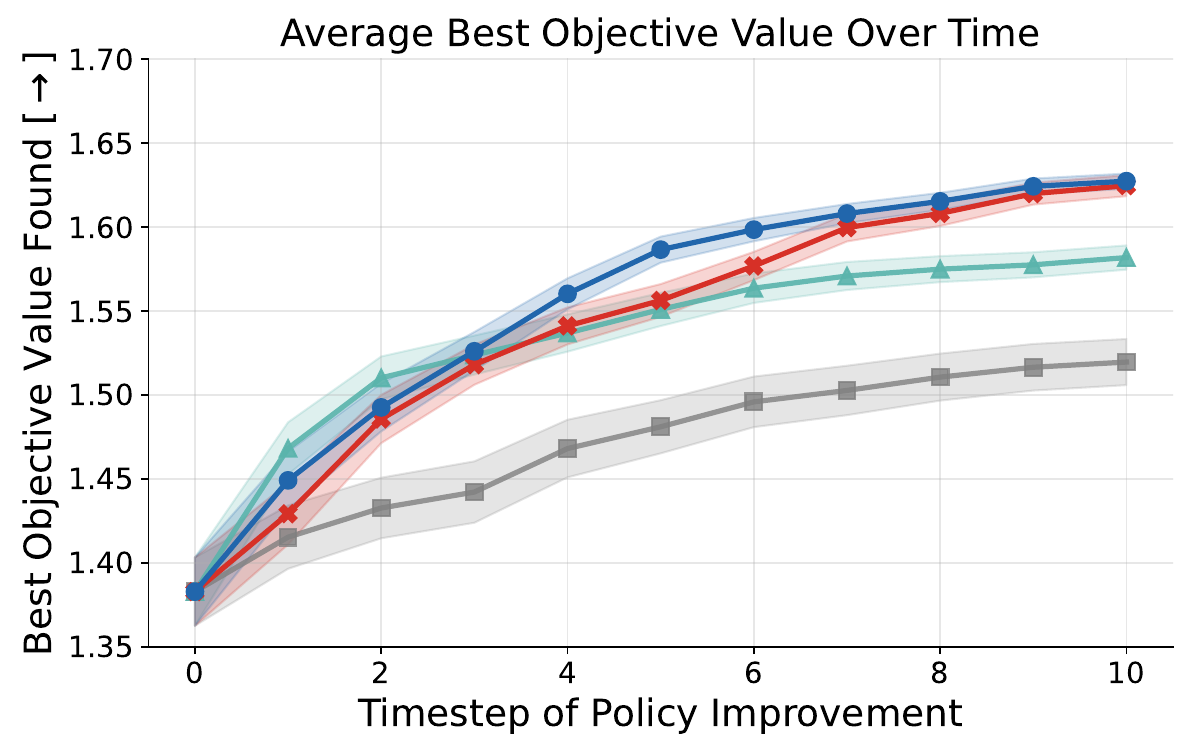}
    \end{subfigure}
    \hfill
    \begin{subfigure}{0.3\textwidth}
        \includegraphics[width=\textwidth]{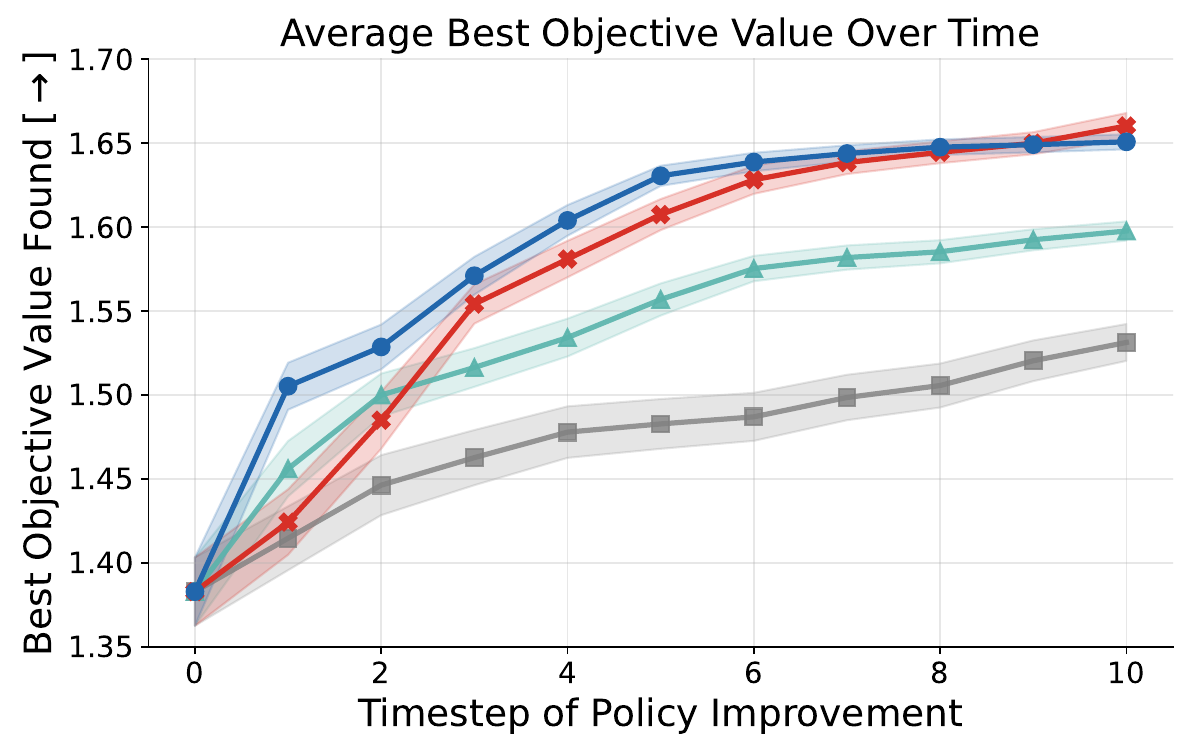}
    \end{subfigure}
    \caption{
    Empirical evaluation of a recent Constrained Bayesian Optimization (CBO) method \citep{tian2024boundary} with CPC (blue) and without CPC (red) over various acquisition sampling temperature settings. Results: Across temperature settings of 2.0 (left column), 1.0 (middle column), and 0.1 (right column), the CPC-controlled CBO method controls the constraint violation rate below the target level of alpha=0.4 (blue, top row). In contrast, the standard CBO baseline has much higher constraint violation rates above alpha=0.4; the baseline's violation in early policy improvement rounds appears to slightly decrease with reduced sampling temperatures (red, top row), but not as a function of alpha, so hyperparameter tuning would be required to find an acceptable temperature value (if one exists). The CPC+CBO method attains comparable or slightly improved reward over the CBO baseline without CPC across settings (bottom row), possibly due to wasting fewer evaluations on infeasible samples. 
    }
    \label{fig:app_conservative_opt_TownsendFunc}
\end{figure*}

\clearpage

\section{Proofs}
\label{app:proofs}

\subsection{Proofs for Generalized Conformal Risk Control (Parameterized Losses, with Exchangeable Data)}
\label{app:gcrc_proof_nonmonotonic_1dim}

\subsubsection{Notation and Definitions}

\begin{itemize}
    \item \textbf{Data (sequences of random variables and observed values):} For all $i\in \mathbb{N}$, denote the random variable for the $i$-th data sample as $Z_i:=(X_i, Y_i)\in \mathcal{X}\times\mathcal{Y}=\mathcal{Z}$, and denote observed value of that random variable in lower case as $z_i:=(x_i, y_i)\in \mathcal{X}\times\mathcal{Y}=\mathcal{Z}$. Similarly, denote a sequence of random variables as $Z_{1:m}:=(Z_1, ..., Z_m)$ and a sequence of observed data values as $z_{1:m}:=(z_1, ..., z_m)$ for any $m \in \mathbb{N}$. To concisely denote indices, let $[m]:=\{1, ..., m\}$.
    \item \textbf{Hyperparameter space:} Let $\Lambda \subseteq \mathbb{R}$ be the space of hyperparameters. So, $\lambda\in \Lambda$ denotes a particular hyperparameter in this space.
    \item \textbf{Loss functions:} For some real-valued function $g :\mathcal{Z}\times \Lambda \rightarrow \mathbb{R}$, and for all $\lambda\in\Lambda$ and all $i\in \mathbb{N}$, let $L_i(\lambda):= g(Z_i,\lambda)$ denote the loss function determined by a random variable $Z_i$ and parameter $\lambda\in \Lambda$, and let $\ell_i(\lambda):= g(z_i,\lambda)$ denote the loss function determined by the observed value $z_i$ and parameter $\lambda\in \Lambda$. To emphasize when no particular $\lambda$ has yet been given as an argument, we will use $L_i(\cdot):=g(Z_i, \cdot)$ to denote the loss function with randomness over $Z_i$ not yet determined, and we will use $\ell_i(\cdot):= g(z_i, \cdot)$ to denote the loss function for observed value $z_i$. 
    \item \textbf{Bags:} For any sequence $v_{1:m}:=(v_1, ..., v_m)$, let $\lbag v_{1:m} \rbag$ denote the \textit{bag} (or multiset) containing the values of that sequence $v_{1:m}$ (but importantly, $\lbag v_{1:m} \rbag$ does \textit{not} contain the information about the order of $v_{1:m}$). For example, $\lbag (1, 7, 42, 7) \rbag$ conveys that there is one 1, two 7s, and one 42 in the bag, but it does \textit{not} convey the order that these appeared in the sequence $(1, 7, 42, 7)$. In particular, with this notation, for the random variables $Z_{1:m}$, let $\lbag Z_{1:m} \rbag$ denote the bag containing those random variables; for observed values $z_{1:m}$, let $\lbag z_{1:m} \rbag$ denote the bag containing those observed values; and thus $\lbag Z_{1:m} \rbag = \lbag z_{1:m} \rbag$ denotes the event of observing $\lbag z_{1:m} \rbag$, in other words, that each $Z_i$ takes a different value in $\lbag z_{1:m} \rbag$, but it is not yet known \textit{which} value (for all $i\in [m]$).
    \item \textbf{Distributions:} Let $F_{Z_{1:m}}:=F_{Z_1, ..., Z_m}$ denote the joint distribution function for $Z_1, ..., Z_m$, and let $F_{\lbag Z_{1:m}\rbag}:=F_{\lbag Z_1, ..., Z_m\rbag}$ denote the (pushforward) distribution induced by the mapping $(z_1, ..., z_m)\rightarrow \lbag z_{1:m} \rbag$. 
\end{itemize}

\begin{definition}[$\epsilon$-replace-one stability]
    \label{appeq:def_replace_one_stability}
    Let $z_{1:n+1}$ denote the observed values, so that $\lbag Z_{1:n+1}\rbag = \lbag z_{1:n+1}\rbag$. We say a function  $h$ is $\epsilon$-replace-one stable if for all $i, j\in [n+1]$, 
    \begin{align*}
        \mathbb{E}\Big[\big|h(\lbag Z_{1:n+1}\rbag\backslash \{z_i\})-h(\lbag Z_{1:n+1}\rbag\backslash \{Z_j\})\big|\Big] \leq \epsilon.
    \end{align*}
\end{definition}

\begin{remark}[Interpretation of $\epsilon$-replace-one stability]
\label{rem:ro_stability_interpret}
    The mixed lower and upper cases in Definition \ref{appeq:def_replace_one_stability} are intentional. The interpretation is clarified by expanding the expectation via the law of total expectation over the event $\lbag Z_{1:n+1}\rbag = \lbag z_{1:n+1}\rbag$:
    \begin{align*}
        \mathbb{E}\Big[&\big|h(\lbag Z_{1:n+1}\rbag\backslash \{z_i\})-h(\lbag Z_{1:n+1}\rbag\backslash \{Z_j\})\big|\Big] \\ = 
        &\mathbb{E}\Big[\mathbb{E}\Big[\big|h(\lbag Z_{1:n+1}\rbag\backslash \{z_i\})-h(\lbag Z_{1:n+1}\rbag\backslash \{Z_j\})\big|\; \mid \; \lbag Z_{1:n+1}\rbag = \lbag z_{1:n+1}\rbag\Big]\Big],
    \end{align*}
    where in the inner expectation both $z_i$ and $Z_j$ are defined conditional on the event $\lbag Z_{1:n+1}\rbag = \lbag z_{1:n+1}\rbag$. For example, let $n+1=4$, assume the bag is $\lbag z_{1:4}\rbag = \lbag (1, 7, 42, 7) \rbag$, consider the selected value $z_i = 42$, and let $Z_j = Z_4$ denote the last random variable. Conditional on the bag of unordered values $\lbag (1, 7, 42, 7) \rbag$, one can deterministically select the value $z_i = 42$ from the bag, but it is not yet known which value is attained by the random variable, $Z_j = Z_4$ (i.e., it could take any value in $\lbag (1, 7, 42, 7) \rbag$). We will specifically apply this assumption with $i=j=n+1$, where the inner expectation is over the randomness of $Z_{n+1}$ conditional on the bag. The condition recovers the standard average replace-one stability assumption (e.g., see Chapter 13 of \citet{shalev2014understanding}).
\end{remark}

\clearpage

\subsubsection{Proof for Main Generalized Conformal Risk Control Result}

\begin{theorem}[Restatement of Theorem \ref{thm:gcrc_nonmonotonic}]
    Assume exchangeable $L_i(\lambda)$. Define  $\lambda_{\max}:=\sup\Lambda\in \Lambda$ and assume 
    \begin{align*}
        L_i(\lambda_{\max})\leq \alpha, \qquad \sup_{\lambda}L_i(\lambda)\leq B < \infty \quad \text{almost surely}.
    \end{align*}
    If the $L_i(\lambda)$ are $K$-Lipschitz in $\lambda$ and $\hat\lambda_+$ is $\epsilon$-replace-one stable, then
    \begin{align*}
        \mathbb{E}\big[L_{n+1}(\hat\lambda_+)\big] \leq \alpha + K\epsilon.
    \end{align*}
\end{theorem}

\begin{proof} \qquad

\textbf{Step 0: Law of total expectation over draw of bag.}

Expanding $\mathbb{E}[L_{n+1}(\hat{\lambda}_{+})]$ via the law of total expectation over the event $\lbag Z_{1:n+1} \rbag = \lbag z_{1:n+1} \rbag$, we have
\begin{align}\label{eq:gcrc_exchangeable_step_0}
    \mathbb{E}\big[L_{n+1}(\hat{\lambda}_{+})\big] = \mathbb{E}_{F_{\lbag Z_{1:n+1}\rbag}}\Big[\mathbb{E}\Big[L_{n+1}(\hat{\lambda}_{+}) \mid \lbag Z_{1:n+1} \rbag = \lbag z_{1:n+1} \rbag\Big]\Big].
\end{align}
We can thus see that, to prove the desired result, it is sufficient to upper bound $\mathbb{E}\big[L_{n+1}(\hat{\lambda}_{+}) \mid \lbag Z_{1:n+1} \rbag = \lbag z_{1:n+1} \rbag\big]$ almost surely given any draw of the bag, $\lbag Z_{1:n+1}\rbag\sim F_{\lbag Z_{1:n+1}\rbag}$. 

\vspace{1cm}

\textbf{Step 1: Deterministic inequalities on bag-conditional risks.}

Define the following family of thresholds:
\begin{align*}
    \hat{\lambda}_+(\lbag L_{1:m} \rbag, \gamma) = \inf \Big\{\lambda_0 \in \Lambda \quad : \quad \forall \qquad \text{fixed } \lambda \geq \lambda_0, \qquad \sum_{i=1}^{m}\frac{l_{i}(\lambda)}{m} \leq \gamma \Big\}.
\end{align*}
When the set is empty, define $\hat{\lambda}_+(\lbag L_{1:m} \rbag, \gamma)=\lambda_{\max}$.

By the right-continuity of $L_i(\cdot)$ (for all $i\in \mathbb{N}$), the sum $\sum_{i=1}^m\frac{L_m(\lambda)}{m}$ is right-continuous for any $Z_{1:m}$, which gives us the following deterministic inequality for any $\gamma$:
\begin{align}
    \label{eq:gcrc_exchangeable_step_1}
    \sum_{i=1}^m\frac{L_i(\lambda)}{m} \leq \gamma \qquad \forall \qquad \text{fixed } \lambda \geq \hat{\lambda}_{+}(\lbag L_{1:m}\rbag; \gamma).
\end{align}

\vspace{1cm}

\textbf{Step 2: Showing an oracle procedure deterministically controls the bag-conditional risk.}

Using the family of thresholds we defined above and applying it to the bag of loss functions for both the calibration set and the test point, define the oracle threshold as
\begin{align}\label{eq:oracle_lambda_def}
    \lambda_+^* := \hat{\lambda}_+(\lbag L_{1:n+1} \rbag, \alpha).
\end{align}

Now, consider the quantity $\mathbb{E}\big[L_{n+1}(\lambda_{+}^*) \mid \lbag Z_{1:n+1} \rbag = \lbag z_{1:n+1} \rbag\big]$, which we can view as the bag-conditional oracle risk. That is, we can interpret this quantity as the expected value that the random variable $L_{n+1}(\lambda_{+}^*)$ will take on, given the bag of data $\lbag z_{1:n+1} \rbag$. Note that conditioning on the bag of data, $\lbag z_{1:n+1} \rbag$, implies the event of observing the bag of loss functions, $\lbag l_{1:n+1} \rbag$,\footnote{That is, implies the event $\lbag L_{1}(\cdot), ..., L_{n+1}(\cdot)\rbag = \lbag l_1(\cdot), ..., l_{n+1}(\cdot) \rbag$, because the loss functions are defined as $l_i(\cdot):=g(z_i, \cdot)$ for all $i\in \mathbb{N}$.} which thereby determines the value of $\lambda_{+}^*$, since $\lambda_{+}^*$ is a symmetric function of $\lbag z_{1:n+1}\rbag$. 

This fact, that $\lambda_{+}^*$ is a symmetric function of (and thus determined by) the bag of data, implies that $L_{n+1}(\lambda_{+}^*)$ will take a value in $\lbag l_1(\lambda_{+}^*), ..., l_{n+1}(\lambda_{+}^*) \rbag$ and the specific value $L_{n+1}(\lambda_{+}^*)$ will take on is determined by the draw of $Z_{n+1}$ from $\lbag z_{1:n+1} \rbag$; that is, given $\lbag z_{1:n+1} \rbag$, we have the equivalency $L_{n+1}(\lambda_{+}^*) = l_i(\lambda_{+}^*)  \iff Z_{n+1} = z_i$ due to $g(\cdot,\lambda_{+}^*)$ inducing a bijection from the discrete bag of datapoints $\lbag z_1, ..., z_{n+1} \rbag$ to the discrete bag of loss values $\lbag l_{1}(\lambda_{+}^*), ..., l_{n+1}(\lambda_{+}^*) \rbag$.

First using this observation and the definition of conditional expectation, and secondly using the definition of conditional probability, we have
\begin{align*}
    \mathbb{E}\big[L_{n+1}(\lambda_{+}^*) \mid \lbag Z_{1:n+1} \rbag = \lbag z_{1:n+1} \rbag\big] & = \sum_{i=1}^{n+1}l_i(\lambda_{+}^*)\cdot \mathbb{P}\big(Z_{n+1} = z_{i} \mid \lbag Z_{1:n+1} \rbag = \lbag z_{1:n+1} \rbag \big) \\
    & = \sum_{i=1}^{n+1}l_i(\lambda_{+}^*)\cdot \frac{\mathbb{P}\big(Z_{n+1}=z_{i}, \ \lbag Z_{1:n+1} \rbag = \lbag z_{1:n+1} \rbag \big)}{\mathbb{P}\big(\lbag Z_{1:n+1} \rbag = \lbag z_{1:n+1} \rbag\big)}.
\end{align*}

Applying the law of total probability over permutations ($\sigma:[n+1]\rightarrow[n+1]$) to the probabilities in the numerator and denominator of the RHS, and then simplifying using exchangeability, we have
\begin{align*}
    \mathbb{E}\big[L_{n+1}(\lambda_{+}^*) \mid \lbag Z_{1:n+1} \rbag =  \lbag z_{1:n+1} \rbag\big]
    & = \sum_{i=1}^{n+1}l_i(\lambda_{+}^*)\cdot \frac{\sum_{\sigma:\sigma(n+1)=i}f(z_{\sigma(1)}, ..., z_{\sigma(n+1)})}{\sum_{\sigma}f(z_{\sigma(1)}, ..., z_{\sigma(n+1)})} \\
    & = \sum_{i=1}^{n+1}l_i(\lambda_{+}^*)\cdot \frac{\sum_{\sigma:\sigma(n+1)=i}f(z_1, ..., z_{n+1})}{\sum_{\sigma}f(z_1, ..., z_{n+1})} \\
    & = \sum_{i=1}^{n+1}l_i(\lambda_{+}^*)\cdot \frac{n!}{(n+1)!} \\
    & = \sum_{i=1}^{n+1}\frac{l_i(\lambda_{+}^*)}{n+1}
\end{align*}

Now, using Eq. \eqref{eq:gcrc_exchangeable_step_1}, 
\begin{align}\label{eq:gcrc_exchangeable_step_2}
    &\sum_{i=1}^{n+1}\frac{l_{i}(\lambda)}{n+1}\leq \alpha \qquad \forall \qquad \text{fixed } \lambda \geq \lambda_{+}^*,
\end{align}
where here by ``fixed'' we mean that $\lambda$ is constant conditional on the bag of data, $\lbag Z_{1:n+1} \rbag = \lbag z_{1:n+1} \rbag$.

\vspace{1cm}

\textbf{Step 3: Relating the empirically-selected hyperparameter to the oracle parameter}

In Step 2, we have shown that an oracle procedure deterministically controls the risk, conditional on the event $\lbag Z_{1:n+1} \rbag = \lbag z_{1:n+1} \rbag$. However, this oracle procedure cannot be implemented in practice due to requiring knowledge of the test point loss, $L_{n+1}$; that is, recall we defined $\lambda_+^* := \hat{\lambda}_+(\lbag L_{1:n+1} \rbag, \alpha)$. In this last step of the proof we would now like to prove an almost-sure relationship between the oracle parameter, $\lambda_+^*$, and the empirically-selected parameter, $\hat{\lambda}_+:=\hat{\lambda}_+(\lbag L_{1:n}, B \rbag, \alpha)$, to conclude valid risk control for the empirical method; that is, we would like to show that $\lambda_+^* \leq \hat{\lambda}_+$ almost surely.

To begin, recall that in Step 2, conditioning on the event $\lbag Z_{1:n+1} \rbag = \lbag z_{1:n+1} \rbag$ implied that $\hat{\lambda}_+(\lbag L_{1:n+1} \rbag, \alpha) = \hat{\lambda}_+(\lbag l_{1:n+1} \rbag, \alpha)$. For the empirical parameter, however, conditioning on $\lbag Z_{1:n+1} \rbag = \lbag z_{1:n+1} \rbag$ does \textit{not} imply what may seem analogous: that is, it could be that $\hat{\lambda}_+(\lbag L_{1:n}, B \rbag, \alpha) \neq \hat{\lambda}_+(\lbag l_{1:n}, B \rbag, \alpha)$. This is because the event $\lbag Z_{1:n+1} \rbag = \lbag z_{1:n+1} \rbag$ only allows us to conclude that the $n$ random variables in $\lbag L_{1:n}\rbag$ take $n$ distinct values from the bag of $n+1$ observed losses, $\lbag l_{1:n+1}\rbag$, but we do not yet know \textit{which} observed losses. That is, what we \textit{can} conclude is, given $\lbag Z_{1:n+1} \rbag = \lbag z_{1:n+1} \rbag$, that
\begin{align*}
    \lbag Z_{1:n+1} \rbag = \lbag z_{1:n+1} \rbag \implies \lbag L_{1:n} \rbag \in \Big\{\lbag l_{1:n+1\backslash i} \rbag \Big\}_{i\in [n+1]},
\end{align*}
where $\lbag l_{1:n+1\backslash i} \rbag := \lbag l_{1:n+1} \rbag \backslash \{l_i\}$. Accordingly, conditioning on the event $\lbag Z_{1:n+1} \rbag = \lbag z_{1:n+1} \rbag$ implies the following equivalence:
\begin{align}\label{eq:oracle_empirical_inequality_equivalences}
    \lambda_+^* \leq \hat\lambda_+(\lbag L_{1:n}, B \rbag, \alpha) \ \text{almost surely} \ & \iff \lambda_+^* \leq \hat\lambda_+(\lbag l_{1:n+1\backslash i}, B \rbag, \alpha) \quad \forall \quad i \in [n+1].
\end{align}
So, we will focus on proving the latter.

Recall that we have assumed that $l$ is bounded above by $B$ (i.e., $\sup_{\lambda}l_i(\lambda)\leq B < \infty$), for any $\lambda\in \Lambda$, almost surely we have
\begin{align}\label{eq:step_3_conservative_adjustment}
    \sum_{j\in[n+1]}\frac{l_j(\lambda)}{n+1}
    \quad \leq \quad  \frac{B}{n+1} + \sum_{j\in[n+1]\backslash i}\frac{l_j(\lambda)}{n+1} \qquad \forall \qquad i \in [n+1],
\end{align}

where the right-hand side of the above inequalities can be interpreted as a conservative risk obtained by replacing $l_i(\lambda)$ with the bound $B$.

Thus, for any $\lambda$, we can use \eqref{eq:step_3_conservative_adjustment} to conclude
\begin{align}\label{eq:step_3_risk_relation}
    \frac{B}{n+1} + \sum_{j\in[n+1]\backslash i}\frac{l_j(\lambda)}{n+1} \leq \alpha \implies & \sum_{j\in[n+1]}\frac{l_j(\lambda)}{n+1} \leq \alpha \qquad \forall \qquad i \in [n+1].
\end{align}

So, by \eqref{eq:step_3_risk_relation} and the definition of $\hat{\lambda}_+(\lbag l_{1:n+1\backslash i}, B \rbag)$, it must hold that 
\begin{align}\label{eq:pf_step_3_result}
    \lambda_+^* \leq & \ \hat{\lambda}_+(\lbag l_{1:n+1\backslash i}, B \rbag) \quad \forall \quad i \in [n+1].
\end{align}

\vspace{1cm}

\textbf{Step 4: Analyzing the bag-conditional risk of the empirical procedure via Lipschitz continuity}

Now we turn to plugging the empirical algorithm's hyperparameter into the bag-conditional risk, and then analyzing this risk using Lipschitz continuity and a type of stability called ``replace-one'' stability. This analysis is necessary because $\hat\lambda_+=\widehat{\lambda}_+(\lbag Z_{1:n}, B\rbag, \alpha)$ is not constant conditional on $\lbag Z_{1:n+1} \rbag = \lbag z_{1:n+1} \rbag$, so Eq. \eqref{eq:gcrc_exchangeable_step_2} would not direclty apply. 

Begin by expanding this quantity by LOTE over the event $Z_{n+1}=z_i$:
\begin{align}\label{eq:empirical_bag_cond_LOTE}
    \mathbb{E}\Big[&L_{n+1}(\hat\lambda_{+}) \mid \lbag Z_{1:n+1} \rbag = \lbag z_{1:n+1} \rbag\Big] \nonumber \\
    & = \mathbb{E}\Big[\mathbb{E}\Big[L_{n+1}(\hat\lambda_{+}) \mid \lbag Z_{1:n+1} \rbag = \lbag z_{1:n+1} \rbag, Z_{n+1}=z_i\Big] \mid \lbag Z_{1:n+1} \rbag = \lbag z_{1:n+1} \rbag\Big] \nonumber \\
    & = \sum_{i=1}^{n+1}\mathbb{E}\Big[L_{n+1}(\hat\lambda_{+}) \mid \lbag Z_{1:n+1} \rbag = \lbag z_{1:n+1} \rbag, Z_{n+1}=z_i\Big]\cdot \mathbb{P}\big(Z_{n+1}=z_i \mid \lbag Z_{1:n+1} \rbag = \lbag z_{1:n+1} \rbag\big).
\end{align}

Now, note that we can write the inner expectation as
\begin{align*}
    \mathbb{E}\Big[L_{n+1}(\hat\lambda_{+}) \mid \lbag Z_{1:n+1} \rbag = \lbag z_{1:n+1} \rbag, Z_{n+1}=z_i\Big] & = \mathbb{E}\Big[L_{n+1}(\hat\lambda_{+}) \mid \lbag Z_{1:n} \rbag = \lbag z_{1:n+1\backslash i} \rbag, Z_{n+1}=z_i\Big] \\
    & = \mathbb{E}\Big[L_{n+1}(\hat\lambda_{+}) \mid \lbag L_{1:n} \rbag = \lbag l_{1:n+1\backslash i} \rbag, L_{n+1}=l_i\Big].
\end{align*}

And recalling that we defined $\hat\lambda_+:= \widehat{\lambda}_+\big(\lbag L_{1:n}, B \rbag, \ \alpha\big)$, we have
\begin{align*}
    \mathbb{E}\Big[&L_{n+1}(\hat\lambda_{+}) \mid \lbag Z_{1:n+1} \rbag = \lbag z_{1:n+1} \rbag, Z_{n+1}=z_i\Big]  \\
    & = \mathbb{E}\Big[L_{n+1}\Big(\widehat{\lambda}_+\big(\lbag L_{1:n}, B  \rbag, \ \alpha\big)\Big) \mid \lbag L_{1:n} \rbag = \lbag l_{1:n+1\backslash i} \rbag, L_{n+1}=l_i\Big] \\
    & = \mathbb{E}\Big[l_{i}\Big(\widehat{\lambda}_+\big(\lbag l_{1:n+1\backslash i}, B  \rbag, \ \alpha\big)\Big) \mid \lbag L_{1:n} \rbag = \lbag l_{1:n+1\backslash i} \rbag, L_{n+1}=l_i\Big] \\
    & = l_{i}\Big(\widehat{\lambda}_+\big(\lbag l_{1:n+1\backslash i}, B  \rbag, \ \alpha\big)\Big).
\end{align*}

So, plugging this into Eq. \eqref{eq:empirical_bag_cond_LOTE}, we have
\begin{align*}
    \mathbb{E}\Big[& L_{n+1}(\hat\lambda_{+}) \mid \lbag Z_{1:n+1} \rbag = \lbag z_{1:n+1} \rbag\Big] \\
    & = \sum_{i=1}^{n+1}l_{i}\Big(\widehat{\lambda}_+\big(\lbag l_{1:n+1\backslash i}, B  \rbag, \ \alpha\big)\Big) \cdot \mathbb{P}\big(Z_{n+1}=z_i \mid \lbag Z_{1:n+1} \rbag = \lbag z_{1:n+1} \rbag\big) \\
    & = \sum_{i=1}^{n+1}l_{i}\Big(\widehat{\lambda}_+\big(\lbag l_{1:n+1\backslash i}, B  \rbag, \ \alpha\big)\Big) \cdot\frac{\mathbb{P}\big(Z_{n+1}=z_{i}, \ \lbag Z_{1:n+1} \rbag = \lbag z_{1:n+1} \rbag \big)}{\mathbb{P}\big(\lbag Z_{1:n+1} \rbag = \lbag z_{1:n+1} \rbag\big)}
\end{align*}

Applying the law of total probability over permutations ($\sigma:[n+1]\rightarrow[n+1]$) to the probabilities in the numerator and denominator of the RHS, and then simplifying using exchangeability, we have
\begin{align}\label{eq:gcrc_empirical_risk_bag_conditional}
    \mathbb{E}\big[L_{n+1}(\hat\lambda_{+}) \mid \lbag Z_{1:n+1} \rbag =  \lbag z_{1:n+1} \rbag\big] \nonumber
    & = \sum_{i=1}^{n+1}l_{i}\Big(\widehat{\lambda}_+\big(\lbag l_{1:n+1\backslash i}, B  \rbag, \ \alpha\big)\Big)\cdot \frac{\sum_{\sigma:\sigma(n+1)=i}f(z_{\sigma(1)}, ..., z_{\sigma(n+1)})}{\sum_{\sigma}f(z_{\sigma(1)}, ..., z_{\sigma(n+1)})} \nonumber \\
    & = \sum_{i=1}^{n+1}l_{i}\Big(\widehat{\lambda}_+\big(\lbag l_{1:n+1\backslash i}, B  \rbag, \ \alpha\big)\Big)\cdot \frac{\sum_{\sigma:\sigma(n+1)=i}f(z_1, ..., z_{n+1})}{\sum_{\sigma}f(z_1, ..., z_{n+1})} \nonumber \\
    & = \sum_{i=1}^{n+1}l_{i}\Big(\widehat{\lambda}_+\big(\lbag l_{1:n+1\backslash i}, B  \rbag, \ \alpha\big)\Big)\cdot \frac{n!}{(n+1)!} \nonumber \\
    & = \frac{1}{n+1}\sum_{i=1}^{n+1}l_{i}\Big(\widehat{\lambda}_+\big(\lbag l_{1:n+1\backslash i}, B  \rbag, \ \alpha\big)\Big).
\end{align}

Importantly, note that in Eq. 
\eqref{eq:gcrc_empirical_risk_bag_conditional}, which is the true bag-conditional risk of the empirical algorithm, that the value of $\widehat{\lambda}_+\big(\lbag l_{1:n+1\backslash i}, B  \rbag, \ \alpha\big)$ differs for each loss function, $l_i$, in the summation, which is why we cannot apply Eq. \eqref{eq:gcrc_exchangeable_step_2} to bound it directly. Instead, we relate Eq. \eqref{eq:gcrc_empirical_risk_bag_conditional} and Eq. \eqref{eq:gcrc_exchangeable_step_2} using Lipschitz continuity and the replace-one (in)stability of the algorithm.

For all $i\in[n]$, denote the magnitude of perturbation to the hyperparameter $\hat\lambda_+(\lbag l_{1:n}, B\rbag, \alpha)$ caused by replacing the loss function $l_i$ with $l_{n+1}$ as 
\begin{align*}
    \epsilon_i = \Big|\hat\lambda_+(\lbag l_{1:n}, B\rbag, \alpha)-\hat\lambda_+(\lbag l_{1:n+1\backslash i}, B\rbag, \alpha)\Big|.
\end{align*}
Then, by $K$-Lipschitz continuity, we know that for all $i\in [n]$, that
\begin{align*}
    \Big|l_i\big(\hat\lambda_+(\lbag l_{1:n}, B\rbag, \alpha)\big) - l_i\big(\hat\lambda_+(\lbag l_{1:n+1\backslash i}, B\rbag, \alpha)\big)\Big| \leq K \epsilon_i.
\end{align*}

Then, note that 
\begin{align*}
    \sum_{i=1}^{n+1}\frac{l_i\big(\hat\lambda_+(\lbag l_{1:n+1\backslash i}, B\rbag, \alpha)\big)}{n+1} & \leq \sum_{i=1}^{n+1}\frac{l_i\big(\hat\lambda_+(\lbag l_{1:n}, B\rbag, \alpha)\big) + K\epsilon_i}{n+1} = \sum_{i=1}^{n+1}\frac{l_i\big(\hat\lambda_+(\lbag l_{1:n}, B\rbag, \alpha)\big)}{n+1} + \frac{K}{n+1}\sum_{i=1}^{n+1}\epsilon_i 
\end{align*}

By Eq. \eqref{eq:pf_step_3_result}, we know that $\hat\lambda_+(\lbag l_{1:n}, B\rbag, \alpha) \geq \lambda_+^*$, and because this $\hat\lambda_+(\lbag l_{1:n}, B\rbag, \alpha)$ does not depend on $i$, we can apply Eq. \eqref{eq:gcrc_exchangeable_step_2} to bound it above by $\alpha$, which gives 
\begin{align*}
     \sum_{i=1}^{n+1}\frac{l_i\big(\hat\lambda_+(\lbag l_{1:n+1\backslash i}, B\rbag, \alpha)\big)}{n+1} & \leq \alpha + \frac{K}{n+1}\sum_{i=1}^{n+1}\epsilon_i \\
     \iff \mathbb{E}\big[L_{n+1}(\hat\lambda_{+}) \mid \lbag Z_{1:n+1} \rbag =  \lbag z_{1:n+1} \rbag\big] & \leq \alpha + \frac{K}{n+1}\sum_{i=1}^{n+1}\epsilon_i.
\end{align*}

Then, marginalizing over the event and applying $\epsilon$-replace-one stability, we have
\begin{align*}
    \mathbb{E}\big[L_{n+1}(\hat\lambda_{+}) \big] & \leq \alpha + K \epsilon.
\end{align*}

Note that the above analysis also applies for any alternative $\hat\lambda_+'$ that is a deterministic function of $\lbag l_{1:n}\rbag$ and where almost surely $\hat\lambda_+'\geq \lambda_+^*$, and assuming, $\hat\lambda_+'$ is $\epsilon$-replace-one stable.

\end{proof}

\subsubsection{Counterexample to gCRC risk control on arbitrary bounded losses}
\label{app:counterexample}

\begin{proposition}
    
    The following assumptions are not sufficient for $\hat\lambda_+$ as defined in Eq. \eqref{eq:methods_lambda_empirical_def} to achieve the risk control guarantee at level $\alpha$:
    \begin{itemize}
        \item Right-continuity of $L_i(\lambda)$ in $\lambda$,
        \item Existence of a safe hyperparameter: For $\lambda_{\max}:=\sup\Lambda\in \Lambda$, that $L_i(\lambda_{\max})\leq \alpha$,
        \item Boundedness of the loss functions: $\sup_{\lambda}L_i(\lambda)\leq B < \infty \quad \text{almost surely}$.
    \end{itemize}

\end{proposition}

\begin{proof}
This counterexample assumes exchangeable data.
Let $\Lambda = \{\lambda_1, \lambda_2, \lambda_3, \lambda_{\max}\}$ where $\lambda_1 < \lambda_2 < \lambda_3 < \lambda_{\max}$.
Fix $\alpha \in (0, 1]$.
Consider the following $n + 1 = 3$ loss functions, where $B = \frac{3}{2}\alpha$:
\begin{align}
\label{eq:bag-loss}
l_1(\lambda) =
    \begin{cases}
    B, & \lambda = \lambda_1 \\
    B/2, & \lambda = \lambda_2 \\
    0, & \lambda = \lambda_3 \\
    0, & \lambda = \lambda_{\max}
  \end{cases}, \qquad 
    l_2(\lambda) =
    \begin{cases}
    B/2, & \lambda = \lambda_1 \\
    B, & \lambda = \lambda_2 \\
    0, & \lambda = \lambda_3 \\
    0, & \lambda = \lambda_{\max}
  \end{cases}, \qquad 
    l_3(\lambda) =
    \begin{cases}
    B/2, & \lambda = \lambda_1 \\
    0, & \lambda = \lambda_2 \\
    B, & \lambda = \lambda_3 \\
    0, & \lambda = \lambda_{\max}
  \end{cases}.
\end{align}
Note that $\max_{\lambda \in \Lambda} l_i(\lambda) = B$ and $l_i(\lambda_{\max}) = 0 < \alpha$ for all $i \in [3]$, satisfying our assumptions.

Consider the bag-conditional empirical risk,
\begin{align}
\label{eq:tower-again}
    \mathbb{E}[L_3(\hat{\lambda}) \mid \lbag L_{1:3}\rbag = \lbag l_{1:3}\rbag] & = \mathbb{E}\left[ \; \mathbb{E}[ L_3(\hat{\lambda}) \mid \lbag L_{1:3}\rbag = \lbag l_{1:3}\rbag, L_3 = l_i] \; \middle| \; \lbag L_{1:3}\rbag = \lbag l_{1:3}\rbag \right],
\end{align}
where the inner expectation on the righthand side conditions on the value of the bag of loss functions, $\lbag L_{1:3} \rbag = \lbag l_{1:3} \rbag$, and additionally on the event that the test loss takes the value of the $i$-th loss function in the bag, $L_{3} = l_i$.
Note that the test loss at the empirical threshold, $L_{3}(\hat{\lambda})$, is a deterministic function of the values of $\lbag L_{1:3} \rbag$ and $L_{3}$.
Therefore, this inner expectation is simply
\begin{align}
    \mathbb{E}\left[L_{3}(\hat{\lambda}) \; \middle| \; \lbag L_{1:3}\rbag = \lbag l_{1:3}\rbag, L_{3} = l_i\right] & = l_i(\hat{\lambda}(\lbag l_{1:3\setminus i}, B \rbag, \alpha)).
\end{align}
Substituting this expression for the inner expectation in Eq.~\eqref{eq:tower-again}, we can write the bag-conditional empirical risk as
\begin{align}
\label{eq:bag-conditional-empirical-risk}
    \mathbb{E}[L_3(\hat{\lambda}) \mid \lbag L_{1:3}\rbag = \lbag l_{1:3}\rbag] = \mathbb{E}[ \; l_i(\hat{\lambda}(\lbag l_{1:3\setminus i}, B \rbag, \alpha)) \mid \lbag L_{1:3}\rbag = \lbag l_{1:3}\rbag] = \frac{1}{3}\sum_{i =1}^{3} l_i(\hat{\lambda}(\lbag l_{1:3\setminus i}, B \rbag, \alpha)).
\end{align}
We now calculate this quantity for the bag of loss functions given in Eqs.~\eqref{eq:bag-loss}.
Each case or ``world'' below computes one summand in the mean in Eq.~\eqref{eq:bag-conditional-empirical-risk}.
To simplify forthcoming calculations, note that
\begin{align*}
    \hat{\lambda}(\lbag l_{1:3\setminus i}, B \rbag, \alpha) = \min \{\lambda' \in \Lambda : \forall \lambda \geq \lambda', \sum_{j \neq i}l_j(\lambda) \leq 3\alpha - B = 2B - B = B\}.
\end{align*}
\begin{itemize}
    \item \textbf{Case 1:} $L_3 = l_1$. We have
    \begin{align}
        \hat{\lambda}(\lbag l_2, l_3, B \rbag, \alpha) & = \min \{\lambda' \in \Lambda : \forall \lambda \geq \lambda', l_2(\lambda) + l_3(\lambda) \leq B\} \\
        & = \lambda_1,
    \end{align}
    so $L_3(\hat{\lambda}) = l_1(\hat{\lambda}(\lbag l_2, l_3, B \rbag, \alpha)) = l_1(\lambda_1) = B$.
    \item \textbf{Case 2:} $L_3 = l_2$. We have
    \begin{align}
        \hat{\lambda}(\lbag l_1, l_3, B \rbag, \alpha) & = \min \{\lambda' \in \Lambda : \forall \lambda \geq \lambda', l_1(\lambda) + l_3(\lambda) \leq B\} \\
        & = \lambda_2,
    \end{align}
    so $L_3(\hat{\lambda}) = l_2(\hat{\lambda}(\lbag l_1, l_3, B \rbag, \alpha)) = l_2(\lambda_2) = B$.
    \item \textbf{Case 3:} $L_3 = l_3$. We have
    \begin{align}
        \hat{\lambda}(\lbag l_1, l_2, B \rbag, \alpha) & = \min \{\lambda' \in \Lambda : \forall \lambda \geq \lambda', l_1(\lambda) + l_2(\lambda) \leq B\} \\
        & = \lambda_3,
    \end{align}
    so $L_3(\hat{\lambda}) = l_3(\hat{\lambda}(\lbag l_1, l_2, B \rbag, \alpha)) = l_3(\lambda_3) = B$.
\end{itemize}
Plugging these values into Eq.~\eqref{eq:bag-conditional-empirical-risk} gives
\begin{align*}
    \mathbb{E}[L_3(\hat{\lambda}) \mid \lbag L_{1:3}\rbag = \lbag l_{1:3}\rbag] & = \frac{1}{3}\sum_{i =1}^{3} l_i(\hat{\lambda}(\lbag l_{1:3\setminus i}, B \rbag, \alpha)) \\
    & = \frac{1}{3}(B + B + B) \\
    & = B = \frac{3}{2}\alpha > \alpha.
\end{align*}

\end{proof}

\subsubsection{Conservative adjustment to gCRC for level $\alpha$ validity if $K$ is known}
\label{app:conservative_adjustment_K_known}

Assume the gCRC setting of Theorem \ref{thm:gcrc_nonmonotonic} in Section \ref{subsec:finite_sample_guarantees_crc_non_monotone} (although an analogous result would also hold for conformal policy control). In practice, the deviation of the empirical gCRC algorithm due to replacing one sample can be bounded above by
\begin{align*}
    \hat\epsilon_i := \max_{b\in \{0, B\}}\Big|\hat\lambda_+(\lbag l_{1:n}, B\rbag, \alpha)-\hat\lambda_+(\lbag l_{1:n\backslash i}, b, B\rbag, \alpha)\Big|,
\end{align*}
and if the Lipschitz constant for the loss functions is known, a conservative significance level can be obtained as
\begin{align*}
    \hat\alpha := \sup\Big(\Big\{\alpha' \leq \alpha \quad : \quad \alpha' + \frac{K}{n+1}\sum_{i=1}^{n+1}\hat\epsilon_i(\alpha')\leq \alpha\Big\}  \Big). 
\end{align*}
Then, the procedure run at $\hat\alpha$ attains level $\alpha$ validity.
\begin{proposition}
    In the gCRC setting of Theorem \ref{thm:gcrc_nonmonotonic},
    \begin{align*}
        \mathbb{E}\big[L_{n+1}\big(\hat{\lambda}_{+}(\hat\alpha)\big)\big] \leq \alpha.
    \end{align*}
\end{proposition}

\newpage

\subsection{Proofs for Conformal Policy Control}
\label{app:cpc_proofs}

\subsubsection{Notation and Definitions}
\label{app:cpc_notation_and_defs}

\textbf{Notation:}
\begin{itemize}
    \item \textbf{Data (sequences of random variables and observed values):} For all $i\in \mathbb{N}$, denote the random variable for the $i$-th data sample as $Z_i\in \mathcal{Z}$, and denote observed value of that random variable in lower case as $z_i\in\mathcal{Z}$. For example, we might have $\mathcal{Z}=\mathcal{X}\times \mathcal{A}\times \mathbb{R}\times \mathbb{R}$ if $Z=(X, A, R, L)$ is a vector of context, action, reward, and loss. Similarly, denote a sequence of random variables as $Z_{1:m}:=(Z_1, ..., Z_m)$ and a sequence of observed data values as $z_{1:m}:=(z_1, ..., z_m)$ for any $m \in \mathbb{N}$. To concisely denote indices, let $[m]:=\{1, ..., m\}$.
    \item \textbf{Hyperparameters:} Let $\mathcal{B}\subseteq(0, \infty)$ be the space of hyperparameters, and let $\beta\in \mathcal{B}$ denote a particular hyperparameter. 
    \item \textbf{Loss functions:} For some real-valued function $g :\mathcal{Z}\times \mathcal{B} \rightarrow \mathbb{R}$, and for all $\beta\in\mathcal{B}\subseteq\mathbb{R}$ and all $i\in \mathbb{N}$, let $L_i(\beta):= g(Z_i,\beta)$ denote the loss function determined by a random variable $Z_i$ and parameter $\beta$, and let $\ell_i:= g(z_i,\beta)$ denote the loss function determined by the observed value $z_i$ and parameter $\beta$. Note that allowing the loss functions to be a function of $\beta$ does not preclude the setting where the losses are independent of $\beta$, that is where $L_i(\beta)=L_i$ for all $\beta$ (i.e., where changing $\beta$ does not change the loss). Such a setting motivates allowing $\beta$ to parameterize the policies, which is the setting we focus on.
    \item \textbf{Bags:} For any sequence $v_{1:m}:=(v_1, ..., v_m)$, let $\lbag v_{1:m} \rbag$ denote the \textit{bag} (or multiset) containing the values of that sequence $v_{1:m}$ (but importantly, $\lbag v_{1:m} \rbag$ does \textit{not} contain the information about the order of $v_{1:m}$). For example, $\lbag (1, 7, 42, 7) \rbag$ conveys that there is one 1, two 7s, and one 42 in the bag, but it does \textit{not} convey the order that these appeared in the sequence $(1, 7, 42, 7)$. In particular, with this notation, for the random variables $Z_{1:m}$, let $\lbag Z_{1:m} \rbag$ denote the bag containing those random variables; for observed values $z_{1:m}$, let $\lbag z_{1:m} \rbag$ denote the bag containing those observed values; and thus $\lbag Z_{1:m} \rbag = \lbag z_{1:m} \rbag$ denotes the event of observing $\lbag z_{1:m} \rbag$, in other words, that each $Z_i$ takes a different value in $\lbag z_{1:m} \rbag$, but it is not yet known \textit{which} value (for all $i\in [m]$). 
    \item \textbf{Policies (distributions):} Let $\pi_{0:t}:=\pi_{Z_0, ..., Z_t}$ denote the policy or joint density function\footnote{In a slight abuse of notation for the policy control setting, we overload $\pi$ for both density functions and CDFs: in particular, we use $\pi(z)$ to denote a density evaluated at some point $z$ and $\mathbb{E}_{\pi}$ to denote an expectation taken with respect to the distribution $\pi$.} for $Z_0, ..., Z_t$, and let $\pi_{\lbag Z_{0:t}\rbag}:=\pi_{\lbag Z_0, ..., Z_t\rbag}$ denote the (pushforward) distribution induced by the mapping $(z_0, ..., z_t)\rightarrow \lbag z_{0:t} \rbag$. When clear from context, we let $\hat\beta$ denote the current control parameter for time $t$, and so $\pi_{0:t}^{(\hat\beta)}(z_0, ..., z_t):=\pi_{t\mid 0:t-1}^{(\hat\beta)}(z_t \mid z_0, ..., z_{t-1})\cdot \pi_{0:t-1}(z_0, ..., z_{t-1})$, with $\hat\beta$ only tuning the last policy. We will also assume a split conformal setting, where $Z_t$ may only depend symmetrically on $z_{0:t-1}$, so $\pi_{t\mid 0:t-1}^{(\hat\beta)}(z_t \mid z_0, ..., z_{t-1}) = \pi_{t\mid 0:t-1}^{(\hat\beta)}(z_t \mid \lbag z_{0:t-1}\rbag )$ and we will sometimes abbreviate $\pi_{t}^{(\hat\beta)}(z_t ):=\pi_{t\mid 0:t-1}^{(\hat\beta)}(z_t \mid \lbag z_{0:t-1}\rbag )$.
    \item \textbf{Normalized and unnormalized oracle conformal (permutation) weights:} For any control parameter $\hat\beta$, any policies $\pi_{0:t}^{(\hat\beta)}$, and any bag of data $\lbag z_{0:t}\rbag$, following \citet{tibshirani2019conformal, prinster2024conformal} define the oracle normalized conformal weights as
    \begin{align}\label{eq:cp_weights_def_cpc}
        \tilde{w}_i^{(\hat\beta)} := \tilde{w}_i\big(\pi_{0:t}^{(\hat\beta)}, \lbag z_{0:t}\rbag\big) = \mathbb{P}_{\pi_{0:t}^{(\hat\beta)}}(Z_t=z_i \mid \lbag z_{0:t}\rbag) =\frac{\sum_{\sigma:\sigma(t)=i}\pi_{0:t}^{(\hat\beta)}(z_{\sigma(0)}, ..., z_{\sigma(t)})}{\sum_{\sigma}\pi_{0:t}^{(\hat\beta)}(z_{\sigma(0)}, ..., z_{\sigma(t)})},
    \end{align}
    for all $i\in \{0, ..., t\}$, where $\sigma$ is a permutation of $\{0, ..., t\}$. Also define the unnormalized weights, 
    \begin{align}\label{eq:cp_weights_def_cpc_unnormed}
        w_i^{(\hat\beta)} := w_i\big(\pi_{0:t}^{(\hat\beta)}, \lbag z_{0:t}\rbag\big) = \sum_{\sigma:\sigma(t)=i}\pi_{0:t}^{(\hat\beta)}(z_{\sigma(0)}, ..., z_{\sigma(t)}),
    \end{align}
    so note that $\tilde{w}_i^{(\hat\beta)} = \frac{w_i^{(\hat\beta)}}{\sum_{i=0}^tw_i^{(\hat\beta)}}$,
    and define the vector $\tilde{w}_{0:t}^{(\hat\beta)}:=(\tilde{w}_0^{(\hat\beta)}, ..., \tilde{w}_t^{(\hat\beta)})$.
    
    Also define the conservative (normalized) weight for an unknown test point as
    \begin{align*}
        \tilde{w}_{\max}^{(\hat\beta)}:=\sup_{z_t\in \mathcal{Z}}(\tilde{w}_t^{(\hat\beta)}).
    \end{align*}
    \item \textbf{Likelihood-ratio bound control parameter} For convenience, when we refer to the likelihood-ratio bound control parameter, we refer to $\beta$ as in
    \begin{align}\label{appeq:piecewise_constraint_def}
    \pi_{t}^{(\beta)}(a \mid x) \propto \min\big(\pi_{t}(a \mid x), \ \beta\cdot \pi_{0}(a \mid x)\big) .
    \end{align}
\end{itemize}

Whereas the smoothness and stability assumptions in the setting of Theorem \ref{thm:gcrc_nonmonotonic} were on an additive scale, here the following assumptions concern multiplicative or relative changes that are more interpretable for policy control.

\begin{definition}[$K$-log-Lipschitz policy family]
\label{def:loglip_family}
    For $K>0$, say $\{\pi_t^{(\beta)}\}_{\beta\in\mathcal{B}}$ is \emph{$K$-log-Lipschitz in $\log\beta$} if for all $\beta,\beta'\in\mathcal{B}$ and all $z\in\mathcal{Z}$ with $\pi_t^{(\beta)}(z),\pi_t^{(\beta')}(z)>0$,
    \begin{align*}
        \Big|\log \pi_t^{(\beta')}(z)-\log \pi_t^{(\beta)}(z)\Big| \; \leq \; K\,\big|\log\beta'-\log\beta\big| .
    \end{align*}
\end{definition}

\begin{definition}[$\epsilon^{(K)}$-relative replace-one stability]
\label{def:rel_replace_one_stability}
    Let $z_{0:t}$ denote the observed values, so that $\lbag Z_{0:t}\rbag = \lbag z_{0:t}\rbag$. Conditional on this event, let $\hat\beta^{\backslash z_i}$ denote $\hat\beta$ fit with the loss for the value $z_i$ removed (i.e., $\hat\beta^{\backslash z_i} :=\hat\beta\big(\lbag L_{0:t}, B \rbag\backslash \{l_i\}, \tilde{w}_{0:t-1}^{(\cdot)}, \tilde{w}_{\max}^{(\cdot)}, \alpha\big)$), and let $\hat\beta^{\backslash Z_j}$ denote $\hat\beta$ fit with the loss for the random variable $Z_j$ removed (i.e., $\hat\beta^{\backslash Z_j} :=\hat\beta\big(\lbag L_{0:t}, B \rbag\backslash \{L_j\}, \tilde{w}_{0:t-1}^{(\cdot)}, \tilde{w}_{\max}^{(\cdot)}, \alpha\big)$). We say that $\hat\beta$ is $\epsilon^{(K)}$-relative replace-one stable if for all $i, j\in \{0, ..., t\}$,
    \begin{align*}
        \mathbb{E}_{\pi_{0:t}^{(\hat\beta)}}\bigg[\max\bigg(\frac{\hat\beta^{\backslash Z_j}}{\hat\beta^{\backslash z_i}},\ \frac{\hat\beta^{\backslash z_i}}{\hat\beta^{\backslash Z_j}}\bigg)^{K}-1\bigg] \; \leq \; \epsilon^{(K)}.
    \end{align*}
\end{definition}

As in the case with gCRC, we apply this definition only with $i=j=t$.

\vspace{0.5cm}

We refer to Remark \ref{rem:ro_stability_interpret} for intuition on interpreting the mixed lower and upper cases.

\vspace{0.5cm}

\begin{remark}[Interpretation as bound on moment generating function]
\label{rem:cpc_mgf} Definition~\ref{def:rel_replace_one_stability} can equivalently be viewed a bound on
the moment generating function of the $\log\beta$-scale replace-one perturbation $|\log(\hat\beta^{\backslash Z_j}) - \log(\hat\beta^{\backslash z_i})|$,
evaluated at $K$:
\begin{align*}
    \mathbb{E}_{\pi_{0:t}^{(\hat\beta)}}\big[e^{K|\log(\hat\beta^{\backslash Z_j}) - \log(\hat\beta^{\backslash z_i})|}\big] \; \leq \; 1+\epsilon^{(K)} .
\end{align*}

\end{remark}

\vspace{0.5cm}

\begin{remark}[Equivalent form and interpretation for $K=1$]
    For $K=1$, $\epsilon^{(1)}$-relative replace-one stability is equivalent to
    \begin{align*}
        \mathbb{E}_{\pi_{0:t}^{(\hat\beta)}}\bigg[\frac{\big|\hat\beta^{\backslash Z_j}-\hat\beta^{\backslash z_i}\big|}{\min\big(\hat\beta^{\backslash Z_j},\hat\beta^{\backslash z_i}\big)}\bigg] \; \leq \; \epsilon^{(1)}.
    \end{align*}
    That is, this stability notion says that the expected replace-one perturbation of $\hat\beta$ \emph{relative} to its own scale is bounded, rather than bounding the perturbation in absolute units.
\end{remark}

\vspace{0.5cm}

\subsubsection{Proofs CPC with Arbitrary $\beta$ and CPC for Likelihood-Ratio Bound $\beta$}

Unless specified otherwise, throughout this subsection $\beta\in\mathcal{B}\subseteq(0,\infty)$ is an \emph{arbitrary} control parameter indexing a family of policies $\{\pi_t^{(\beta)}\}_{\beta\in\mathcal{B}}$, subject only to the smoothness condition of Definition \ref{def:loglip_family} and to the existence of a safe element $\beta_{\min}=\inf\mathcal{B}\in\mathcal{B}$ with $\pi_t^{(\beta_{\min})}=\pi_0$. For such arbitrary $\beta$, we prove Theorem \ref{thm:cpc_general}; the result when $\beta$ is the likelihood-ratio bound is recovered as the special case treated in Theorem \ref{thm:cpc} (reproduced in \ref{cor:cpc_free}, which follows as a corollary from Theorem \ref{thm:cpc_general}). 

\vspace{0.5cm}

\begin{theorem}[Conformal Policy Control Guarantee for arbitrary $\beta$]
\label{thm:cpc_general}
    Assume that $\mathbb{E}_{\pi_0}[L_0]\leq \alpha$,  $L_i\in[0,B], B<\infty$ for all $i$, and that $\pi_0, \pi_1, ..., \{\pi_t^{(\beta)}\}_{\beta\in \mathcal{B}}$ are mutually absolutely continuous for all $\beta\in\mathcal{B}$. Assume that $\hat\beta$ (Eq. \eqref{eq:beta_hat_def_methods}) exists almost surely. If $\{\pi_t^{(\beta)}\}_{\beta\in\mathcal{B}}$ is $K$-log-Lipschitz in $\log\beta$ for some $K>0$ (Definition \ref{def:loglip_family}) and $\hat\beta$ is $\epsilon^{(K)}$-relative replace-one stable, then
    \begin{align*}
        \mathbb{E}_{\pi_{0:t}^{(\hat\beta)}}\big[L_{t} \big] \; \leq \; \alpha + \max(\alpha,\,B-\alpha)\,\epsilon^{(K)} .
    \end{align*}
\end{theorem}

As in Theorem \ref{thm:gcrc_nonmonotonic}, the guarantee is the target level, $\alpha$, plus a term that shrinks when the selection of the control parameter, $\hat\beta$, is more stable to replace-one changes in the calibration data, and when $\pi_t^{(\beta)}(z)$ changes more smoothly with respect to changes in $\beta$ (for all $z$). 

\vspace{0.5cm}

\begin{theorem}[Restatement of Theorem \ref{thm:cpc}: Conformal policy control with likelihood-ratio bound $\beta$ requires no smoothness assumption]
\label{cor:cpc_free}
    Assume that $\mathbb{E}_{\pi_0}[L_0]\leq \alpha$,  $L_i\in[0,B], B<\infty$ for all $i$, and that $\pi_0, \pi_1, ..., \{\pi_t^{(\beta)}\}_{\beta\in \mathcal{B}}$ are mutually absolutely continuous for all $\beta\in\mathcal{B}$. Assume that $\hat\beta$ (Eq. \eqref{eq:beta_hat_def_methods}) exists almost surely, for the likelihood-ratio control parameter as in Eq. \eqref{appeq:piecewise_constraint_def}. If  $\hat\beta$ is $\epsilon^{(1)}$-relative replace-one stable, then
    \begin{align*}
        \mathbb{E}_{\pi_{0:t}^{(\hat\beta)}}\big[L_{t} \big] \; \leq \; \alpha + \max(\alpha,\,B-\alpha)\,\epsilon^{(1)}.
    \end{align*}
\end{theorem}

\begin{proof} \qquad

For ease of notation, we conduct the proof in an online setting where only one action is taken (or sample is generated) at each policy improvement step. This will allow the indices $t=0, 1, ...$ to denote both the timestep of policy improvement and the timestep of the action/sample, but the result will hold generally for a batch setting too (with the guarantee still holding marginally over the whole procedure). 

We proceed by induction. For $t=0$, we assumed access to a safe initial policy, $\pi_0$, satisfying our risk control criterion. That is, we assume
\begin{align*}
    \mathbb{E}_{Z_0\sim \pi_0}[L_0] \leq \alpha.
\end{align*}
Now, for the induction hypothesis assume that risk control holds at time $t-1$, that is, assume
\begin{align*}
    \mathbb{E}_{Z_{0:t-1}\sim \pi_{0:t-1}^{(\hat\beta_1:\hat\beta_{t-1})}}[L_{t-1}] \leq \alpha.
\end{align*}

We then want to show that risk control also holds at time $t$:
\begin{align*}
    \mathbb{E}_{Z_{0:t}\sim \pi_{0:t}^{(\hat\beta_1:\hat\beta_t)}}[L_{t}] \leq \alpha,
\end{align*}
where note that $\pi_{0:t}^{(\hat\beta_t)}(z_{0:t}) = \pi_{t}^{(\hat\beta_t)}(z_t \mid z_{0:t-1})\pi_{t-1}^{(\hat\beta_1:\hat\beta_{t-1})}(z_{0:t-1})$. Note that at time $t$, we will only consider tuning $\hat\beta_t$, and we treat $\hat\beta_1, ..., \hat\beta_{t-1}$ as previously tuned. Accordingly, to lighten notation, we hereon denote the expectation we wish to bound as $\mathbb{E}_{\pi_{0:t}^{(\hat\beta)}}[L_{t}]:=\mathbb{E}_{Z_{0:t}\sim \pi_{0:t}^{(\hat\beta_1:\hat\beta_t)}}[L_{t}]$, where note that we also let $\hat\beta:=\hat\beta_t$.

Assume there is a safe $\beta_{\min}=\inf\mathcal{B}\in\mathcal{B}$ with $\pi_t^{(\beta_{\min})}=\pi_0$; when $\beta$ is the likelihood-ratio bound control parameter, assume $\beta_{\min}\ll1$ such that $\beta_{\min}\leq \pi_t(a\mid x)/\pi_0(a\mid x)$ for all $a\in \mathcal{A}$, $x\in \mathcal{X}$. 

\vspace{0.5cm}

\textbf{Step 0: Law of total expectation over draw of bag.}

Expanding $\mathbb{E}_{\pi_{0:t}^{(\hat\beta)}}[L_{t}]$ via the law of total expectation over the event $\lbag Z_{0:t} \rbag = \lbag z_{0:t} \rbag$, we have
\begin{align}\label{eq:cpc_step_0}
    \mathbb{E}_{\pi_{0:t}^{(\hat\beta)}}\big[L_{t}\big] = \mathbb{E}\Big[\mathbb{E}_{\pi_{0:t}^{(\hat\beta)}}\Big[L_{t} \mid \lbag Z_{0:t} \rbag = \lbag z_{0:t} \rbag\Big]\Big].
\end{align}
We can thus see that, to prove the desired result, it is sufficient to upper bound $\mathbb{E}_{\pi_{0:t}^{(\hat\beta)}}\big[L_{t} \mid \lbag Z_{0:t} \rbag = \lbag z_{0:t} \rbag\big]$ almost surely given any draw of the bag, $\lbag Z_{0:t}\rbag$.

\vspace{0.5cm}

\textbf{Step 1: Deterministic inequalities on bag-conditional risks.}

Define the following family of hyperparameters:
\begin{align*}
    \hat{\beta}(\lbag l_{1:m} \rbag, \tilde{w}_{1:m}^{(\cdot)},  \gamma) = \sup \Big\{\beta' \in \mathcal{B} \quad : \quad \forall \qquad \beta \leq \beta', \qquad \sum_{i=1}^ml_i\cdot \tilde{w}_i^{(\beta)} \leq \gamma \Big\}.
\end{align*}
When the set is empty, define $\hat{\beta}(\lbag l_{1:m} \rbag, \tilde{w}_{1:m}^{(\cdot)},  \gamma)=\beta_{\min}$.

By the continuity of $\tilde{w}_{i}^{(\cdot)}$ (and $L_i(\cdot)$) for all $i\in \mathbb{N}$, the sum $\sum_{i=1}^ml_i\cdot \tilde{w}_i^{(\beta)}$ is left-continuous for any $Z_{1:m}$, which gives us the following deterministic inequality for any $\gamma$:
\begin{align}
    \label{eq:cpc_step_1}
    \sum_{i=1}^ml_i\cdot \tilde{w}_i^{(\beta)} \leq \gamma \qquad \forall \qquad \beta \leq \hat{\beta}(\lbag L_{1:m} \rbag, \tilde{w}_{1:m}^{(\cdot)},  \gamma).
\end{align}

\vspace{1cm}

\textbf{Step 2: Showing an oracle procedure deterministically controls the bag-conditional risk.}

Using the family of thresholds we defined above and applying it to the bag of loss functions for both the calibration set and the test point, define the oracle threshold as
\begin{align}\label{eq:cpc_oracle_lambda_def}
    \beta^* := \hat{\beta}(\lbag L_{0:t} \rbag, \tilde{w}_{0:t}^{(\cdot)},  \alpha).
\end{align}

Now, consider the quantity $\mathbb{E}_{\pi_{0:t}^{(\beta^*)}}\big[L_{t} \mid \lbag Z_{0:t} \rbag = \lbag z_{0:t} \rbag\big]$, which we can view as the bag-conditional oracle risk. That is, we can interpret this quantity as the expected value, with respect to the policy $\pi_{0:t}^{(\beta^*)}$, that the random variable $L_{t}$ will take on, given the bag of data $\lbag z_{0:t} \rbag$. Note that conditioning on the bag of data, $\lbag z_{0:t} \rbag$, implies the event of observing the bag of loss functions, $\lbag l_{0:t} \rbag$,\footnote{That is, implies the event $\lbag L_{1}, ..., L_{t}\rbag = \lbag l_1, ..., l_{t} \rbag$, because the loss functions are defined as $l_i:=g(z_i)$ for all $i\in \mathbb{N}$.} which thereby determines the value of $\beta^*$, since $\beta^*$ is a function of $\lbag z_{0:t}\rbag$ (as well as the policy, $\pi_{0:t}^{(\cdot)}$, and the constraint function $g$). 

Using the definition of conditional expectation, and secondly using the definition of conditional probability, we have
\begin{align*}
    \mathbb{E}_{\pi_{0:t}^{(\beta^*)}}\big[L_{t} \mid \lbag Z_{0:t} \rbag = \lbag z_{0:t} \rbag\big] & = \sum_{i=0}^{t}l_i\cdot \mathbb{P}\big(Z_{t} = z_{i} \mid \lbag Z_{0:t} \rbag = \lbag z_{0:t} \rbag \big) \\
    & = \sum_{i=0}^{t}l_i\cdot \frac{\mathbb{P}\big(Z_{t}=z_{i}, \ \lbag Z_{0:t} \rbag = \lbag z_{0:t} \rbag \big)}{\mathbb{P}\big(\lbag Z_{0:t} \rbag = \lbag z_{0:t} \rbag\big)}.
\end{align*}

Applying the law of total probability over permutations ($\sigma:\{0, ..., t\}\rightarrow\{0, ..., t\}$) to the probabilities in the numerator and denominator of the RHS, we have
\begin{align*}
    \mathbb{E}_{\pi_{0:t}^{(\beta^*)}}\big[L_{t} \mid \lbag Z_{0:t} \rbag =  \lbag z_{0:t} \rbag\big]
    & = \sum_{i=0}^{t}l_i\cdot \frac{\sum_{\sigma:\sigma(t)=i}\pi_{0:t}^{(\beta^*)}(z_{\sigma(0)}, ..., z_{\sigma(t)})}{\sum_{\sigma}\pi_{0:t}^{(\beta^*)}(z_{\sigma(0)}, ..., z_{\sigma(t)})} \\
    & = \sum_{i=0}^{t}l_i\cdot \tilde{w}_i^{(\beta^*)}
\end{align*}

Now, using Eq. \eqref{eq:cpc_step_1}, 
\begin{align}\label{eq:cpc_step_2}
    &\sum_{i=0}^{t}l_i\cdot \tilde{w}_i^{(\beta)}\leq \alpha \qquad \forall \qquad \text{fixed } \beta \leq \beta^*,
\end{align}
where here by ``fixed'' we mean that $\beta$ is constant conditional on the bag of data, $\lbag Z_{0:t} \rbag = \lbag z_{0:t} \rbag$.

\vspace{1cm}

\textbf{Step 3: Relating the empirically-selected hyperparameter to the oracle parameter}

In Step 2, we have shown that an oracle procedure deterministically controls the risk, conditional on the event $\lbag Z_{0:t} \rbag = \lbag z_{0:t} \rbag$. However, this oracle procedure cannot be implemented in practice due to requiring knowledge of the test point loss, $L_{t}$, as well as the weight assigned to the test point (as a function of $\beta$), $\tilde{w}_{t}^{\cdot}$; that is, recall we defined $\beta^* := \hat{\beta}(\lbag L_{0:t} \rbag, \tilde{w}_{0:t}^{(\cdot)},  \alpha)$. In this next step of the proof we would now like to prove an almost-sure relationship between the oracle parameter, $\beta^*$, and the empirically-selected parameter, $\hat{\beta}:=\hat{\beta}(\lbag L_{0:t-1}, B \rbag, \tilde{w}_{0:t-1}^{(\cdot)}, \tilde{w}_{\max}^{(\cdot)},  \alpha)$, to conclude valid risk control for the empirical method; that is, we would like to show that $\beta^* \geq \hat{\beta}$ almost surely.

To begin, recall that in Step 2, conditioning on the event $\lbag Z_{0:t} \rbag = \lbag z_{0:t} \rbag$ implied that $\hat{\beta}(\lbag L_{0:t} \rbag, \tilde{w}_{0:t}^{(\cdot)},  \alpha) = \hat{\beta}(\lbag l_{0:t} \rbag, \tilde{w}_{0:t}^{(\cdot)},  \alpha)$. For the empirical parameter, however, conditioning on $\lbag Z_{0:t} \rbag = \lbag z_{0:t} \rbag$ does \textit{not} imply what may seem analogous: that is, it could be that $\hat{\beta}(\lbag L_{0:t-1}, B \rbag, \tilde{w}_{0:t-1}^{(\cdot)}, \tilde{w}_{\max}^{(\cdot)}, \alpha) \neq \hat{\beta}(\lbag l_{0:t-1}, B \rbag, \tilde{w}_{0:t-1}^{(\cdot)}, \tilde{w}_{\max}^{(\cdot)}, \alpha)$. This is because the event $\lbag Z_{0:t} \rbag = \lbag z_{0:t} \rbag$ only allows us to conclude that the $t$ random variables in $\lbag L_{0:t-1}\rbag$ take $t$ distinct values from the bag of $t+1$ observed losses, $\lbag l_{0:t}\rbag$, but we do not yet know \textit{which} observed losses. That is, what we \textit{can} conclude is, given $\lbag Z_{0:t} \rbag = \lbag z_{0:t} \rbag$, that
\begin{align*}
    \lbag Z_{0:t} \rbag = \lbag z_{0:t} \rbag \implies \lbag L_{0:t-1} \rbag \in \Big\{\lbag l_{0:t\backslash i} \rbag \Big\}_{i\in \{0, ..., t\}},
\end{align*}
where $\lbag l_{0:t\backslash i} \rbag := \lbag l_{0:t} \rbag \backslash \{l_i\}$. Accordingly, conditioning on the event $\lbag Z_{0:t} \rbag = \lbag z_{0:t} \rbag$ implies the following equivalence:
\begin{align}\label{eq:cpc_oracle_empirical_inequality_equivalences}
    \beta^* \geq \hat{\beta}(\lbag L_{0:t-1}, B \rbag, \tilde{w}_{0:t-1}^{(\cdot)}, \tilde{w}_{\max}^{(\cdot)}, \alpha) \ \text{almost surely} \ & \iff \beta^* \geq \hat{\beta}(\lbag l_{0:t\backslash i}, B \rbag, \tilde{w}_{0:t-1}^{(\cdot)}, \tilde{w}_{\max}^{(\cdot)}, \alpha) \quad \forall \quad i \in \{0, ..., t\}.
\end{align}
So, we will focus on proving the latter.

Recall that we have assumed that $l$ is bounded above by $B$, and note that $\tilde{w}_i^{(\beta)} \leq \tilde{w}_{\max}^{(\beta)}$ for any $\beta\in \mathbb{R}$ (by the definition of $\tilde{w}_{\max}^{(\cdot)}$), so almost surely we have
\begin{align}\label{eq:cpc_step_3_conservative_adjustment}
    \sum_{j\in\{0, ..., t\}}l_j\cdot\tilde{w}_{j}^{(\beta)}
    \quad \leq \quad  B\cdot \tilde{w}_{\max}^{(\beta)} + \sum_{j\in\{0, ..., t\}\backslash i}l_j\cdot\tilde{w}_{j}^{(\beta)}  \qquad \forall \qquad i \in \{0, ..., t\},
\end{align}

where the right-hand side of the above inequalities can be interpreted as a conservative risk obtained by replacing $l_i\cdot \tilde{w}_i^{(\beta)}$ with the weighted bound $B\cdot \tilde{w}_{\max}^{(\beta)}$.

Thus, for any $\beta$, we can use \eqref{eq:cpc_step_3_conservative_adjustment} to conclude
\begin{align}\label{eq:cpc_step_3_risk_relation}
    B\cdot \tilde{w}_{\max}^{(\beta)} + \sum_{j\in\{0, ..., t\}\backslash i}l_j\cdot\tilde{w}_{j}^{(\beta)} \leq \alpha \implies & \sum_{j\in\{0, ..., t\}}l_j\cdot\tilde{w}_{j}^{(\beta)} \leq \alpha \qquad \forall \qquad i \in \{0, ..., t\}.
\end{align} 

So, by Eq. \eqref{eq:cpc_step_3_risk_relation} and the definition of $\hat{\beta}(\lbag l_{0:t\backslash i}, B \rbag)$, it must hold that 
\begin{align}\label{eq:cpc_pf_step_3_result}
    \beta^* \geq & \ \hat{\beta}(\lbag l_{0:t\backslash i}, B \rbag, \tilde{w}_{0:t-1}^{(\cdot)}, \tilde{w}_{\max}^{(\cdot)}, \alpha) \quad \forall \quad i \in \{0, ..., t\}.
\end{align}

\vspace{1cm}

\textbf{Step 4: Analyzing the bag-conditional risk of the empirical procedure}

Now we turn to analyzing the risk associated with the empirical algorithm's hyperparameter when plugged into the bag-conditional risk. 
This analysis is necessary because $\hat\beta=\hat{\beta}(\lbag L_{0:t-1}, B \rbag, \tilde{w}_{0:t-1}^{(\cdot)}, \tilde{w}_{\max}^{(\cdot)}, \alpha)$ is not constant conditional on $\lbag Z_{0:t} \rbag = \lbag z_{0:t} \rbag$, so Eq. \eqref{eq:cpc_step_2} would not directly apply. 

Begin by expanding the quantity $\mathbb{E}_{\pi_{0:t}^{(\hat\beta)}}\Big[L_{t} \mid \lbag Z_{0:t} \rbag = \lbag z_{0:t} \rbag\Big]$ by LOTE over the event $Z_{t}=z_i$:
\begin{align}\label{eq:cpc_empirical_bag_cond_LOTE}
    \mathbb{E}_{\pi_{0:t}^{(\hat\beta)}}\Big[&L_{t} \mid \lbag Z_{0:t} \rbag = \lbag z_{0:t} \rbag\Big] \nonumber \\
    & = \mathbb{E}_{\pi_{0:t}^{(\hat\beta)}}\Big[\mathbb{E}_{\pi_{0:t}(\hat{\beta}(\lbag L_{0:t-1} \rbag))}\Big[L_{t} \mid \lbag Z_{0:t} \rbag = \lbag z_{0:t} \rbag, Z_{t}=z_i\Big] \mid \lbag Z_{0:t} \rbag = \lbag z_{0:t} \rbag\Big] \nonumber \\
    & = \sum_{i=0}^{t}\mathbb{E}_{\pi_{0:t}^{(\hat\beta)}}\Big[L_{t} \mid \lbag Z_{0:t} \rbag = \lbag z_{0:t} \rbag, Z_{t}=z_i\Big]\cdot \mathbb{P}_{\pi_{0:t}^{(\hat\beta)}}\big(Z_{t}=z_i \mid \lbag Z_{0:t} \rbag = \lbag z_{0:t} \rbag\big) \nonumber \\
    & = \sum_{i=0}^{t}l_i\cdot \mathbb{P}_{\pi_{0:t}^{(\hat\beta)}}\big(Z_{t}=z_i \mid \lbag Z_{0:t} \rbag = \lbag z_{0:t} \rbag\big).
\end{align}

Writing out the conditional probability in Eq. \eqref{eq:cpc_empirical_bag_cond_LOTE} using the definition of conditional probability, and then the law of total probability over $\sigma$ and the chain rule (in both the numerator and denominator),
\begin{align*}
    & \mathbb{P}_{\pi_{0:t}^{(\hat\beta)}}\big(Z_{t}=z_i \mid \lbag Z_{0:t} \rbag = \lbag z_{0:t} \rbag\big) \\
    & = \frac{\sum_{\sigma:\sigma(t)=i}\pi_{0:t}^{(\hat\beta)}(z_{\sigma(0)}, ..., z_{\sigma(t)})}{\sum_{\sigma}\pi_{0:t}^{(\hat\beta)}(z_{\sigma(0)}, ..., z_{\sigma(t)})} \\
    & = \frac{\sum_{\sigma:\sigma(t)=i}\pi_{t\mid 0:t-1}^{(\hat\beta)}(z_{\sigma(t)} \mid z_{\sigma(0)}, ..., z_{\sigma(t-1)})\cdot \pi_{0:t-1}(z_{\sigma(0)}, ..., z_{\sigma(t-1)})}{\sum_{\sigma}\pi_{t\mid 0:t-1}^{(\hat\beta)}(z_{\sigma(t)} \mid z_{\sigma(0)}, ..., z_{\sigma(t-1)})\cdot \pi_{0:t-1}(z_{\sigma(0)}, ..., z_{\sigma(t-1)})},
\end{align*}
and by expanding the summation in the denominator, this is equivalent to 
\begin{align*}
    & \mathbb{P}_{\pi_{0:t}^{(\hat\beta)}}\big(Z_{t}=z_i \mid \lbag Z_{0:t} \rbag = \lbag z_{0:t} \rbag\big) \\
    & = \frac{\sum_{\sigma:\sigma(t)=i}\pi_{t\mid 0:t-1}^{(\hat\beta)}(z_{\sigma(t)} \mid z_{\sigma(0)}, ..., z_{\sigma(t-1)})\cdot \pi_{0:t-1}(z_{\sigma(0)}, ..., z_{\sigma(t-1)})}{\sum_{i}\sum_{\sigma:\sigma(t)=i}\pi_{t\mid 0:t-1}^{(\hat\beta)}(z_{\sigma(t)} \mid z_{\sigma(0)}, ..., z_{\sigma(t-1)})\cdot \pi_{0:t-1}(z_{\sigma(0)}, ..., z_{\sigma(t-1)})}.
\end{align*}
In this expanded form, examine the factor $\pi_{t\mid 0:t-1}^{(\hat\beta)}(z_{\sigma(t)} \mid z_{\sigma(0)}, ..., z_{\sigma(t-1)})$, and recalling that $\hat\beta:= \hat{\beta}(\lbag L_{0:t-1}, B \rbag, \tilde{w}_{0:t-1}^{(\cdot)}, \tilde{w}_{\max}^{(\cdot)}, \alpha)$, note that the conditioning on $z_{\sigma(0)}, ..., z_{\sigma(t-1)}$ inside the summation $\sum_{\sigma:\sigma(t)=i}$, where the permutations are $\sigma:\sigma(t)=i$, implies  $Z_t=z_{\sigma(t)}=z_i \implies \lbag L_{0:t-1} \rbag = \lbag l_{0:t\backslash i} \rbag$ (that is, implies a specific realization of $\lbag L_{0:t-1} \rbag$ among $t$ possible bags of observed losses). Accordingly, let us denote $\hat\beta^{\backslash i}:= \hat\beta^{\backslash z_i} =\hat{\beta}(\lbag l_{0:t\backslash i}, B \rbag, \tilde{w}_{0:t-1}^{(\cdot)}, \tilde{w}_{\max}^{(\cdot)}, \alpha)$ (as in Definition \ref{def:rel_replace_one_stability}). Then, note that the numerator and the analogous inner summation in the denominator are the unnormalized weights in Eq. \eqref{eq:cp_weights_def_cpc_unnormed}, that is, we can concisely write those summations as $\sum_{\sigma:\sigma(t)=i}(\cdots) = w_i^{(\beta^{\backslash i})}$. That is, we have
\begin{align}
    & \mathbb{P}_{\pi_{0:t}^{(\hat\beta)}}\big(Z_{t}=z_i \mid \lbag Z_{0:t} \rbag = \lbag z_{0:t} \rbag\big)  = \frac{w_i^{(\beta^{\backslash i})}}{\sum_{j}w_j^{(\beta^{\backslash j})}}.
\end{align}

Plugging this into Eq. \eqref{eq:cpc_empirical_bag_cond_LOTE}, we have
\begin{align}\label{eq:empirical_risk_bag_conditional}
    \mathbb{E}_{\pi_{0:t}^{(\hat\beta)}}\Big[L_{t} \mid \lbag Z_{0:t} \rbag = \lbag z_{0:t} \rbag\Big] = 
    \sum_{i=0}^{t}l_i\cdot \frac{w_i^{(\beta^{\backslash i})}}{\sum_{j}w_j^{(\beta^{\backslash j})}} = \frac{1}{\sum_{j}w_j^{(\beta^{\backslash j})}}\sum_{i=0}^{t}l_i\cdot w_i^{(\beta^{\backslash i})}
\end{align}

Importantly, note that in Eq. \eqref{eq:empirical_risk_bag_conditional}, which can be viewed as the true bag-conditional risk of the empirical algorithm, that the
unnormalized weight in the numerator, $w_i^{(\beta^{\backslash i})}$, differs for each loss function, $l_i$, in the summation, and moreover the same is true in the denominator's summation, 
which is why we cannot apply Eq. \eqref{eq:cpc_step_2} to bound it directly.

Instead, we now move toward relating Eq. \eqref{eq:empirical_risk_bag_conditional} and Eq. \eqref{eq:cpc_step_2} using smoothness and stability conditions. However, although in the analogous step assuming  parameterized losses (Section \ref{app:gcrc_proof_nonmonotonic_1dim}) we were able to directly apply Lipschitz continuity at this stage, here we must be more careful. That is, although in Section \ref{app:gcrc_proof_nonmonotonic_1dim} we could conservatively assume that each replace-one perturbation $\epsilon_i$ increased its corresponding loss by at most $K\epsilon_i$, a subtle difference here is that if we artificially increased
the unnormalized weight of certain (e.g., smaller-loss) indices by some amount, this may \textit{decrease} the overall bag-conditional weighted risk after re-normalizing the probability weights, and thus not result in a valid upper bound. In other words, conservative estimates of losses monotonically result in conservative estimates of the risk (in Section \ref{app:gcrc_proof_nonmonotonic_1dim}), but here, conservative estimates of (unnormalized) weights introduce an additional source of non-monotonicity that may not result in a conservative estimate of the weighted (bag-conditional) risk.

Making the substitution $w_i^{(\hat\beta^{\backslash i})} = w_i^{(\hat\beta^{\backslash i})} + (w_i^{(\hat\beta^{\backslash t})} - w_i^{(\hat\beta^{\backslash t})})$ in Eq. \eqref{eq:empirical_risk_bag_conditional} and then rearranging terms, we have

\begin{align*}
    \mathbb{E}_{\pi_{0:t}^{(\hat\beta)}}\Big[L_{t} \mid \lbag Z_{0:t} \rbag = \lbag z_{0:t} \rbag\Big] 
    = & \frac{1}{\sum_{j}w_j^{(\hat\beta^{\backslash j})}}\sum_{i=0}^{t}l_i\cdot w_i^{(\hat\beta^{\backslash t})} +  (l_i\cdot w_i^{(\hat\beta^{\backslash i})} - l_i\cdot w_i^{(\hat\beta^{\backslash t})}) \nonumber \\
    = & \frac{1}{\sum_{j}w_j^{(\hat\beta^{\backslash j})}}\sum_{i=0}^{t}l_i\cdot w_i^{(\hat\beta^{\backslash t})} +  \frac{1}{\sum_{j}w_j^{(\hat\beta^{\backslash j})}}\sum_{i=0}^{t}(l_i\cdot w_i^{(\hat\beta^{\backslash i})} - l_i\cdot w_i^{(\hat\beta^{\backslash t})}).
\end{align*}

Now multiplying the first term by $1 = \frac{\sum_{j}w_j^{(\hat\beta^{\backslash t})}}{\sum_{j}w_j^{(\hat\beta^{\backslash t})}}$,
\begin{align*}
    \mathbb{E}_{\pi_{0:t}^{(\hat\beta)}}\Big[L_{t} \mid \lbag Z_{0:t} \rbag = \lbag z_{0:t} \rbag\Big] 
    = & \frac{\sum_{j}w_j^{(\hat\beta^{\backslash t})}}{\sum_{j}w_j^{(\hat\beta^{\backslash t})}}\frac{1}{\sum_{j}w_j^{(\hat\beta^{\backslash j})}}\sum_{i=0}^{t}l_i\cdot w_i^{(\hat\beta^{\backslash t})} +  \frac{1}{\sum_{j}w_j^{(\hat\beta^{\backslash j})}}\sum_{i=0}^{t}(l_i\cdot w_i^{(\hat\beta^{\backslash i})} - l_i\cdot w_i^{(\hat\beta^{\backslash t})}) \\
    = & \frac{\sum_{j}w_j^{(\hat\beta^{\backslash t})}}{\sum_{j}w_j^{(\hat\beta^{\backslash j})}}\sum_{i=0}^{t}\frac{l_i\cdot w_i^{(\hat\beta^{\backslash t})}}{\sum_{j}w_j^{(\hat\beta^{\backslash t})}} +  \frac{1}{\sum_{j}w_j^{(\hat\beta^{\backslash j})}}\sum_{i=0}^{t}(l_i\cdot w_i^{(\hat\beta^{\backslash i})} - l_i\cdot w_i^{(\hat\beta^{\backslash t})}),
\end{align*}

Then by Eq. \eqref{eq:cpc_step_2}, we know that $\sum_{i=0}^{t}\frac{l_i\cdot w_i^{(\hat\beta^{\backslash t})}}{\sum_{j}w_j^{(\hat\beta^{\backslash t})}} \leq \alpha$ (note that $\hat\beta^{\backslash t}$ is fixed with respect to the summations in the term), so
\begin{align}\label{eq:risk_bound_alpha_sub}
    \mathbb{E}_{\pi_{0:t}^{(\hat\beta)}}\Big[L_{t} \mid \lbag Z_{0:t} \rbag = \lbag z_{0:t} \rbag\Big] 
    \leq & \ \frac{\sum_{j}w_j^{(\hat\beta^{\backslash t})}}{\sum_{j}w_j^{(\hat\beta^{\backslash j})}}\cdot\alpha +  \frac{1}{\sum_{j}w_j^{(\hat\beta^{\backslash j})}}\sum_{i=0}^{t}(l_i\cdot w_i^{(\hat\beta^{\backslash i})} - l_i\cdot w_i^{(\hat\beta^{\backslash t})}).
\end{align}

Plugging in $w_j^{(\hat\beta^{\backslash t})} = w_j^{(\hat\beta^{\backslash t})} + (w_j^{(\hat\beta^{\backslash j})} - w_j^{(\hat\beta^{\backslash j})})$ into the numerator of the first term of the RHS of \eqref{eq:risk_bound_alpha_sub} and rearranging terms,

\begin{align}\label{eq:cpc_pf_pre_branch_point}
    \frac{\sum_{j}w_j^{(\hat\beta^{\backslash t})}}{\sum_{j}w_j^{(\hat\beta^{\backslash j})}}\cdot\alpha & +  \frac{1}{\sum_{j}w_j^{(\hat\beta^{\backslash j})}}\sum_{i=0}^{t}(l_i\cdot w_i^{(\hat\beta^{\backslash i})} - l_i\cdot w_i^{(\hat\beta^{\backslash t})}) \nonumber \\
    = & \frac{\sum_{j}(w_j^{(\hat\beta^{\backslash t})} + w_j^{(\hat\beta^{\backslash j})} - w_j^{(\hat\beta^{\backslash j})})}{\sum_{j}w_j^{(\hat\beta^{\backslash j})}}\cdot\alpha +  \frac{1}{\sum_{j}w_j^{(\hat\beta^{\backslash j})}}\sum_{i=0}^{t}(l_i\cdot w_i^{(\hat\beta^{\backslash i})} - l_i\cdot w_i^{(\hat\beta^{\backslash t})}) \nonumber \\
    = & \ \alpha + \frac{\sum_{i=0}^{t}\alpha\cdot w_i^{(\hat\beta^{\backslash t})} - \alpha\cdot w_i^{(\hat\beta^{\backslash i})} + l_i\cdot w_i^{(\hat\beta^{\backslash i})} - l_i\cdot w_i^{(\hat\beta^{\backslash t})}}{\sum_{j}w_j^{(\hat\beta^{\backslash j})}} \nonumber \\
    = & \ \alpha + \frac{\sum_{i=0}^{t}l_i\cdot (w_i^{(\hat\beta^{\backslash i})} - w_i^{(\hat\beta^{\backslash t})}) - \alpha\cdot (w_i^{(\hat\beta^{\backslash i})} -  w_i^{(\hat\beta^{\backslash t})})}{\sum_{j}w_j^{(\hat\beta^{\backslash j})}} \nonumber \\
    = & \ \alpha + \frac{\sum_{i=0}^{t}(l_i - \alpha)\cdot (w_i^{(\hat\beta^{\backslash i})} -  w_i^{(\hat\beta^{\backslash t})})}{\sum_{j}w_j^{(\hat\beta^{\backslash j})}} 
\end{align}

By the absolute continuity of $\pi_{0}, \pi_1^{(\hat\beta_1)}, ..., \{\pi_t^{(\beta)}\}_{\beta\in\mathcal{\beta}}$ for all $\beta \in \mathcal{B}$, $w_i^{(\hat\beta^{\backslash t})}=0 \iff w_i^{(\hat\beta^{\backslash i})}=0$. When $w_i^{(\hat\beta^{\backslash t})}= w_i^{(\hat\beta^{\backslash i})}=0$ the bound holds trivially, so we focus on the case where $w_i^{(\hat\beta^{\backslash i})}>0$, which allows us to write

\begin{align}\label{eq:cpc_pf_branch_point}
    \frac{\sum_{j}w_j^{(\hat\beta^{\backslash t})}}{\sum_{j}w_j^{(\hat\beta^{\backslash j})}}\cdot\alpha +  \frac{1}{\sum_{j}w_j^{(\hat\beta^{\backslash j})}}\sum_{i=0}^{t}(l_i\cdot w_i^{(\hat\beta^{\backslash i})} - l_i\cdot w_i^{(\hat\beta^{\backslash t})}) & = \ \alpha + \frac{\sum_{i=0}^{t}(l_i - \alpha)\cdot \big(1 -  \frac{w_i^{(\hat\beta^{\backslash t})}}{w_i^{(\hat\beta^{\backslash i})}}\big)\cdot w_i^{(\hat\beta^{\backslash i})}}{\sum_{j}w_j^{(\hat\beta^{\backslash j})}} \nonumber \\
    & = \ \alpha + \sum_{i=0}^{t}(l_i - \alpha)\cdot \big(1 -  \frac{w_i^{(\hat\beta^{\backslash t})}}{w_i^{(\hat\beta^{\backslash i})}}\big)\cdot \tilde{w}_i^{(\hat\beta^{\backslash i})}
\end{align}

Let us pause to examine the term $\big(1 -  \frac{w_i^{(\hat\beta^{\backslash t})}}{w_i^{(\hat\beta^{\backslash i})}}\big)$ in Eq. \eqref{eq:cpc_pf_branch_point}. By definition of $w_i^{(\beta)}$ and the chain rule of probability, 

\begin{align*}
    1 -  \frac{w_i^{(\hat\beta^{\backslash t})}}{w_i^{(\hat\beta^{\backslash i})}} & = 1 - \frac{\sum_{\sigma:\sigma(t)=i}\pi_{0:t}^{(\hat\beta^{\backslash t})}(z_{\sigma(0)}, ..., z_{\sigma(t)})}{\sum_{\sigma:\sigma(t)=i}\pi_{0:t}^{(\hat\beta^{\backslash i})}(z_{\sigma(0)}, ..., z_{\sigma(t)})} \\
    & = 1 - \frac{\sum_{\sigma:\sigma(t)=i}\pi_{t\mid 0:t-1}^{(\hat\beta^{\backslash t})}(z_{\sigma(t)} \mid z_{\sigma(0)}, ..., z_{\sigma(t-1)})\cdot \pi_{0:t-1}( z_{\sigma(0)}, ..., z_{\sigma(t-1)})}{\sum_{\sigma:\sigma(t)=i}\pi_{t\mid 0:t-1}^{(\hat\beta^{\backslash i})}(z_{\sigma(t)} \mid z_{\sigma(0)}, ..., z_{\sigma(t-1)})\cdot \pi_{0:t-1}( z_{\sigma(0)}, ..., z_{\sigma(t-1)})}.
\end{align*}

In the split conformal setting we are concerned with (where calibration and training data are disjoint), the constrained policy $\pi_{t\mid 0:t-1}^{(\beta)}$ is symmetric with respect to the calibration data, so for all $\beta$ we have $\pi_{t\mid 0:t-1}^{(\beta)}(z_{\sigma(t)} \mid z_{\sigma(0)}, ..., z_{\sigma(t-1)}) = \pi_{t\mid 0:t-1}^{(\beta)}(z_{i} \mid \lbag z_{0:t}\rbag\backslash z_i)$ for all permutations $\sigma:\sigma(t)=i$ within the summation. With this substitution and abbreviating notation $\pi_t^{(\beta)}(z_{i}):=\pi_{t\mid 0:t-1}^{(\beta)}(z_{i} \mid \lbag z_{0:t}\rbag\backslash z_i)$,

\begin{align*}
    1 -  \frac{w_i^{(\hat\beta^{\backslash t})}}{w_i^{(\hat\beta^{\backslash i})}} 
    & = 1 - \frac{\pi_{t}^{(\hat\beta^{\backslash t})}(z_{i})\cdot\sum_{\sigma:\sigma(t)=i} \pi_{0:t-1}( z_{\sigma(0)}, ..., z_{\sigma(t-1)})}{\pi_{t}^{(\hat\beta^{\backslash i})}(z_{i})\cdot\sum_{\sigma:\sigma(t)=i} \pi_{0:t-1}( z_{\sigma(0)}, ..., z_{\sigma(t-1)})} \\
    & = 1 - \frac{\pi_{t}^{(\hat\beta^{\backslash t})}(z_{i})}{\pi_{t}^{(\hat\beta^{\backslash i})}(z_{i})} 
\end{align*}

Equation \eqref{eq:cpc_pf_branch_point} then becomes

\begin{align*}
    \frac{\sum_{j}w_j^{(\hat\beta^{\backslash t})}}{\sum_{j}w_j^{(\hat\beta^{\backslash j})}}\cdot\alpha & +  \frac{1}{\sum_{j}w_j^{(\hat\beta^{\backslash j})}}\sum_{i=0}^{t}(l_i\cdot w_i^{(\hat\beta^{\backslash i})} - l_i\cdot w_i^{(\hat\beta^{\backslash t})}) =  \ \alpha + \sum_{i=0}^{t}(l_i - \alpha)\cdot \big(1 -  \frac{\pi_t^{(\hat\beta^{\backslash t})}(z_{i})}{\pi_t^{(\hat\beta^{\backslash i})}(z_i)}\big)\cdot \tilde{w}_i^{(\hat\beta^{\backslash i})},
\end{align*}

and combining this with Inequality \eqref{eq:risk_bound_alpha_sub} yields
\begin{align}\label{eq:cpc_pf_branch_point_policy_sub}
    \mathbb{E}_{\pi_{0:t}^{(\hat\beta)}}\Big[L_{t} \mid \lbag Z_{0:t} \rbag = \lbag z_{0:t} \rbag\Big] 
    \leq & \ \alpha + \sum_{i=0}^{t}(l_i - \alpha)\cdot \big(1 -  \frac{\pi_t^{(\hat\beta^{\backslash t})}(z_{i})}{\pi_t^{(\hat\beta^{\backslash i})}(z_i)}\big)\cdot \tilde{w}_i^{(\hat\beta^{\backslash i})}.
\end{align}

\vspace{1cm}

\textbf{Step 5: Applying smoothness assumption to bound relative perturbation terms.}

Recall that by assumption, $\{\pi_t^{(\beta)}\}_{\beta\in\mathcal{B}}$ is $K$-log-Lipschitz in $\log\beta$ for some $K>0$. That is, recalling Definition \ref{def:loglip_family},
\begin{align*}
    \Big|\log \pi_t^{(\beta')}(z)-\log \pi_t^{(\beta)}(z)\Big| \; \leq \; K\,\big|\log\beta'-\log\beta\big|
    \qquad \forall \; \beta,\beta'\in\mathcal{B}, \; z\in\mathcal{Z} .
\end{align*}

Noting the identity $|\log(r)|=\log(\max(r, 1/r))$ for $r>0$, and letting $r = \beta'/\beta$ (recall $\beta'>0,\beta>0$),  we have

\begin{align*}
    \bigg|\log \frac{\pi_t^{(\beta')}(z)}{\pi_t^{(\beta)}(z)}\bigg| \; \leq \; K\cdot\log\Big(\max\Big(\frac{\beta'}{\beta}, \frac{\beta}{\beta'}\Big)\Big)
    \qquad \forall \; \beta,\beta'\in\mathcal{B}, \; z\in\mathcal{Z} .
\end{align*}

Define the multiplicative gap $\kappa_i=\max\big(\hat\beta^{\backslash t}/\hat\beta^{\backslash i},\,\hat\beta^{\backslash i}/\hat\beta^{\backslash t}\big)\geq1$, and letting $\beta'=\hat\beta^{\backslash t}$,  $\beta=\hat\beta^{\backslash i}$ and $z=z_i$, this yields

\begin{align}\label{eq:cpc_ratio_two_sided}
    \bigg|\log \frac{\pi_t^{(\hat\beta^{\backslash t})}(z_i)}{\pi_t^{(\hat\beta^{\backslash i})}(z_i)}\bigg| \; & \leq \; K\log\kappa_i
    \qquad \forall \; i\in\{0, ..., t\} \nonumber \\
    \kappa_i^{-K} \; \leq  \; \frac{\pi_t^{(\hat\beta^{\backslash t})}(z_i)}{\pi_t^{(\hat\beta^{\backslash i})}(z_i)} \; & \leq \; \kappa_i^K
    \qquad \forall \; i\in\{0, ..., t\} \nonumber \\
    1 - \kappa_i^{K} \; \leq  \; 1 - \frac{\pi_t^{(\hat\beta^{\backslash t})}(z_i)}{\pi_t^{(\hat\beta^{\backslash i})}(z_i)} \; & \leq \; 1 - \kappa_i^{-K}
    \qquad \forall \; i\in\{0, ..., t\}.
\end{align}


We now turn to bounding each summand term in the RHS of \eqref{eq:cpc_pf_branch_point_policy_sub}. We treat the two sign cases of $(l_i-\alpha)$ separately, which require opposite sides of \eqref{eq:cpc_ratio_two_sided}.
We will also use the observation that $0 \; \leq \; 1-\kappa_i^{-K} \; \leq \; \kappa_i^{K}-1$ (since $\kappa_i\geq1$, the arithmetic-geometric mean inequality gives $\kappa_i^{K}+\kappa_i^{-K}\geq2$). 

\begin{itemize}
    
\item \emph{Case 1: $l_i\geq\alpha$.}  
Taking the upper bound of \eqref{eq:cpc_ratio_two_sided} and multiplying each side by the nonnegative factor $(l_i-\alpha)$, and by our observation $0 \; \leq \; 1-\kappa_i^{-K} \; \leq \; \kappa_i^{K}-1$, 
\begin{align*}
    (l_i-\alpha)\cdot\bigg(1 - \frac{\pi_t^{(\hat\beta^{\backslash t})}(z_i)}{\pi_t^{(\hat\beta^{\backslash i})}(z_i)}\bigg) \; \leq \; (l_i-\alpha)\cdot\big(1-\kappa_i^{-K}\big) \; \leq \; |l_i-\alpha|\cdot\big(\kappa_i^{K}-1\big) .
\end{align*}

\item \emph{Case 2: $l_i<\alpha$.} Here, note $(l_i-\alpha)(1-\frac{\pi_t^{(\hat\beta^{\backslash t})}(z_i)}{\pi_t^{(\hat\beta^{\backslash i})}(z_i)})=(\alpha-l_i)(\frac{\pi_t^{(\hat\beta^{\backslash t})}(z_i)}{\pi_t^{(\hat\beta^{\backslash i})}(z_i)}-1)$, and from negating the lower bound of \eqref{eq:cpc_ratio_two_sided},
\begin{align*}
    \frac{\pi_t^{(\hat\beta^{\backslash t})}(z_i)}{\pi_t^{(\hat\beta^{\backslash i})}(z_i)} - 1 \leq \kappa_i^{K} - 1.
\end{align*}
Multiplying each side of this bound by the positive factor 
$(\alpha-l_i)=|l_i-\alpha|$, 
and using our observation that $0\leq \kappa_i^K-1$,
\begin{align*}
    (l_i-\alpha)\cdot\bigg(1-\frac{\pi_t^{(\hat\beta^{\backslash t})}(z_i)}{\pi_t^{(\hat\beta^{\backslash i})}(z_i)}\bigg)=(\alpha-l_i)\cdot\bigg(\frac{\pi_t^{(\hat\beta^{\backslash t})}(z_i)}{\pi_t^{(\hat\beta^{\backslash i})}(z_i)}-1\bigg) \; \leq \; |l_i-\alpha|\cdot\,\big(\kappa_i^{K}-1\big) .
\end{align*}
\end{itemize}

So, in both cases, the summand is at most $|l_i-\alpha|$ times $\kappa_i^{K}-1$.  It remains only to bound the coefficient uniformly over the two cases, and since $l_i\in[0,B]$,
\begin{align*}
    |l_i-\alpha| \; \leq \; \max_{l\in[0,B]}|l-\alpha| \; = \; \max(\alpha,\,B-\alpha),
\end{align*}
the maximum being attained at $l=0$ or $l=B$.  Combining, we obtain the pointwise bound
\begin{align}\label{eq:cpc_summand_bound}
    (l_i-\alpha)\cdot\Big(1-\frac{\pi_t^{(\hat\beta^{\backslash t})}(z_{i})}{\pi_t^{(\hat\beta^{\backslash i})}(z_i)}\Big)
    \; \leq \; \max(\alpha,\,B-\alpha)\;\big(\kappa_i^{K}-1\big)
    \qquad \forall \quad i \in\{0,\dots,t\}.
\end{align}

\vspace{0.5cm}

\textbf{Step 6: Marginalizing and applying the relative replace-one stability assumption.}

Substituting \eqref{eq:cpc_summand_bound} into Inequality \eqref{eq:cpc_pf_branch_point_policy_sub}, and using that the normalized conformal weights are nonnegative and satisfy $\sum_{i=0}^{t}\tilde{w}_i^{(\hat\beta^{\backslash i})}=1$,
\begin{align}\label{eq:cpc_rel_bag_conditional}
    \mathbb{E}_{\pi_{0:t}^{(\hat\beta)}}\Big[L_{t} \mid \lbag Z_{0:t} \rbag = \lbag z_{0:t} \rbag\Big]
    \; \leq \; & \ \alpha + \max(\alpha,\,B-\alpha)\sum_{i=0}^{t}\big(\kappa_i^{K}-1\big)\cdot \tilde{w}_i^{(\hat\beta^{\backslash i})} \nonumber \\
    & \; = \; \ \alpha + \max(\alpha,\,B-\alpha)\cdot\mathbb{E}_{\pi_{0:t}^{(\hat\beta)}}\bigg[\max\Big(\frac{\hat\beta^{\backslash Z_t}}{\hat\beta^{\backslash z_t}}, \frac{\hat\beta^{\backslash z_t}}{\hat\beta^{\backslash Z_t}}\Big)^{K}-1 \mid \lbag Z_{0:t}\rbag = \lbag z_{0:t}\rbag\bigg],
\end{align}
where the equality holds because, conditional on the bag, $\hat\beta^{\backslash Z_t}$ takes the value $\hat\beta^{\backslash z_i}$ with probability $\tilde{w}_i^{(\hat\beta^{\backslash i})}$.

Since Inequality \eqref{eq:cpc_rel_bag_conditional} holds almost surely over the draw of the bag, we may marginalize as in Step 0 and apply the $\epsilon^{(K)}$-relative replace-one stability of $\hat\beta$ (Definition \ref{def:rel_replace_one_stability}) to obtain
\begin{align}\label{eq:cpc_pf_marginal_result2}
    \mathbb{E}_{\pi_{0:t}^{(\hat\beta)}}\big[L_{t}\big]
    \; \leq \; & \ \alpha \; + \; \max(\alpha,\,B-\alpha)\,\mathbb{E}_{\pi_{0:t}^{(\hat\beta)}}\bigg[\max\Big(\frac{\hat\beta^{\backslash Z_t}}{\hat\beta^{\backslash z_t}}, \frac{\hat\beta^{\backslash z_t}}{\hat\beta^{\backslash Z_t}}\Big)^{K}-1\bigg]
    \; \leq \; \alpha \; + \; \max(\alpha,\,B-\alpha)\,\epsilon^{(K)} ,
\end{align}
which is the claimed result.  Theorem \ref{cor:cpc_free} follows by applying this with $K=1$ as a corollary, which Lemma \ref{lem:loglip} supplies for the likelihood-ratio bound.

\end{proof}

\vspace{1cm}

\subsubsection{Lemma: Uniform $1$-log-Lipschitz continuity of $\pi_t^{(\beta)}$ in $\log\beta$}

\begin{lemma}[Uniform $1$-log-Lipschitz continuity of normalized clipped densities in $\log\beta$]
\label{lem:loglip}
Let $\pi_0,\pi_t$ be probability densities and, for $\beta>0$, define the clipped
function, its normalizer, and the normalized clipped density by
\[
m_\beta(x)=\min\bigl(\beta\,\pi_0(x),\,\pi_t(x)\bigr),
\qquad
\psi(\beta)=\int m_\beta(x)\,dx,
\qquad
\pi_t^{(\beta)}(x)=\frac{m_\beta(x)}{\psi(\beta)} .
\]
Assume $\psi(\beta)>0$ for the values of $\beta$ considered. Then for all $\beta',\beta>0$
and every $x$ with $\pi_t^{(\beta')}(x),\pi_t^{(\beta)}(x)>0$,
\[
\bigl|\log \pi_t^{(\beta')}(x)-\log \pi_t^{(\beta)}(x)\bigr|
\;\le\;
\bigl|\log\beta'-\log\beta\bigr| ;
\]
that is, $\log\beta\mapsto\log \pi_t^{(\beta)}(x)$ is $1$-Lipschitz, uniformly in $x$. The
constant $1$ is sharp.
\end{lemma}
\begin{proof}
The statement is symmetric in $(\beta',\beta)$, so it suffices to prove the
two-sided multiplicative bound
\begin{equation}
\min\!\left(\frac{\beta'}{\beta},\frac{\beta}{\beta'}\right)
\;\le\;
\frac{\pi_t^{(\beta')}(x)}{\pi_t^{(\beta)}(x)}
\;\le\;
\max\!\left(\frac{\beta'}{\beta},\frac{\beta}{\beta'}\right),
\label{eq:ratio}
\end{equation}
since taking logarithms in~\eqref{eq:ratio} and using
$\log\min(\tfrac{\beta'}{\beta},\tfrac{\beta}{\beta'})=-\,|\log\beta'-\log\beta|$
together with the analogous identity for the maximum gives the claim. Moreover, by
interchanging the roles of $\beta'$ and $\beta$, the upper bound in~\eqref{eq:ratio}
follows from the lower bound applied to the reciprocal ratio; hence it suffices to
prove the lower bound in~\eqref{eq:ratio}.
\smallskip
\noindent\emph{Scaling inequality.}
For all $a,b\ge 0$, $\beta>0$, and $c>0$,
\begin{equation}
\min(c\beta a,\,b)\;\ge\;\min(c,1)\,\min(\beta a,\,b).
\label{eq:scaling}
\end{equation}
Indeed, if $c\ge 1$ then $\min(c,1)=1$ and $c\beta a\ge\beta a$ gives
$\min(c\beta a,b)\ge\min(\beta a,b)$; if $0<c\le 1$ then $c b\le b$, so
$\min(c\beta a,b)\ge\min(c\beta a,c b)=c\min(\beta a,b)
=\min(c,1)\min(\beta a,b)$.
\smallskip
\noindent\emph{Pointwise factor.}
Apply~\eqref{eq:scaling} at fixed $x$ with $a=\pi_0(x)$, $b=\pi_t(x)$, $\beta=\beta$,
and $c=\beta'/\beta$ (so $c\beta=\beta'$):
\begin{equation}
\frac{m_{\beta'}(x)}{m_{\beta}(x)}\;\ge\;\min\!\left(\frac{\beta'}{\beta},1\right).
\label{eq:point}
\end{equation}
\smallskip
\noindent\emph{Normalizer factor.}
Apply~\eqref{eq:scaling} with the roles swapped ($\beta=\beta'$,
$c=\beta/\beta'$) to obtain
$m_{\beta}(x)\ge\min(\beta/\beta',1)\,m_{\beta'}(x)$. Integrating over $x$ and
using $\psi(\beta')>0$,
\begin{equation}
\frac{\psi(\beta)}{\psi(\beta')}\;\ge\;\min\!\left(\frac{\beta}{\beta'},1\right).
\label{eq:norm}
\end{equation}
\smallskip
\noindent\emph{Combine.}
Writing the ratio as
$\dfrac{\pi_t^{(\beta')}(x)}{\pi_t^{(\beta)}(x)}
=\dfrac{m_{\beta'}(x)}{m_{\beta}(x)}\cdot\dfrac{\psi(\beta)}{\psi(\beta')}$ and
substituting~\eqref{eq:point} and~\eqref{eq:norm},
\[
\frac{\pi_t^{(\beta')}(x)}{\pi_t^{(\beta)}(x)}
\ge\min\!\left(\frac{\beta'}{\beta},1\right)\min\!\left(\frac{\beta}{\beta'},1\right)
=\min\!\left(\frac{\beta'}{\beta},\frac{\beta}{\beta'}\right),
\]
where the last equality follows by checking the cases $\beta'\ge\beta$ and
$\beta'\le\beta$. This is the lower bound in~\eqref{eq:ratio}, completing the proof.
\smallskip
\noindent\emph{Sharpness.}
At a point $x$ where clipping is inactive ($\pi_t(x)<\beta\pi_0(x)$, so
$m_\beta(x)=\pi_t(x)$ is independent of $\beta$) while the bulk of the mass lies in the
clipped region (so $\psi(\beta)\propto\beta$), one has
$\pi_t^{(\beta')}(x)/\pi_t^{(\beta)}(x)\to\beta/\beta'$, attaining the lower bound; hence the
Lipschitz constant $1$ cannot be improved.
\end{proof}

\clearpage

\clearpage

\clearpage

\section{Detailed Methodology}
\label{app:detailed_methodology}

This appendix provides implementation details for Conformal Policy Control. Section~\ref{app:conformal_calibration} describes the calibration procedure for selecting the likelihood ratio bound $\beta$ via weighted conformal prediction, including grid construction, normalization constant estimation, and conformal weighting. Section~\ref{app:sampling} describes sampling algorithms for drawing actions from the constrained policy, including rejection sampling and Independence Metropolis-Hastings.

\subsection{Conformal Calibration of $\beta$}
\label{app:conformal_calibration}

This section details the procedure for calibrating the likelihood ratio bound $\beta$ via weighted conformal risk control. We first define the likelihood ratio and describe how to construct a grid of candidate values (Section~\ref{app:lr_grid_prep}). We then show how to estimate the normalization constant $\psi(\beta)$ for each candidate via importance-weighted Monte Carlo integration (Section~\ref{app:estimating_norm_const}). Finally, we describe the conformal weighting scheme that accounts for distribution shift between calibration and deployment, and the criterion for selecting the largest $\beta$ that controls risk at level $\alpha$ (Section~\ref{app:conformal_weighting}).

\begin{remark}
Throughout this section we write $\pi_0$ for the reference policy used in the likelihood ratio, which generalizes $\pi_0$ from the main text: $\pi_0$ may be the initial safe policy $\pi_0$ or the previous constrained policy $\pi_{t-1}^{(\beta_{t-1})}$. We also assume a fixed context distribution $p(X)$ shared by the safe and optimized policies. Under this assumption, the context marginal cancels in the likelihood ratio $\pi_t(A_i \mid X_i) / \pi_0(A_i \mid X_i)$, and it suffices to work with conditional densities evaluated at the same context $X_i$.
\end{remark}

\subsubsection{Likelihood Ratios and Grid Preparation}
\label{app:lr_grid_prep}

For each action $x$, define the likelihood ratio
\begin{align}
    \mathrm{LR}(x) \coloneqq \frac{\pi_t(x)}{\pi_0(x)},
\end{align}
where $\pi_t$ is the current unconstrained policy and $\pi_0$ is either the initial policy $\pi_0$ or the previous constrained policy $\pi_{t-1}^{(\beta_{t-1})}$.

To calibrate $\beta$, we require two sets of samples:
\begin{itemize}
    \item \textbf{Calibration data} $\mathcal{D}_{\mathrm{cal}}$: samples from previous rounds with feasibility labels, used to compute conformal weights.
    \item \textbf{Proposal data} $\mathcal{D}_{\mathrm{prop}}$: fresh samples from $\pi_t$, used to estimate the normalization constant $\psi(\beta)$.
\end{itemize}

We construct a grid $G$ of candidate $\beta$ values from the observed likelihood ratios across both datasets:
\begin{align}
    G = \textsc{PrepareGrid}\left(\{\mathrm{LR}(x) : x \in \mathcal{D}_{\mathrm{cal}} \cup \mathcal{D}_{\mathrm{prop}}\}\right).
\end{align}
The grid is sorted in increasing order, and we search from the most conservative (smallest $\beta$) to the most permissive, stopping at the first $\beta$ where the weighted infeasibility rate exceeds $\alpha$. See Figure~\ref{fig:beta_calibration} for an illustration.

\begin{figure}[!ht]
    \centering
    \includegraphics[width=0.5\linewidth]{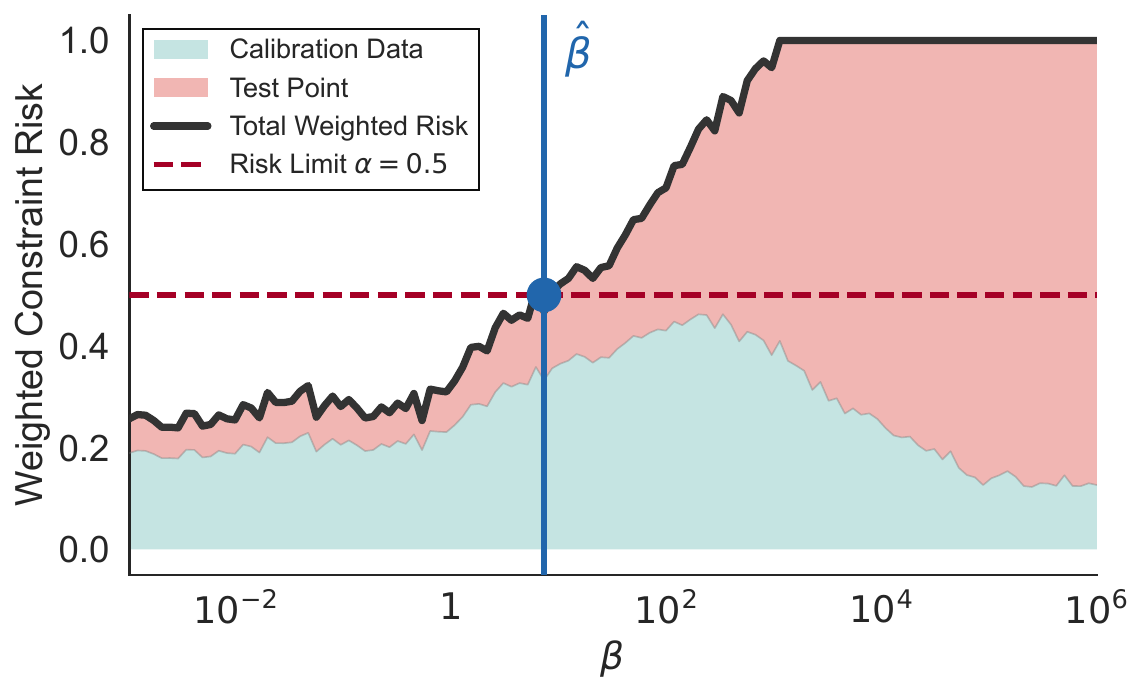}
    \caption{An illustration of fitting $\hat \beta$ to calibration data drawn from $\pi_0$. We draw a test point from the optimized policy and conservatively assume the loss attains the upper bound $B$. As $\beta$ increases, the contribution of the test point to the weighted average constraint violation increases. We search $\beta$ in increasing order, recomputing the conformal weights for each setting of $\beta$ and halting at the first point $\hat \beta$ where the average constraint violation exceeds $\alpha$.}
    \label{fig:beta_calibration}
\end{figure}

\subsubsection{Estimating the Normalization Constant $\psi(\beta)$}
\label{app:estimating_norm_const}

Recall that the constrained policy is defined as
\begin{align}
    \pi^{(\beta)}(x) = \frac{\min\{\pi_t(x), \beta \cdot \pi_0(x)\}}{\psi(\beta)},
\end{align}
where $\psi(\beta) = \int \min\{\pi_t(x), \beta \cdot \pi_0(x)\} \, dx$ is the normalization constant. We estimate $\psi(\beta)$ via importance-weighted Monte Carlo integration using the proposal data $\mathcal{D}_{\mathrm{prop}}$.

When sampling from the optimistic policy $\pi_t$:
\begin{align}
    \psi(\beta)
    &= \mathbb{E}_{\pi_t}\left[\frac{\min\{\pi_t(x), \beta \cdot \pi_0(x)\}}{\pi_t(x)}\right]
    = \mathbb{E}_{\pi_t}\left[\min\left\{\frac{\beta}{\mathrm{LR}(x)}, 1\right\}\right] \\
    &\approx \frac{1}{|\mathcal{D}_{\mathrm{prop}}|} \sum_{x \in \mathcal{D}_{\mathrm{prop}}} \min\left\{\frac{\beta}{\mathrm{LR}(x)}, 1\right\}.
\end{align}

When sampling from the safe policy $\pi_0$:
\begin{align}
    \psi(\beta)
    &= \mathbb{E}_{\pi_0}\left[\frac{\min\{\pi_t(x), \beta \cdot \pi_0(x)\}}{\pi_0(x)}\right]
    = \mathbb{E}_{\pi_0}\left[\min\left\{\mathrm{LR}(x), \beta\right\}\right] \\
    &\approx \frac{1}{|\mathcal{D}_{\mathrm{prop}}|} \sum_{x \in \mathcal{D}_{\mathrm{prop}}} \min\left\{\mathrm{LR}(x), \beta\right\}.
\end{align}

In practice, all computations are performed in log space for numerical stability. Writing $r_i = \log \mathrm{LR}(x_i)$ for each sample, the log normalization constant is
\begin{align}
    \log \psi(\beta) = \mathrm{logsumexp}\left(\{\min(\log \beta - r_i, 0)\}_{i=1}^{n}\right) - \log n
\end{align}
for the optimistic proposal, and similarly with $\min(r_i, \log \beta)$ for the safe proposal.

\subsubsection{Conformal Mixture Weight Approximation}
\label{app:conformal_weighting}

Here we will derive the mixture weight approximation to the exact conformal permutation weights. For ease of exposition, we initially introduce weights for a given known test point, $z_t$, even though CPC occurs prior to sampling, before the test point is revealed. We will then describe the conservative adjustment for unknown test point. We also focus on a given fixed $\beta$ and fixed policy $\pi_t^{(\beta)}$, as the focus of this section is to form a weighted risk estimate of a given candidate policy, rather than the CPC procedure for searching over policies.

Recall that we define the \textit{normalized }oracle weights (from \citet{tibshirani2019conformal} and \citet{prinster2024conformal}) as
\begin{align*}
        \tilde{w}_i^{(\beta)} := \tilde{w}_i\big(\pi_{0:t}^{(\beta)}, \lbag z_{0:t}\rbag\big) = \mathbb{P}_{\pi_{0:t}^{(\beta)}}(Z_t=z_i \mid \lbag z_{0:t}\rbag) =\frac{\sum_{\sigma:\sigma(t)=i}\pi_{0:t}^{(\beta)}(z_{\sigma(0)}, ..., z_{\sigma(t)})}{\sum_{\sigma}\pi_{0:t}^{(\beta)}(z_{\sigma(0)}, ..., z_{\sigma(t)})},
\end{align*}
    for all $i\in \{0, ..., t\}$, where $\sigma$ is a permutation of $\{0, ..., t\}$. 
    
    Note that $\pi_{0:t}^{(\beta)}(z_0, ..., z_t)=\pi_{t\mid 0:t-1}^{(\beta)}(z_t\mid z_0, ..., z_{t-1})\cdot\pi_{0:t-1}(z_0, ..., z_{t-1})$, and recall we focus on the split conformal setting where we only allow $Z_t$ to depend on $Z_{0:t-1}$ symmetrically via $\beta$, so  $\pi_{t\mid 0:t-1}^{(\beta)}(z_t\mid z_0, ..., z_{t-1}) = \pi_{t\mid 0:t-1}^{(\beta)}(z_t\mid \lbag z_{0:t-1}\rbag)$. Thus, we can make the substitution $\pi_{0:t}^{(\beta)}(z_{\sigma(0)}, ..., z_{\sigma(t)})=\pi_{t\mid 0:t-1}^{(\beta)}(z_{\sigma(t)}\mid \lbag z_{0:t\backslash \sigma(t)}\rbag)\cdot\pi_{0:t-1}(z_{\sigma(0)}, ..., z_{\sigma(t-1)})$, to get
    \begin{align*}
        \tilde{w}_i^{(\beta)} & =  \frac{\sum_{\sigma:\sigma(t)=i}\pi_{t\mid 0:t-1}^{(\beta)}(z_{\sigma(t)}\mid \lbag z_{0:t\backslash \sigma(t)}\rbag)\cdot\pi_{0:t-1}(z_{\sigma(0)}, ..., z_{\sigma(t-1)})}{\sum_{\sigma}\pi_{t\mid 0:t-1}^{(\beta)}(z_{\sigma(t)}\mid \lbag z_{0:t\backslash \sigma(t)}\rbag)\cdot\pi_{0:t-1}(z_{\sigma(0)}, ..., z_{\sigma(t-1)})}.
\end{align*}
Re-writing the denominator sum as $\sum_i\sum_{\sigma:\sigma(t)=i}$ and substituting $\sigma(t)=i$ within the summations,
\begin{align*}
    \tilde{w}_i^{(\beta)} & =  \frac{\sum_{\sigma:\sigma(t)=i}\pi_{t\mid 0:t-1}^{(\beta)}(z_{i}\mid \lbag z_{0:t\backslash i}\rbag)\cdot\pi_{0:t-1}(z_{\sigma(0)}, ..., z_{\sigma(t-1)})}{\sum_i\sum_{\sigma:\sigma(t)=i}\pi_{t\mid 0:t-1}^{(\beta)}(z_{i}\mid \lbag z_{0:t\backslash i}\rbag)\cdot\pi_{0:t-1}(z_{\sigma(0)}, ..., z_{\sigma(t-1)})} \\
    & =  \frac{\pi_{t\mid 0:t-1}^{(\beta)}(z_{i}\mid \lbag z_{0:t\backslash i}\rbag)\cdot\sum_{\sigma:\sigma(t)=i}\pi_{0:t-1}(z_{\sigma(0)}, ..., z_{\sigma(t-1)})}{\sum_i\pi_{t\mid 0:t-1}^{(\beta)}(z_{i}\mid \lbag z_{0:t\backslash i}\rbag)\cdot\sum_{\sigma:\sigma(t)=i}\pi_{0:t-1}(z_{\sigma(0)}, ..., z_{\sigma(t-1)})} \\
    & =  \frac{\pi_t^{(\beta)}(z_i)\cdot\sum_{\sigma:\sigma(t)=i}\pi_{0:t-1}(z_{\sigma(0)}, ..., z_{\sigma(t-1)})}{\sum_i\pi_t^{(\beta)}(z_i)\cdot\sum_{\sigma:\sigma(t)=i}\pi_{0:t-1}(z_{\sigma(0)}, ..., z_{\sigma(t-1)})}.
\end{align*}
where the second line follows since $\pi_{t\mid 0:t-1}^{(\beta)}(z_{i}\mid \lbag z_{0:t\backslash i}\rbag)$ does not depend on $\sigma$, and the third line by our abbreviated notation $\pi_t^{(\beta)}(z_i):=\pi_{t\mid 0:t-1}^{(\beta)}(z_{i}\mid \lbag z_{0:t\backslash i}\rbag)$.

We can now see where the joint distribution of past policies, $\pi_{0:t-1}$, is isolated as the remaining term inside the sum over permutations $\sigma$, and thus the remaining term that is keeping $\tilde{w}_i^{(\beta)}$ intractable. This motivates us to approximate $\pi_{0:t-1}$ by the mixture distribution over past policies, 
\begin{align*}
    \pi_{0:t-1}^{\text{mix}}(z):=\frac{1}{t}\pi_0(z) + ... + \frac{1}{t}\pi_{t-1}(z),
\end{align*}
or more generally, the mixture weights for each policy should be proportional to the number of calibration points sampled in that round.

This approximation assumes that each past datapoint was drawn exchangeably (in particular, i.i.d.) from this mixture of $\pi_0, ..., \pi_{t-1}$, which weights every permutation equally, allowing for us to derive a computationally efficient estimate. Plugging in $\pi_{0:t-1}^{\text{mix}}$ for $\pi_{0:t-1}$ and multiplying by $1=\frac{\pi_{0:t-1}^{\text{mix}}(z_i)}{\pi_{0:t-1}^{\text{mix}}(z_i)}$ inside the summations,

\begin{align*}
    \tilde{w}_i^{(\beta)} 
    \approx &\frac{\pi_t^{(\beta)}(z_i)\cdot\sum_{\sigma:\sigma(t)=i}\pi_{0:t-1}^{\text{mix}}(z_{\sigma(0)}, ..., z_{\sigma(t-1)})}{\sum_i\pi_t^{(\beta)}(z_i)\cdot\sum_{\sigma:\sigma(t)=i}\pi_{0:t-1}^{\text{mix}}(z_{\sigma(0)}, ..., z_{\sigma(t-1)})} \\
    & = \frac{\pi_t^{(\beta)}(z_i)\cdot\sum_{\sigma:\sigma(t)=i}\pi_{0:t-1}^{\text{mix}}(z_{\sigma(0)}, ..., z_{\sigma(t-1)})\cdot\frac{\pi_{0:t-1}^{\text{mix}}(z_i)}{\pi_{0:t-1}^{\text{mix}}(z_i)}}{\sum_i\pi_t^{(\beta)}(z_i)\cdot\sum_{\sigma:\sigma(t)=i}\pi_{0:t-1}^{\text{mix}}(z_{\sigma(0)}, ..., z_{\sigma(t-1)})\cdot\frac{\pi_{0:t-1}^{\text{mix}}(z_i)}{\pi_{0:t-1}^{\text{mix}}(z_i)}},
\end{align*}
then pulling $\frac{1}{\pi_{0:t-1}^{\text{mix}}(z_i)}$ out of the summation over permutations, noting that by independence $\pi_{0:t-1}^{\text{mix}}(z_{\sigma(0)}, ..., z_{\sigma(t-1)})=\pi_{0:t-1}^{\text{mix}}(z_{\sigma(0)})\cdots\pi_{0:t-1}^{\text{mix}}(z_{\sigma(t-1)})$, and collecting product terms, our approximation becomes
\begin{align*}
    \tilde{w}_i^{(\beta)} 
    \approx &\frac{\frac{\pi_t^{(\beta)}(z_i)}{\pi_{0:t-1}^{\text{mix}}(z_i)}\cdot\sum_{\sigma:\sigma(t)=i}\prod_{j=0}^t\pi_{0:t-1}^{\text{mix}}(z_j)}{\sum_i\frac{\pi_t^{(\beta)}(z_i)}{\pi_{0:t-1}^{\text{mix}}(z_i)}\cdot\sum_{\sigma:\sigma(t)=i}\prod_{j=0}^t\pi_{0:t-1}^{\text{mix}}(z_j)} \\
    & = \frac{\frac{\pi_t^{(\beta)}(z_i)}{\pi_{0:t-1}^{\text{mix}}(z_i)}\cdot t!}{\sum_i\frac{\pi_t^{(\beta)}(z_i)}{\pi_{0:t-1}^{\text{mix}}(z_i)}\cdot t!} \\
    & = \frac{\pi_t^{(\beta)}(z_i)/\pi_{0:t-1}^{\text{mix}}(z_i)}{\sum_i\pi_t^{(\beta)}(z_i)/\pi_{0:t-1}^{\text{mix}}(z_i)},
\end{align*}
and let our unnormalized weight estimate be $w_i^{(\beta)}=\pi_t^{(\beta)}(z_i)/\pi_{0:t-1}^{\text{mix}}(z_i)$.

However, because for CPC the test point is unknown at calibration time, we conservatively estimate its unnormalized weight by taking a maximum over a set of proposal samples:
\begin{align}
    w_{\mathrm{max}}^{(\beta)} = \max_{z \in \mathcal{D}_{\mathrm{prop}}} \frac{\pi^{(\beta)}(z)}{\pi_{0:t-1}^{\mathrm{mix}}(z)},
\end{align}

and then we normalize all weights to sum to one:
\begin{align}
    \tilde{w}_i^{(\beta)} = \frac{w_i^{(\beta)}}{\sum_{j} w_j^{(\beta)} + w_{\mathrm{max}}^{(\beta)}}, \quad
    \tilde{w}_{\mathrm{max}}^{(\beta)} = \frac{w_{\mathrm{max}}^{(\beta)}}{\sum_{j} w_j^{(\beta)} + w_{\mathrm{max}}^{(\beta)}}.
\end{align}

Approximating the unnormalized $w_{\text{max}}^{(\beta)}$ prior to normalizing is valid for approximating the normalized $\tilde{w}_{\text{max}}^{(\beta)}$ due to the transformation being monotone for the mixture approximation.

Let $B$ denote the upper bound on the loss $\ell(x)$. The weighted risk estimate is
\begin{align}
    \hat{R}(\beta) = \sum_{i \in \mathcal{D}_{\mathrm{cal}}} \tilde{w}_i^{(\beta)} \, \ell(x_i) + \tilde{w}_{\mathrm{max}}^{(\beta)}\cdot B,
\end{align}
where the test point is conservatively assigned the worst-case loss $B$. We select $\hat{\beta}$ as the largest $\beta \in G$ such that $\hat{R}(\beta') \leq \alpha$ for all $\beta'\leq \beta$.

\subsection{Sampling from the Constrained/Risk-Controlling Policy}
\label{app:sampling}

In this section, we describe three approaches to sample from $\pi^{(\beta)}$: rejection sampling with either $\pi_0$ or $\pi_t$ as the proposal (Section~\ref{app:rejection_simple}), rejection sampling with a mixture proposal (Section~\ref{app:rejection_mixture}), and Independence Metropolis-Hastings (Section~\ref{app:imh}).

\subsubsection{Rejection Sampling with Simple Proposals}
\label{app:rejection_simple}

We first review the classic accept-reject procedure. Let $\tilde p(x)$ be an unnormalized target density and $q(x)$ a normalized proposal such that
$$\frac{\tilde p(x)}{q(x)} \leq M \quad \forall x,$$
for some envelope constant $M < \infty$.\footnote{Alternatively, we can view $M q(x)$ as the unnormalized envelope proposal with normalizing constant $M$.} To sample from $p(x) = \tilde p(x) / \int \tilde p(x) dx$, we proceed as follows: draw $X \sim q$, draw $U \sim {\rm Unif}(0, 1)$, and accept $X$ if $U \leq \frac{\tilde p(X)}{M q(X)}$. The optimal (smallest) envelope constant is
$$M^\star = \sup_x \ \frac{\tilde p(x)}{q(x)}.$$

The marginal acceptance probability is
\begin{align}
\mathbb{P} \left[U \leq \frac{\tilde p(X)}{M q(X)} \right] 
&= \mathbb{E} \left[ \mathbbm{1} \left\{U \leq \frac{\tilde p(X)}{M q(X)} \right\} \right] \nonumber \\
&= \mathbb{E} \left[ \mathbb{E}\left[ \mathbbm{1} \left\{U \leq \frac{\tilde p(X)}{M q(X)} \right\} \Big| X \right] \right] \nonumber \\
&= \mathbb{E} \left[ \mathbb{P}\left[ U \leq \frac{\tilde p(X)}{M q(X)}  \Big| X \right] \right] \nonumber \\
&= \mathbb{E}_{X\sim q}\left[\frac{\tilde p(X)}{M q(X)} \right] \nonumber \\
&= \int_{x: q(x)>0} q(x) \frac{\tilde p(x)}{M q(x)} dx \nonumber \\
&= \frac{1}{M}\int_{x: q(x)>0} \tilde p(x) dx
\label{eq:app:accept_prob_rejection_sampling}
\end{align}
so tighter $M$ leads to higher acceptance rates.

\paragraph{Safe policy proposal.} Take $\tilde p(x) = \tilde \pi^{(\beta)}(x) = \min\{ \pi_t(x), \beta \pi_{\rm safe}(x) \}$ and $q(x) = \pi_{\rm safe}(x)$. Then for all $x$,
$$
\frac{\tilde p(x)}{q(x)} = \min \left\{ \frac{\pi_t(x)}{\pi_{\rm safe}(x)}, \beta \right\} \leq \beta,
$$
so we choose $M = \beta$. The pointwise acceptance probability then is
\begin{equation}
\frac{\tilde p(x)}{\beta \pi_{\rm safe}(x)} = \min\left\{\frac{\pi_t(x)}{\beta\pi_{\rm safe}(x)}, 1 \right\},
\label{eq:app:accept_prob_safe_proposal}
\end{equation}
or $\frac{\pi_t(x)}{\beta\pi_{\rm safe}(x)}$ in the unclipped region and $1$ in the clipped region.

\paragraph{Optimized policy proposal.} If we instead take $q(x) = \pi_t(x)$, then for all $x$,
$$
\frac{\tilde p(x)}{q(x)}
=
\frac{\min\{\pi_t(x),\,\beta\pi_{\rm safe}(x)\}}{\pi_t(x)}
=
\min \left\{1, \frac{\beta\pi_{\rm safe}(x)}{\pi_t(x)} \right\}
 \leq 1,
$$
so we may take $M=1$.
The pointwise acceptance probability is
\begin{align}
\frac{\tilde p(x)}{\pi_t(x)}
=
\min\left\{1, \frac{\beta\pi_{\rm safe}(x)}{\pi_t(x)} \right\},
\label{eq:app:accept_prob_opt_proposal}
\end{align}
which is $1$ in the unclipped region and $\frac{\beta\pi_{\rm safe}(x)}{\pi_t(x)}$ in the clipped region.

Intuitively, the relative efficiency of the two proposals depends on the constraint parameter $\beta$ controlling how much the optimized policy is truncated toward the safe policy. When $\beta$ is small, the constrained target resembles a rescaled version of $\pi_{\rm safe}$ over most of its support. Proposing from $\pi_{\rm safe}$ tends to be efficient in this regime, because most points satisfy $\pi_t(x)/\pi_{\rm safe}(x) \ll \beta$ and the acceptance probability in \autoref{eq:app:accept_prob_safe_proposal}, while small, varies smoothly over $x$. In contrast, proposing from $\pi_t$ can be inefficient with small $\beta$, since $\pi_t$ may place substantial mass in regions that are heavily downweighted by the constraint, leading to frequent rejections.

When $\beta$ is large, the constrained target resembles $\pi_t$. As long as $\pi_t(x) / \pi_{\rm safe}(x) \lesssim \beta$ for most of the support of $\pi_t$, then using $\pi_t$ as the proposal tends to be more efficient in this regime, as the acceptance probability in \autoref{eq:app:accept_prob_opt_proposal} is frequently close to $1$. If, however, $\pi_t$ concentrates mass in regions where $\pi_t(x) / \pi_{\rm safe}(x)$ is extremely large, this proposal can still perform poorly.

At intermediate values of $\beta$, neither proposal dominates. The constrained target may allocate comparable mass from $\pi_{\rm safe}$ and $\pi_t$. Both proposals can suffer from mismatches with different parts of the target, motivating the use of mixture proposals that adaptively trade off between $\pi_{\rm safe}$ and $\pi_t$.

\subsubsection{Rejection Sampling with Mixture Proposal}
\label{app:rejection_mixture}

Let the proposal be the mixture
$$
q_w(x) = w \pi_{\rm safe}(x) + (1-w) \pi_t(x),
\quad w \in [0, 1].
$$
The tight envelope constant is
$$
M^\star(w) = \sup_x \frac{\tilde p(x)}{q_w(x)}
=
\sup_x \frac{\min\{\pi_t(x), \beta \pi_{\rm safe}(x)\}}{w \pi_{\rm safe}(x) + (1-w) \pi_t(x)}.
$$
Given any $M\ge M^\star(w)$, the pointwise acceptance probability is
\begin{equation}
\alpha_w(x) \coloneqq \frac{\tilde p(x)}{M q_w(x)}
= \frac{\min\{\pi_t(x), \beta \pi_{\rm safe}(x)\}}{M\left(w \pi_{\rm safe}(x)+(1-w) \pi_t(x) \right)}.
\label{eq:app:accept_prob_mixture}
\end{equation}

Choosing smaller $M$ (i.e., closer to $M^\star$) yields higher acceptance rates, as derived in \autoref{eq:app:accept_prob_rejection_sampling}. In practice, we may choose $M$ by estimating $M^\star(w)$ from a set of points $\mathcal{S}$ covering the support of $q_w$, where $\mathcal{S}$ is obtained by drawing samples from $q_w$ or some appropriate fraction of samples from each component to cover the tails. Taking the empirical maximum 
$$\max_{x \in \mathcal{S}} \ \frac{\tilde p(x) }{ q_w(x) }$$ 
then yields an estimate of $M^\star (w)$.

We now describe a few simple methods for choosing $w$.
\paragraph{Joint grid search with $M$.} 
We can iterate over a grid of $w$ and choose $w, M$ jointly. For each candidate value of $w$, estimate $M^\star(w)$ using the procedure described above and choose $w$ that minimizes this estimate. Note that the support points $\mathcal{S}$ used to estimate $M^\star(\cdot)$ can be reused for different values of $w$.

\paragraph{Estimated overlap coefficients.}
The \textit{overlap coefficient} between two density functions $p_1, p_2$ is defined as
$$
\mathrm{OVL}(p_1, p_2):=\int \min \left\{ p_1(x), p_2(x) \right\} \ dx = \int \min \left\{ \frac{p_1(x)}{p_2(x)}, 1 \right\} p_2(x) \ dx \quad \in [0, 1].
$$
It is related to the total variation (TV) distance as $\mathrm{OVL}(p_1, p_2) = 1 - {\rm TV}(p_1, p_2)$.

Letting $p_1 = p = \pi^{(\beta)}$, our normalized target, 
we can estimate $\mathrm{OVL}(\pi^{(\beta)}, \pi_{\rm safe})$ using samples $x_1, \dotsc, x_n$ drawn from $\pi_{\rm safe}$:
\begin{align}
\widehat{\mathrm{OVL}}(\pi^{(\beta)}, \pi_{\rm safe})
=
\frac{1}{n}\sum_{i=1}^n \min \left\{\frac{\pi^{(\beta)}(x_i)}{\pi_{\rm safe}(x_i)}, 1\right\}
=
\frac{1}{n}\sum_{i=1}^n \min \left\{\frac{\pi_t(x_i)}{\psi(\beta) \pi_{\rm safe}(x_i)}, \frac{\beta}{\psi(\beta)}, 1\right\},
\label{eq:ovl_hat_safe}
\end{align}
and similarly $\widehat{\mathrm{OVL}}(\pi^{(\beta)}, \pi_t)$ using samples drawn from $\pi_t$.
The normalizing constant $\psi(\beta)$ is estimated with importance-weighted MC integration as described in Section \ref{app:estimating_norm_const}. 
A simple heuristic is then to set
\begin{align}
w
=
\frac{\widehat{\mathrm{OVL}}(\pi^{(\beta)}, \pi_{\rm safe})}{\widehat{\mathrm{OVL}}(\pi^{(\beta)}, \pi_{\rm safe})+\widehat{\mathrm{OVL}}(\pi^{(\beta)}, \pi_t)}.
\label{eq:w_overlap_heuristic}
\end{align}
This chooses the mixture weights proportional to each component's estimated overlap with the target, biasing the proposal
toward whichever component more closely matches the constrained policy, while retaining support from both.

\paragraph{Adaptive update.}  
Rather than fixing $w$, we can adapt it online in the direction of a desirable sampling diagnostic. For example, we might define an update based on the observed acceptance rates of proposals drawn from each component, so that $w$ shifts toward the component that yields higher acceptance.

\subsubsection{Independence Metropolis-Hastings Sampling}
\label{app:imh}

While rejection sampling yields exact samples from the target distribution, it requires finding a finite global bound $M$ satisfying ${\tilde p}(x)/q(x) \leq M$ for all $x$. In many settings, especially in high dimensions, this bound is too loose to be practical or may not exist at all. The Metropolis-Hastings algorithm \citep{metropolis1953equation, hastings1970monte} relaxes this requirement.
In particular, here we discuss a simple variant called the independence Metropolis Hastings (IMH) algorithm, where the proposal does not depend on the current state \citep{tierney1994markov}.

Given an unnormalized target $\tilde p(x)$ and a normalized proposal $q(x)$, we initialize the Markov chain at $x_0$ and the IMH algorithm evolves it as follows. From the current state $x_t$, draw a candidate state $x' \sim q$, compute the acceptance probability
$$
a(x', x_t) = \min \left\{1, \frac{r(x')}{r(x_t)} \right\}, \quad r(x) \coloneqq \frac{\tilde p(x)}{q(x)},
$$
draw $U \sim {\rm Unif}(0, 1)$, and accept $x'$ (i.e., set $x_{t+1} = x'$) if $U \leq a(x', x_t)$ otherwise copy the old state $x_{t+1}=x_t$. It can be shown that the invariant distribution of the resulting Markov chain $x_0, x_1, x_2 \dotsc$ is the target $p(x)$. Intuitively, IMH compares each proposal only to the current state, whereas rejection sampling compares each proposal to a global worst case.

\clearpage

\section{Medical QA Implementation Details and Additional Results}
\label{app:medical_qa}

This appendix provides implementation details for our medical question-answering experiments, connecting the general conformal risk control framework to the specific problem of controlling LLM factuality. We describe the problem setup and notation (Section~\ref{subsec:medqa_setup}), prior approaches to conformal factuality (Section~\ref{subsec:medqa_prior_approaches}), loss functions for factuality control including the non-monotonic FDR loss (Section~\ref{subsec:medqa_loss_functions}), experimental setup (Section~\ref{subsec:medqa_experimental_setup}), and hyperparameters (Section~\ref{subsec:medqa_hyperparams}).

\subsection{Problem Setup and Notation}
\label{subsec:medqa_setup}

Following \citet{mohri2024language} and \citet{cherian2024large}, we formalize the natural language factuality setting as follows. Each data sample $Z_i$ is a prompt-response-claim-annotation tuple:
\begin{align}
    Z_i := (X_i, X_i', \boldsymbol{C}_i, \boldsymbol{Y}_i),
\end{align}
where $X_i$ is the prompt (question), $X_i'$ is the full LLM response, $\boldsymbol{C}_i = (C_{i1}, \ldots, C_{ik_i})$ is the vector of atomic subclaims extracted from $X_i'$, and $\boldsymbol{Y}_i = (Y_{i1}, \ldots, Y_{ik_i})$ contains binary annotations indicating whether each subclaim is factually correct ($Y_{ij} = 1$) or incorrect ($Y_{ij} = 0$).

For each subclaim $C_{ij}$, let $p(C_{ij} \mid X_i) \in [0, 1]$ denote a prompt-conditional confidence score estimating the likelihood that the subclaim is factual. Higher scores indicate greater confidence in factuality.

\subsection{Prior Approaches to Conformal Factuality}
\label{subsec:medqa_prior_approaches}

\paragraph{Mohri \& Hashimoto (2024).} \citet{mohri2024language} achieve conformal factuality guarantees by filtering subclaims based on a threshold $\tau$ determined by split conformal prediction. A subclaim $C_{ij}$ is included in the filtered response if $p(C_{ij} \mid X_i) \geq \tau$. The corresponding non-conformity score for each response is the smallest threshold such that all included subclaims are factual:
\begin{align}
    s(Z_i) = \inf\big\{\tau : Y_{ij} = 1 \ \text{for all} \ j \ \text{with} \ p(C_{ij} \mid X_i) \geq \tau \big\}.
\end{align}
The threshold $\hat{\tau}$ is then selected as the $(1-\alpha)$-quantile of the calibration scores:
\begin{align}
    \hat{\tau} = \inf \Big\{\tau : \frac{1}{n}\sum_{i=1}^n\mathbbm{1}\{s(Z_i) \leq \tau\} \geq \frac{\lceil(n+1)(1-\alpha)\rceil}{n}\Big\}.
\end{align}

\paragraph{Cherian et al.\ (2024).} \citet{cherian2024large} extend this approach using conditional calibration and multicalibration to provide guarantees that hold across different subgroups of prompts, addressing the limitation that marginal guarantees may not hold uniformly across topics or question types.

\subsection{Loss Functions for Factuality Control}
\label{subsec:medqa_loss_functions}

\paragraph{Binary indicator loss (prior work).} The loss function implicitly used by \citet{mohri2024language} and \citet{cherian2024large} is the binary indicator for whether any included subclaim is incorrect:
\begin{align}
    L_i^{\text{binary}}(\tau) = \mathbbm{1}\Big\{\exists \ j : p(C_{ij} \mid X_i) \geq \tau \ \text{and} \ Y_{ij} = 0\Big\}.
\end{align}
Controlling this loss at level $\alpha$ guarantees $\mathbb{P}(\text{any included claim is false}) \leq \alpha$. This loss is monotonically non-increasing in $\tau$ by construction: as the threshold increases, fewer claims are included, so the probability of including a false claim can only decrease.

\paragraph{False discovery rate loss (this work).} We instead use a more granular loss: the false discovery rate (FDR), \textit{i.e.}, the fraction of included subclaims that are false:
\begin{align}
    L_i(\tau) = \underbrace{\frac{\sum_{j=1}^{k_i} (1 - Y_{ij}) \cdot \mathbbm{1}\{p(C_{ij} \mid X_i) \geq \tau\}}{\sum_{j=1}^{k_i} \mathbbm{1}\{p(C_{ij} \mid X_i) \geq \tau\}}}_{\text{FDR}_i(\tau) = 1 - \text{Precision}_i(\tau)},
    \label{eq:fdr_loss}
\end{align}
with $L_i(\tau) = 0$ when no claims are included. This is precisely $1 - \text{Precision}_i(\tau)$, where precision is the fraction of included claims that are true. Controlling this loss at level $\alpha$ ensures:
\begin{align}
    \mathbb{E}[\text{FDR}_i(\tau)] \leq \alpha \quad \Longleftrightarrow \quad \mathbb{E}[\text{Precision}_i(\tau)] \geq 1 - \alpha.
\end{align}

\paragraph{Why FDR control requires generalized CRC.} Unlike the binary loss, the FDR loss in \eqref{eq:fdr_loss} is not monotonic in $\tau$. As $\tau$ increases, both the numerator (false claims included) and denominator (total claims included) decrease, but not necessarily at the same rate. This non-monotonicity means standard CRC \citep{angelopoulos2022conformal} does not apply directly, motivating our generalized CRC approach.

For methods requiring monotonicity, one can use the monotonized loss:
\begin{align}
    \tilde{L}_i(\tau) = \sup_{\tau' \geq \tau} L_i(\tau'),
\end{align}
which upper-bounds the FDR and is monotonically non-increasing by construction. However, this monotonization can be overly conservative.

\paragraph{Reward metric: true claim recall.} While we control the FDR (our loss), we measure the recall of true claims as our reward metric:
\begin{align}
    \text{Recall}_i(\tau) = \frac{\sum_{j: Y_{ij} = 1} \mathbbm{1}\{p(C_{ij} \mid X_i) \geq \tau\}}{\sum_{j=1}^{k_i} Y_{ij}},
    \label{eq:recall_metric}
\end{align}
\textit{i.e.}, the fraction of true subclaims that are retained after filtering. For a given FDR target $\alpha$, methods that achieve higher recall preserve more useful information while satisfying the same safety constraint.

\subsection{Experimental Setup}
\label{subsec:medqa_experimental_setup}

Our experimental setup closely follows \citet{cherian2024large}, with the key modification of controlling the FDR loss \eqref{eq:fdr_loss} and measuring recall \eqref{eq:recall_metric} as a reward metric.

\subsubsection{Dataset}

We use the \textbf{MedLFQA} (Medical Long-Form Question Answering) dataset from \citet{jeong2024olaph}, which aggregates four established medical QA benchmarks:
\begin{itemize}
    \item \textbf{HealthSearchQA} ($n = 3047$): Consumer health questions
    \item \textbf{K-QA} ($n = 1077$): Knowledge-intensive medical questions, with a subset having gold-standard physician responses (K-QA Golden, $n = 201$) and the remainder having LLM-generated responses (K-QA Silver)
    \item \textbf{LiveQA} ($n = 100$): Real consumer health questions from the National Library of Medicine
    \item \textbf{MedicationQA} ($n = 627$): Questions about medications and drug interactions
\end{itemize}
Each prompt is accompanied by either an LLM or human-generated reference response, which serves as ground truth for evaluating factual accuracy.

\subsubsection{Data Generation Pipeline}

We use the annotated dataset from \citet{cherian2024large}, which was constructed through a three-stage pipeline:
\begin{enumerate}
    \item \textbf{Response generation}: GPT-3.5-Turbo was queried with each medical question to obtain a candidate response $X_i'$.
    \item \textbf{Subclaim extraction}: GPT-4o parsed each response into atomic subclaims $\boldsymbol{C}_i = (C_{i1}, \ldots, C_{ik_i})$.
    \item \textbf{Factuality annotation}: GPT-3.5-Turbo evaluated whether each subclaim is substantiated by the reference response, yielding binary annotations $Y_{ij} \in \{0, 1\}$.
\end{enumerate}
For subclaim scoring, we use \textbf{self-evaluation scores} $p(C_{ij} \mid X_i)$, where the LLM is prompted to assess its own confidence in each claim.

\subsubsection{Methods Compared}

We compare three conformal risk control methods:
\begin{itemize}
    \item \textbf{gCRC (proposed)}: Generalized conformal risk control as described in Section~\ref{sec:theory}, which searches thresholds from safest to most aggressive and stops when the augmented empirical risk exceeds $\alpha$.

    \item \textbf{Monotonized-losses CRC} \citep{angelopoulos2022conformal, mohri2024language}: Applies standard CRC to the monotonized loss $\tilde{L}_i(\tau) = \sup_{\tau' \geq \tau} L_i(\tau')$, which upper-bounds the FDR.

    \item \textbf{Learn Then Test (LTT)} \citep{angelopoulos2021learn}: A multiple testing approach using Hoeffding-Bentkus $p$-values to identify valid thresholds with family-wise error rate control. Unlike CRC methods, LTT does not add a conservative test-point loss term.
\end{itemize}

\subsubsection{Experimental Protocol}

\begin{itemize}
    \item \textbf{Calibration/test split}: 70\%/30\% random split of the dataset
    \item \textbf{Risk levels}: $\alpha \in \{0.005, 0.01, 0.015, \ldots, 0.10\}$ (20 values)
    \item \textbf{Threshold grid}: 200 candidate thresholds sampled from the empirical distribution of subclaim scores
    \item \textbf{Random seeds}: Results averaged over 25 independent train/test splits
    \item \textbf{Score jittering}: Uniform noise ($\sim 10^{-8}$) added to break ties
\end{itemize}

\subsection{Hyperparameters}
\label{subsec:medqa_hyperparams}

Table~\ref{tab:medqa_hyperparameters} summarizes the experimental configuration.

\begin{table}[h]
\centering
\caption{Hyperparameters for medical QA experiments.}
\label{tab:medqa_hyperparameters}
\begin{tabular}{lll}
\toprule
\textbf{Parameter} & \textbf{Symbol} & \textbf{Value} \\
\midrule
Calibration fraction & -- & 0.7 \\
Number of trials & -- & 25 \\
Threshold grid size & $|\Lambda|$ & 200 \\
Risk levels tested & $\alpha$ & $\{0.005, 0.01, \ldots, 0.10\}$ \\
Score jitter magnitude & -- & $10^{-8}$ \\
Response LLM & -- & GPT-3.5-Turbo \\
Subclaim extraction LLM & -- & GPT-4o \\
Annotation LLM & -- & GPT-3.5-Turbo \\
Subclaim scoring method & -- & Self-evaluation \\
\bottomrule
\end{tabular}
\end{table}

\subsection{Additional Medical QA Results}
\label{app:med_QA_add_results}

\begin{figure}[!t]
\centering
\includegraphics[width=\linewidth]{figures/legend_logprobs_risk.pdf}\\[0.5em]
\textbf{Log-Probabilities Subclaim Scoring}
\\
\begin{subfigure}{0.42\linewidth}
    \includegraphics[width=\linewidth]{figures/loss_factuality_fdrLoss_logprobsScoring_risk_25trials_200ngrid.pdf}
\end{subfigure}
\hfill
\begin{subfigure}{0.42\linewidth}
    \includegraphics[width=\linewidth]{figures/loss_factuality_fdrLoss_logprobsScoring_claims_25trials_200ngrid.pdf}
\end{subfigure}
\\
\textbf{Self-Evaluation Subclaim Scoring}
\\
\begin{subfigure}{0.42\linewidth}
    \includegraphics[width=\linewidth]{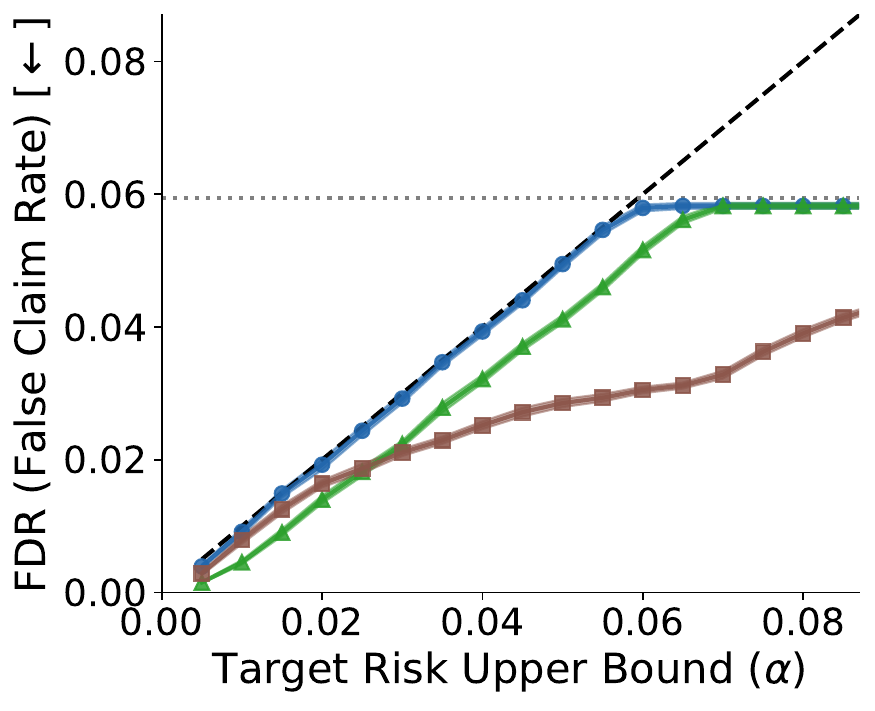}
\end{subfigure}
\hfill
\begin{subfigure}{0.42\linewidth}
    \includegraphics[width=\linewidth]{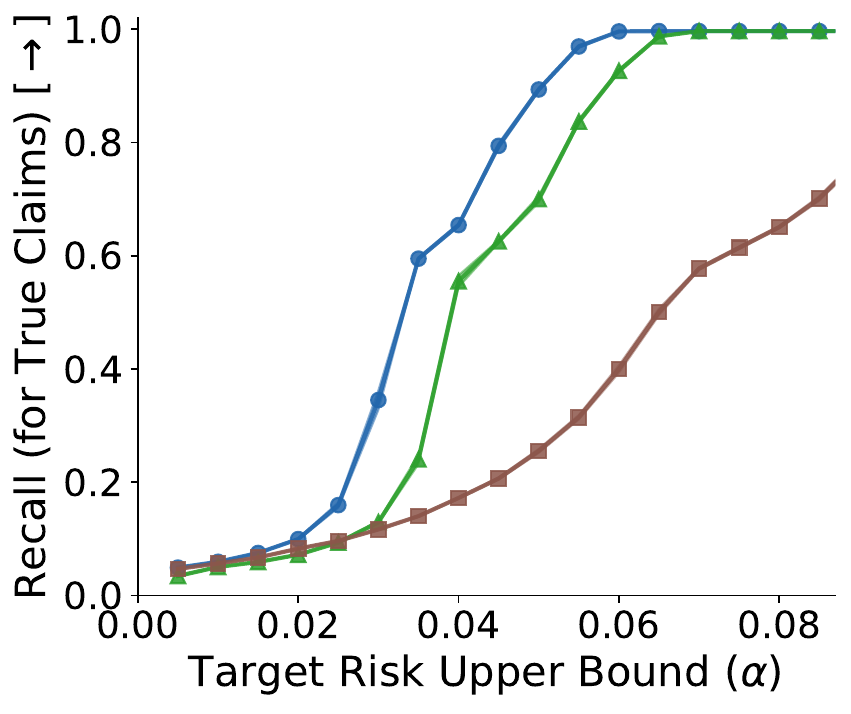}
\end{subfigure}
\\
\textbf{Frequency Subclaim Scoring}
\\
\begin{subfigure}{0.42\linewidth}
    \includegraphics[width=\linewidth]{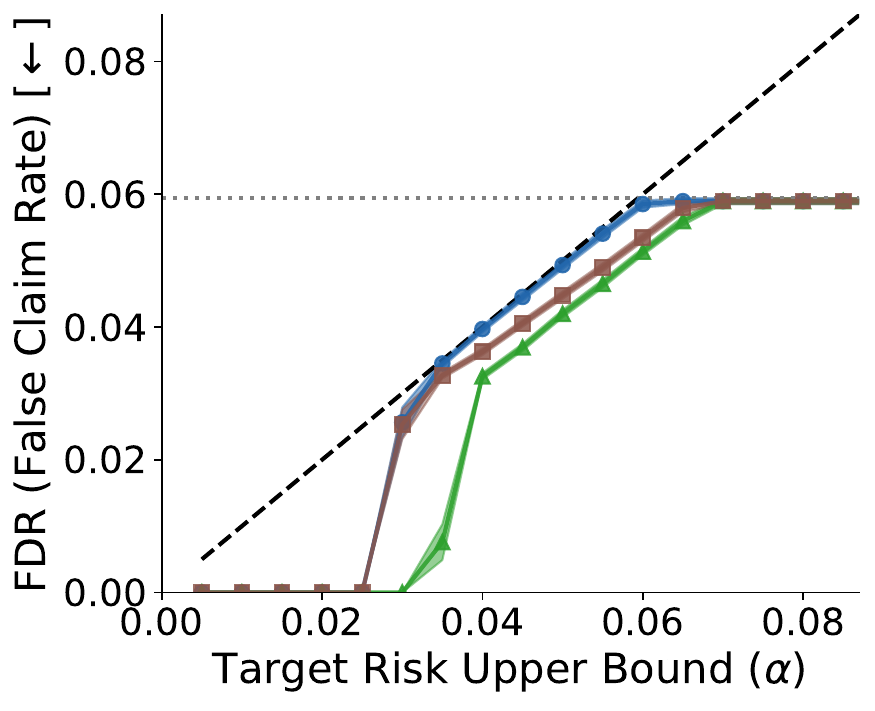}
\end{subfigure}
\hfill
\begin{subfigure}{0.42\linewidth}
    \includegraphics[width=\linewidth]{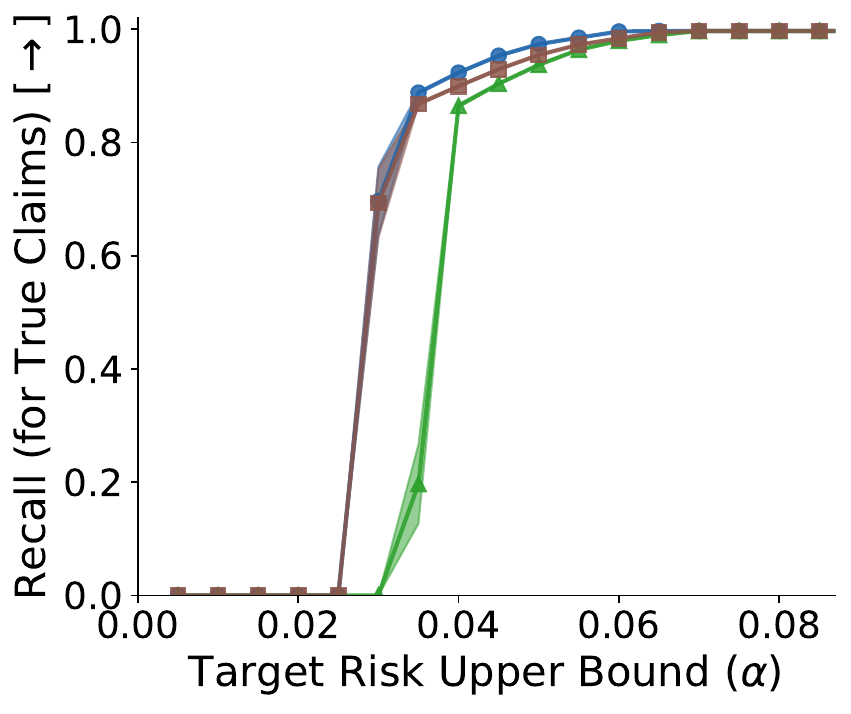}
\end{subfigure}
\caption{Medical QA factuality control results on MedLFQA on three subclaim scoring methods: log-probabilities \textbf{(top row)}, self-evaluation \textbf{(middle row)}, and frequency \textbf{(bottom row)}. \textbf{Left column:} Empirical FDR (fraction of retained claims that are false) vs.\ target risk level $\alpha$. The black dashed line indicates $y = x$; valid methods should fall at or below this line. \textbf{Right column:} Recall (fraction of true claims retained) vs.\ $\alpha$. gCRC tightly controls FDR at the target level while achieving superior recall to baselines. Results averaged over 25 random splits. Error regions are standard errors.}
\end{figure}

\clearpage

\section{Constrained Tabular Active Learning Implementation Details and Additional Results}
\label{app:active_learning_details}

This appendix provides implementation details for the constrained active learning experiments in Section~\ref{subsec:active_learning_experiments}. We describe the datasets used for evaluation (Section~\ref{subsec:al_datasets}), the construction of synthetic feasibility constraints based on principal component analysis (Section~\ref{subsec:feasibility_construction}), the active learning setup including surrogate model and acquisition policy (Section~\ref{subsec:al_setup}), and hyperparameter settings (Section~\ref{subsec:al_hyperparams}).

\subsection{Datasets}
\label{subsec:al_datasets}

We evaluate on three regression datasets:

\begin{itemize}
    \item \textbf{Robot Arm Kinematics} \citep{delve1996kin} (8192 samples, 8 features): Predict end-effector distance from joint angles of an 8-link all-revolute robot arm. We use the kin8nm variant from the DELVE repository, a highly non-linear regression task with moderate noise.
    \item \textbf{Airfoil Self-Noise} \citep{brooks1989airfoil} (1503 samples, 5 features): Predict sound pressure level from aerodynamic and geometric properties of NACA 0012 airfoil blade sections tested in an anechoic wind tunnel. We apply log transforms to frequency and suction-side displacement thickness following standard practice.
    \item \textbf{Medical Expenditure Panel Survey (Medical Utilization or MEPS)} \citep{ezzati2008sample} (33005 samples, 107 features): Large-scale survey dataset on medical utilization and costs of health care and health insurance coverage. Surveys are of individuals and families, their medical providers, and employers across the United States. We pre-process the data (for instance constructing dummy/indicator covariates) as described in \citet{barber2021predictive}. Data available at \url{https://meps.ahrq.gov/mepsweb/data_stats/download_data_files_detail.jsp?cboPufNumber=HC-192}
\end{itemize}

\subsection{Feasibility Constraint Construction}
\label{subsec:feasibility_construction}

Our datasets do not include recorded feasibility constraints, so we construct synthetic constraints based on the first principal component of the covariate matrix. The first principal component captures the dominant pattern of covariation in the data. Records with high PC1 values follow this typical covariation pattern and are likely feasible, while records with low PC1 values deviate from this pattern. Their covariates vary in an atypical or ``weird'' way that we treat as risky. This captures the intuition that unusual configurations are more likely to violate constraints. In this experiment, infeasible records receive higher label measurement noise than feasible records rather than receiving no feedback.

We define the feasibility indicator for each record as follows:
\begin{enumerate}
    \item \textbf{PCA projection}: Compute the first principal component of the full covariate matrix $X \in \mathbb{R}^{n \times d}$ and project all records onto this axis, yielding $z_i \in \mathbb{R}$ for each record $i$.

    \item \textbf{Exponential tilting}: Apply min-max normalization to $\{z_i\}$ and compute exponentially-tilted values $\tilde{z}_i = \exp(\gamma \cdot z_i^{\text{norm}})$, where $\gamma > 0$ is the \texttt{initial\_sampling\_bias} parameter.

    \item \textbf{Feasibility probability}: Convert relative ranks of $\{\tilde{z}_i\}$ to feasibility probabilities via a logistic CDF:
    \begin{align}
        p_i^{\text{feasible}} = \sigma\Big(\text{rank}(\tilde{z}_i) / n;\ \mu, s\Big),
    \end{align}
    where $\sigma(\cdot; \mu, s)$ is the logistic CDF with location $\mu = \min(2.5\alpha, 0.98)$, scale $s = 0.1$, and $\alpha$ is the target risk level for CPC. We tie the logistic location to $\alpha$ so that the proportion of infeasible records in the pool scales with the target risk level, ensuring the constraint remains binding across different choices of $\alpha$. For our experiments with $\alpha = 0.2$, this yields $\mu = 0.5$.

    \item \textbf{Bernoulli sampling}: Sample feasibility indicators $F_i \sim \text{Bernoulli}(p_i^{\text{feasible}})$.
\end{enumerate}

The initial training data is sampled with the same exponential tilting toward high PC1 values (Section~\ref{subsec:al_setup}), so the learner begins in a high-feasibility region where covariates follow the dominant covariation pattern. The pool contains many records from the opposite end of the PC1 spectrum, atypical configurations with low feasibility probability that a naive uncertainty-based acquisition policy would preferentially select.

\subsection{Active Learning Setup}
\label{subsec:al_setup}

\textbf{Initial data.} We hold out 20\% of each dataset as a fixed test set for evaluating MSE. From the remaining records, we sample $n_{\text{initial}}$ points with exponential tilting toward high PC1 values (using the same bias parameter $\gamma$ from Section~\ref{subsec:feasibility_construction}), then split these into training (80\%) and calibration (20\%) sets. This biased initialization places the learner in the high-feasibility region where covariates follow typical covariation patterns, reflecting a scenario where the incumbent safe policy operates in well-understood regimes.

\textbf{Surrogate model.} We fit a Gaussian Process (GP) regression model with a DotProduct + WhiteKernel covariance function using the \texttt{scikit-learn} package \citep{pedregosa2011scikit}:
\begin{align}
    k(x, x') = \sigma_0^2 + x^\top x' + \sigma_n^2 \cdot \mathbf{1}[x = x'],
\end{align}
where $\sigma_0^2$ is the signal variance and $\sigma_n^2$ is the noise variance.

\textbf{Acquisition policy.} At each iteration, we define an acquisition distribution over pool records via exponential tilting toward posterior variance. The action $a \in \mathcal{A}_t$ is an index into the pool of unlabeled records at iteration $t$, and the acquisition policy is
\begin{align}
    \pi(a) \propto \exp\Big(\lambda \cdot \frac{\hat{\sigma}^2_a}{\max_{j \in \mathcal{A}_t} \hat{\sigma}^2_j - \min_{j \in \mathcal{A}_t} \hat{\sigma}^2_j}\Big),
\end{align}
where $\hat{\sigma}^2_a$ is the posterior variance of the latent function $f$ (not the noisy labels $y$) at the covariates for record $a$, and $\lambda > 0$ controls the degree of uncertainty-seeking behavior. This policy preferentially selects records where the model is uncertain, which in our setup corresponds to regions with atypical covariation patterns (and thus potentially infeasible).

\textbf{Conformal policy control.} Before sampling, we apply CPC to constrain the acquisition policy relative to the initial source distribution (the biased sampling distribution used for initialization). This clips the likelihood ratio between the acquisition policy and source policy, as described in Section~\ref{subsec:cpc_selecting_rc_policy}, to control the risk of selecting infeasible records.

\textbf{Sequential updates.} At each iteration, we sample one record from the (constrained) acquisition policy with replacement (the pool remains unchanged), observe a noisy version of its label (Gaussian noise with magnitude 0.05) and its feasibility indicator, and add the observation to either the training or calibration set with equal probability. The GP is then refit on the updated training set. We repeat this process for $T$ iterations.

\subsection{Hyperparameters}
\label{subsec:al_hyperparams}

Table~\ref{tab:al_hyperparameters} summarizes the hyperparameters used in our experiments.

\begin{table}[h]
\centering
\caption{Hyperparameters for constrained active learning experiments.}
\label{tab:al_hyperparameters}
\begin{tabular}{lll}
\toprule
\textbf{Parameter} & \textbf{Symbol} & \textbf{Value} \\
\midrule
Initial pool size & $n_{\text{initial}}$ & 48 \\
Initial train/cal split & -- & 80\%/20\% \\
Acquisition iterations & $T$ & 10 \\
New point train probability & -- & 0.5 \\
Sampling bias & $\gamma$ & 1.0 \\
Acquisition temperature & $\lambda$ & 10.0 \\
GP signal variance & $\sigma_0^2$ & 1.0 \\
GP noise variance & $\sigma_n^2$ & 1.0 \\
Observation noise magnitude & -- & 0.05 \\
Infeasible noise scale & -- & 1.1 \\
Target risk level & $\alpha$ & 0.2 \\
\bottomrule
\end{tabular}
\end{table}

\clearpage

\section{Constrained Black-Box Sequence Optimization Implementation Details and Results}
\label{app:sequence_opt}

This appendix provides implementation details for the constrained black-box sequence optimization experiments in Section~\ref{subsec:cpc_LLM_methods}. We describe the Ehrlich test functions used for evaluation (Section~\ref{subsec:ehrlich}), the procedure for training the safe policy via supervised fine-tuning on genetic algorithm data (Section~\ref{subsec:safe_policy_training}), the policy improvement procedure via DPO (Section~\ref{subsec:policy_improvement}), and supplementary results across additional risk levels (Figure~\ref{fig:extra_llm_policy_control_results}).

\subsection{Ehrlich Test Functions}
\label{subsec:ehrlich}

We evaluate CPC on Ehrlich functions \citep{chen2025generalists}, a class of synthetic test functions designed to capture the geometric structure of biomolecular sequence optimization problems such as antibody affinity maturation. An Ehrlich function $f: \mathcal{V}^L \rightarrow (-\infty, 1]$ maps sequences of length $L$ over a vocabulary $\mathcal{V}$ to a score indicating how well the sequence satisfies a collection of $c$ gapped motifs $\mathbf m^{(i)}$ with spacing $\mathbf s^{(i)}$. Formally,
\begin{align}
    f(\mathbf{x}) = \begin{cases}
        \prod_{i=1}^c g \circ h_q(\mathbf{x}, \mathbf{m}^{(i)}, \mathbf{s}^{(i)}) & \text{if } \mathbf{x} \in \mathcal{F} \\
        -\infty & \text{otherwise}
    \end{cases},
\end{align}
where $h_q$ measures motif satisfaction with precision $q$, $g$ is a response function capturing epistatic effects, and $\mathcal{F}$ is a feasibility set defined by a discrete Markov process with certain infeasible transitions (\textit{i.e.} some bigrams have zero probability mass). Sequences outside $\mathcal{F}$ are assigned $-\infty$, modeling expression constraints in real biophysical assays.

We follow the naming convention \textbf{Ehr(}$|\mathcal{V}|$\textbf{,} $L$\textbf{)-}$c$\textbf{-}$k$\textbf{-}$q$ from \citet{chen2025generalists}, where $k$ is the motif length and $q$ controls signal sparsity. Ehrlich functions are procedurally generated, have adjustable difficulty, and are provably solvable, making them well-suited for benchmarking constrained optimization algorithms. In our experiments, we use \textbf{Ehr(32, 32)-4-4-4}: sequences of length $L=32$ over a vocabulary of size $|\mathcal{V}|=32$, with $c=4$ motifs of length $k=4$ and quantization precision $q=4$.

\subsection{Training the Safe Policy}
\label{subsec:safe_policy_training}

Following \citet{chen2025generalists}, we initialize the optimization loop with a data package $\mathcal{D}^{(0)}$ collected from a genetic algorithm (GA) pre-solver. Starting from a single feasible seed sequence $\mathbf{x}_0$ sampled from the discrete Markov process, we run $n_0=10$ rounds of the GA with 10000 particles per iteration. This produces approximately 100K scored sequences that provide initial coverage of the feasible region $\mathcal{F}$. The GA uses mutation probability $p_m=0.05$ and survival quantile $\alpha=0.5$.

From the GA history, we construct the initial supervised fine-tuning (SFT) dataset by identifying improving pairs: for each sequence $\mathbf{x}$ in the history, we search for sequences $\mathbf{x}'$ within edit distance $\delta$ such that $f(\mathbf{x}') > f(\mathbf{x})$. These $(\mathbf{x}, \mathbf{x}')$ pairs form instruction-response examples for training the safe policy $\pi_{\theta_0}$, which we initialize from a Pythia 14M-parameter decoder-only transformer \citep{biderman2023pythia}. An initial 20 seeds for the safe policy were then selected from the top 0.1\% scoring feasible sequences via ranked farthest-first traversal to encourage diversity among top-scoring sequences. See Table~\ref{tab:sft_hyperparams} for hyperparameter settings.

\begin{table}[h]
\centering
\caption{Hyperparameters for training the safe policy via SFT.}
\label{tab:sft_hyperparams}
\begin{tabular}{ll}
\toprule
\textbf{Hyperparameter} & \textbf{Value} \\
\midrule
Base model & Pythia 14M \\
Learning rate & $10^{-6}$ \\
Training epochs & 10 \\
Batch size & 32 \\
Edit distance threshold ($\delta$) & $0.3 \cdot L$ \\
\bottomrule
\end{tabular}
\end{table}

\begin{algorithm}[t]
    \caption{Conformal Policy Control}
    \label{alg:cpc_round}
    \begin{algorithmic}
        \REQUIRE Safe policy $\pi_0$; context draws $\{X_i\}_{i=1}^n \sim p(X)$; target risk level $\alpha$
        \item[] \textbf{Phase 1: Safe Policy Data Collection}
        \STATE Draw actions $A_i \sim \pi_0(\cdot \mid X_i)$ for $i = 1, \ldots, n$
        \STATE Observe rewards $R_i = r(X_i, A_i)$ and losses $L_i = \ell(X_i, A_i)$
        \STATE Split data: $\mathcal{D}_{\text{train}}, \mathcal{D}_{\text{cal}} \gets \textsc{RandomSplit}(\{(X_i, A_i, R_i, L_i)\})$
        \item[] \textbf{Phase 2: Policy Improvement}
        \STATE $\pi_t \gets \textsc{ImprovePolicy}(\pi_0, \mathcal{D}_{\text{train}})$
        \item[] \textbf{Phase 3: Calibration}
        \STATE $\hat{\beta} \gets \textsc{CalibrateBeta}(\pi_0, \pi_t, \mathcal{D}_{\text{cal}}, \alpha)$ (Algorithm \ref{alg:calibrate_beta})
        \item[] \textbf{Phase 4: Constrained Deployment}
        \STATE Draw improved actions $A_i' \sim \pi_t^{(\hat{\beta})}(\cdot \mid X_i)$ via accept-reject for $i = 1, \ldots, n$
        \STATE \textbf{Return:} $\{A_i'\}_{i=1}^n$
    \end{algorithmic}
\end{algorithm}

\subsection{Improving the Safe Policy}
\label{subsec:policy_improvement}

To improve the safe policy, we take a fixed set of seed sequences from the GA optimizer history and sample actions from $\pi_{\theta_0}$. From the resulting data, we construct DPO preference triplets of the form $(\mathbf{x}, \mathbf{x}^+, \mathbf{x}^-)$, where $\mathbf{x}^+$ is preferred over $\mathbf{x}^-$ based on the oracle scores. We initialize the improved policy $\pi_{\theta_t}$ at the safe policy weights and train using the DPO loss \citep{rafailov2023direct}. At each round, 10 seeds were used for the DPO-post-trained policy, with 2 from the initial safe model's seeds and 8 selected as the highest-scoring historical training sequences. See Algorithm \ref{alg:cpc_round} for pseudocode of a single round of CPC.

We chose DPO as our policy improvement method because \citet{chen2025generalists} demonstrated that it struggles with poor feasibility rates compared to other fine-tuning approaches, making it a challenging test case for our risk control procedure. See Table~\ref{tab:dpo_hyperparams} for hyperparameter settings.\footnote{DPO has a $\beta$ hyperparameter that controls its regularization strength, not to be confused with the use of $\beta$ as the likelihood-ratio bound in CPC.} 

\begin{table}[h]
\centering
\caption{Hyperparameters for policy improvement via DPO.}
\label{tab:dpo_hyperparams}
\begin{tabular}{ll}
\toprule
\textbf{Hyperparameter} & \textbf{Value} \\
\midrule
Learning rate & $1.5e^{-7}$ \\
DPO $\beta_{\text{dpo}}$ & 0.1 \\
Training epochs per round & 1 \\
Batch size & 32 \\
Number of DPO rounds & 20 \\
\bottomrule
\end{tabular}
\end{table}

\begin{figure*}
    \centering
    \begin{subfigure}{\textwidth}
        \includegraphics[width=\textwidth]{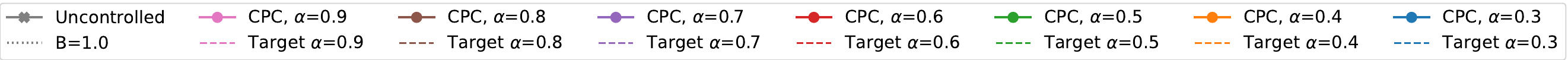}
    \end{subfigure}
    \\
    \begin{subfigure}{0.33\textwidth}
        \includegraphics[width=\textwidth]{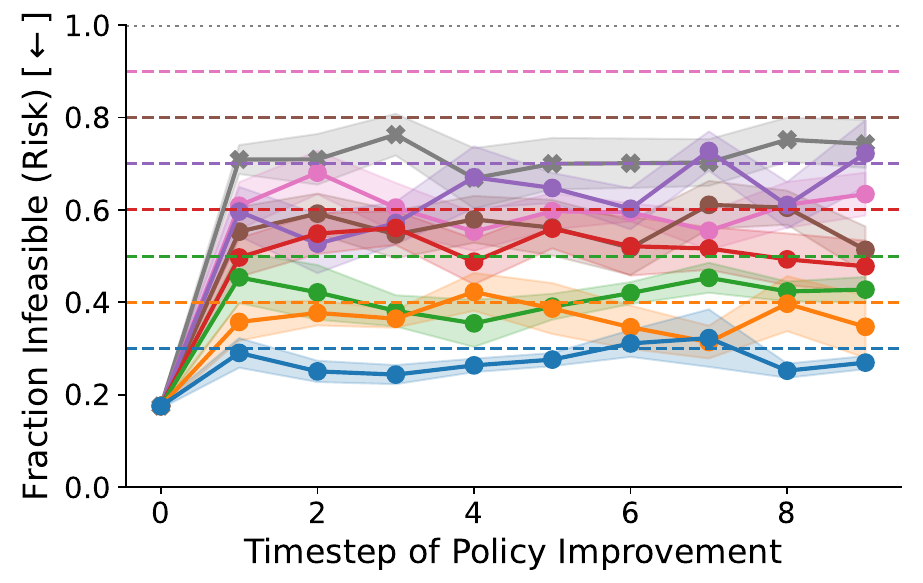}
    \end{subfigure}
    \hfill
    \begin{subfigure}{0.33\textwidth}
        \includegraphics[width=\textwidth]{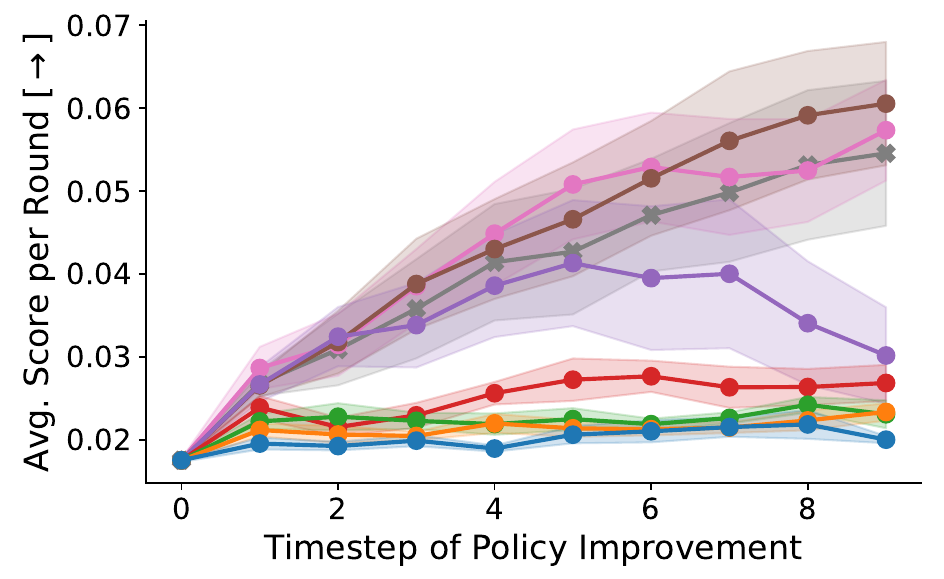}
    \end{subfigure}
    \hfill
    \begin{subfigure}{0.33\textwidth}
        \includegraphics[width=\textwidth]{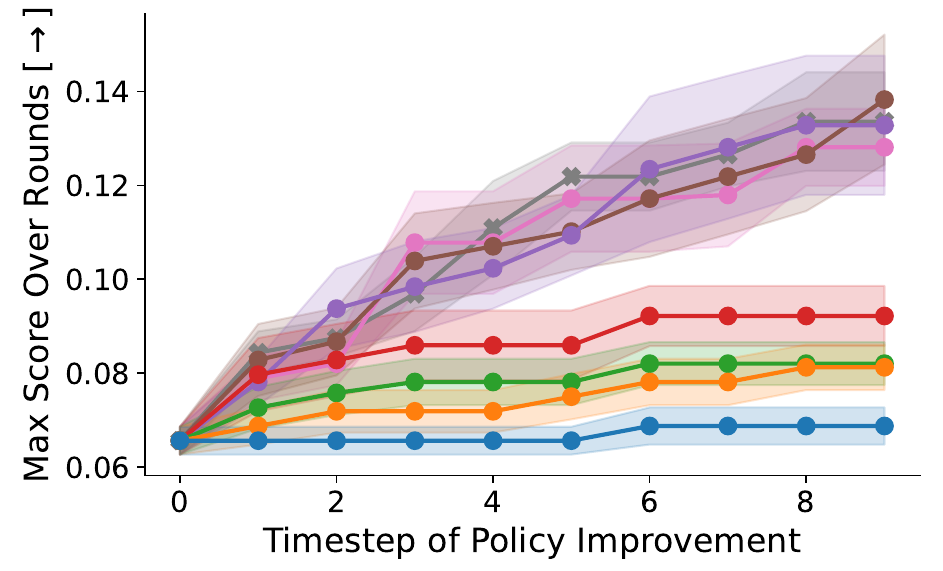}
    \end{subfigure}
    \caption{
        Supplementary experimental results applying CPC to constrained black-box optimization of token sequences scored with Ehrlich functions. In the \textbf{left} panel, we show the average feasibility rate round-over-round. In the \textbf{center}, we show the average objective value of solutions sampled from the DPO-tuned policy, and in the \textbf{right}, the best solution found so far at each round. CPC allows us to easily modify the constraint violation risk by directly manipulating $\alpha$. We find that moderate risk control may actually stabilize the algorithm and lead to better overall performance since fewer samples are wasted to infeasibility constraints.
    }
    \label{fig:extra_llm_policy_control_results}
\end{figure*}

\clearpage

\end{document}